\documentclass[11pt]{article}
\pdfoutput=1

\usepackage{enumerate}
\usepackage[OT1]{fontenc}
\usepackage{amsmath,amssymb}
\usepackage{subfigure}
\usepackage{natbib}
\usepackage[usenames]{color}
\usepackage[colorlinks,
            linkcolor=red,
            anchorcolor=blue,
            citecolor=blue
            ]{hyperref}
            
\usepackage{enumitem}
\usepackage{mylatexstyle}
\usepackage{mathrsfs}
\usepackage{fullpage}
\usepackage[protrusion=false,expansion=true]{microtype}

\def \polylog {\mathrm{polylog}}

\def \la {\langle}
\def \ra {\rangle}
\def \poly {\mathrm{poly}}

\begin{document}

\title{\huge Per-Example Gradient Regularization Improves Learning Signals from Noisy Data}


\author
{
Xuran Meng\thanks{\scriptsize Department of Statistics \& Actuarial Science, The University of Hong Kong; {\tt xuranmeng@connect.hku.hk}}
\qquad
Yuan Cao\thanks{\scriptsize Department of Statistics \& Actuarial Science,
	The University of Hong Kong;  {\tt yuancao@hku.hk}}
\qquad 
Difan Zou\thanks{\scriptsize Department of Computer   Science   \&   Institute   of   Data   Science, The University of Hong Kong;
 {\tt dzou@cs.hku.hk}}
}





\date{}

\maketitle

\begin{abstract}
Gradient regularization, as described in \citet{barrett2021implicit}, is a highly effective technique for promoting flat minima during gradient descent. Empirical evidence suggests that this regularization technique can significantly enhance the robustness of deep learning models against noisy perturbations, while also reducing test error. In this paper, we explore the per-example gradient regularization (PEGR) and present a theoretical analysis that demonstrates its effectiveness in improving both test error and robustness against noise perturbations. Specifically, we adopt a signal-noise data model from \citet{cao2022benign} and show that PEGR can learn signals effectively while suppressing noise. In contrast, standard gradient descent struggles to distinguish the signal from the noise, leading to suboptimal generalization performance. Our analysis reveals that  PEGR penalizes the variance of pattern learning, thus effectively suppressing the memorization of noises from the training data. These findings underscore the importance of variance control in deep learning training and offer useful insights for developing more effective training approaches.
\end{abstract}

\section{Introduction}

Regularization in deep learning refers to a set of techniques aimed at improving the performance of the model \citep{kukavcka2017regularization}. Gradient Regularization, as a new regularization technique \citep{barrett2021implicit}, proposed a modified gradient flow, represented by $L(\Wb) + (\lambda/4)\|\nabla L(\Wb)\|^2$. The equation combines the original loss function, $L(\Wb)$, with an implicit regularizer that penalizes the Euclidean norm of the gradient. In a recent study, \citet{barrett2021implicit} analyzed the impact of finite learning rates on the iterates of gradient descent and discovered that such regularization improves testing accuracy and is tied to sharpness-aware minimization with some minor differences.


It is worth noting that the gradient norm lacks a stochastic approximation, which means that it cannot be integrated with commonly used stochastic gradient descent (SGD) methods \citep{keskar2017large, smith2020generalization}. To bridge this gap, \citet{smith2021origin} proposed a per-example gradient regularization (PEGR) that modifies the regularization term $\|\nabla L(\Wb)\|^2$ to $1/m\sum_{k=1}^{m}\|\nabla L_i(\Wb)\|^2$, where $\|\nabla L_i(\Wb)\|$ represents the Euclidean norm of the gradient in the minibatch of data. This kind of gradient descent takes into account the concept of sharpness aware minimization \citep{foret2021sharpness,geiping2021stochastic}. Surprisingly, it has been empirically observed that this regularization method achieves significantly improved generalization performance compared to full-batch gradient regularization \citep{andriushchenko2022towards}.

While PEGR has been observed to have good performance, the essential theory behind this intuitive selection such as the stochastic selection and full gradient descent remains unclear. Therefore, it requires to seek a more fundamental and comprehensive understanding of the behavior of gradient dynamics under such regularization. However, understanding gradient regularization in neural networks is challenging, primarily due to nonconvexity. The activation and loss functions make the training dynamic analysis a highly nonconvex optimization problem  \citep{dherin2022neural}. Previous work has not provided a comprehensive theoretical analysis of the improvement in testing accuracy, and there is no theoretical explanation for the behavior of gradient regularization, especially its role during model training.

In this work, we investigate the algorithmic behavior of PEGR.
In particular, we present an algorithmic analysis of learning two-layer convolutional neural networks (CNNs) with fixed second-layer parameters of $+1$'s and $-1$'s and a square ReLU activation function, i.e., $\sigma(z)=\max\{0,z\}^2$. We consider a setting where the input data comprise label-dependent signals and label-independent noises, and utilize a signal-noise decomposition of the CNN filters to precisely characterize how PEGR affects the signal learning and noise memorization during the model training. We then prove a separation between the generalization performances when the model training is performed with and without PEGR.  Our paper makes significant contributions in the following ways:
\begin{enumerate}
    \item We identify that per-example gradient regularization (PEGR) can effectively suppress noise memorization while promoting signal learning in over-parameterized neural networks. Specifically, we demonstrate that when certain conditions are met, CNN models trained using gradient descent with PEGR prioritize learning the signal over memorizing the noise. Furthermore, by appropriately closing the regularization at the right time, these models achieve convergence of training gradient and exhibit lower test error.
    \item Additionally, we present a negative result that demonstrates how CNN models trained using gradient descent without PEGR are prone to memorize noise instead of learning the signal. Taken together with our previous finding, this result provides clear evidence of the importance of PEGR in promoting effective learning in over-parameterized neural networks. 
    \item Our theoretical analysis suggests that the advantages of PEGR in promoting signal learning are most pronounced in the early stages of training. As the network continues to learn and the signal becomes sufficiently strong, we also provide a theoretical framework for determining an appropriate time to close the gradient regularization, which ensures gradient convergence while avoiding over-regularization of the signal. 
\end{enumerate}
To summarize, Our results demonstrate that PEGR under gradient descent can effectively learn signals while suppressing noise and provide guidance on choosing hyperparameters.  The remainder of this paper is organized as follows. First, we provide some additional references and notations below. In Section~\ref{sec:problemsetting}, we introduce the problem settings. Next, in Section~\ref{sec:mainresults}, we present the theoretical analysis of the efficacy of PEGR under gradient descent. Section~\ref{sec:experiments} shows the experimental results which support our theories. Section~\ref{sec:overviewproof} provides a brief overview of the proof of the main theorems. Finally, we conclude the paper in Section~\ref{sec:conclusion} and discuss some related questions for future investigation.

\subsection{Additional related work}
In  this  section,  we  will  discuss  in  detail  some  of  the  related  work  brieﬂy  mentioned  before.\newline

\noindent \textbf{Sharpness aware minimization.} 
The study on the connection between sharpness and generalization can be traced back to \citet{hochreiter1997flat}. \citet{keskar2017large} observed a positive correlation between the batch size, the generalization error, and the sharpness of the loss landscape when changing the batch size. \citet{jastrzkebski2017three} extended this by finding a correlation between the sharpness and the ratio between learning rate to batch size. \citet{jiang2020fantastic} performed a large-scale empirical study on various generalization measures and show that sharpness-based measures have the highest correlation with generalization.
\citet{foret2021sharpness} introduced a novel Sharpness-Aware Minimization (SAM) procedure for simultaneously minimizing loss value and loss sharpness to improve model generalization ability, which results in state-of-the-art performance on several benchmark datasets and models, as well as providing robustness to label noise.  
\citet{zhao2022penalizing}  shows that penalizing the gradient norm of the loss function during optimization is a method to improve the generalization performance of deep neural networks, which leads to better flat minima, and achieves state-of-the-art results on various datasets. 
\citet{wen2023does} clarifies the mechanism and rigorously defines the sharpness notion that Sharpness-Aware Minimization (SAM) regularization technique is based on, revealing the alignment between gradient and top eigenvector of Hessian as the key mechanism behind its effectiveness. \newline

\noindent \textbf{Neural network training techniques.} Besides the gradient regularization we previously discussed, a series of recent works have also studied some other training techniques.  \citet{blanc2020implicit} found that networks trained with perturbed training labels exhibit an implicit regularization term that drives them towards simpler models, regardless of their architecture or activation function.
Inspired by \citet{martin2021implicit}'s study of weight matrix spectra in deep neural networks, \citet{meng2023impact} introduced a spectral criterion to identify the presence of heavy tails, a sign of regularization in DNNs, enabling early stopping of the training process without testing data to avoid overfitting while preserving generalization ability. \citet{zou2023benefits} provided a theoretical explanation for Mixup's efficacy in improving neural network performance by showing its ability to learn rare features through a signal-noise data model, and suggests early stopping as a useful technique for Mixup training. \citet{allen2020towards} described a study on how ensemble of independently trained neural networks with the same architecture can improve test accuracy and how this superiority can be distilled into a single model using knowledge distillation, based on a theory that when data has a structure called `` multi-view".\newline

\noindent\textbf{Overfitting in over-parameterized regime.} We introduce several  works related to the benign/harmful overfitting in over-parameterized regime. \citet{hastie2019surprises,wu2020optimal} investigated a scenario where the dimension and sample size increase while maintaining a fixed ratio between them, and observed a double descent risk curve in relation to this ratio.  \citet{bartlett2020benign} established upper and lower risk bounds for the over-parameterized minimum norm interpolator and demonstrated that benign overfitting can occur under certain conditions on the data covariance spectrum.
\citet{zou2021benign} examined the generalization performance of constant stepsize stochastic gradient descent with iterate averaging or tail averaging in the over-parameterized regime.   \citet{mei2022generalization,meng2022multiple} studied the generalization of random feature models under the setting where the data dimension, data size, and feature dimension tend to infinity proportionally, and observe the phenomenon of multiple descent. \citet{cao2022benign} examines benign overfitting in training two-layer CNNs and finds a sharp phase transition between benign and harmful overfitting, determined by the signal-to-noise ratio, where a CNN trained by gradient descent achieves small training and test loss under certain conditions and only achieves a constant level test loss otherwise. \citet{kou2023benign} investigates the phenomenon of benign overfitting in ReLU neural networks and provides algorithm-dependent risk bounds for learning two-layer ReLU convolutional neural networks with label-flipping noise, demonstrating a sharp transition between benign and harmful overfitting under different conditions on data distribution.

\subsection{Notation}
Given two sequences ${x_n}$ and ${y_n}$, we denote $x_n = O(y_n)$ if there exist some absolute constant $C_1 > 0$ and $N > 0$ such that $|x_n| \leq C_1 |y_n|$ for all $n \geq N$. Similarly, we denote $x_n = \Omega(y_n)$ if there exist $C_2 > 0$ and $N > 0$ such that $|x_n| \geq C_2 |y_n|$ for all $n > N$. We say $x_n = \Theta(y_n)$ if $x_n = O(y_n)$ and $x_n = \Omega(y_n)$ both hold. We use $\tilde{O}(\cdot)$, $\tilde{\Omega}(\cdot)$, and $\tilde{\Theta}(\cdot)$ to hide logarithmic factors in these notations respectively. Moreover, we denote $x_n = \poly(y_n)$ if $x_n = O(y_n^D)$ for some positive constant $D$, and $x_n = \polylog(y_n)$ if $x_n = \polylog(y_n)$. Finally, for two scalars $a$ and $b$, we denote $a \vee b = \max\{a,b\}$.

\section{Problem Setting}
\label{sec:problemsetting}
In this section, we will discuss the data generation model and the convolutional neural network (CNN) that we utilized in our paper. Our research focuses on binary classification. In order to establish the context for our work, we will introduce the data distribution $\cD$ that we considered in the following definition.

\begin{definition}[Data model]
    \label{def:Data_distribution}
Let $\bmu \in \RR^d$ denote a fixed vector representing the signal contained in each data point. Each data point $(\xb, y)$, where $\xb=[\xb^{(1)\top},\xb^{(2)\top}]^\top \in \RR^{2d}$ and $y\in{\pm 1 }$, is generated from the following distribution $\cD$:
\begin{enumerate}
\item The label $y$ is generated as a Rademacher random variable.
\item A noise vector $\bxi$ is generated from the Gaussian distribution $\cN(\zero, \sigma_p^2 \cdot (\Ib - \bmu\bmu^\top/\|\bmu\|^2))$.
\item One of $\xb^{(1)}$ or $\xb^{(2)}$ is given as $y \cdot \bmu$, which represents the signal, while the other is given by $\bxi$, which represents noise.
\end{enumerate}
\end{definition}
A series of similar data have been widely studied in recent works \citep{frei2022benign,cao2022benign,shen2022data,zou2023benefits,zou2023understanding}. The data generation model draws inspiration from image data, where the input is comprised of various patches, with only a subset of these patches related to the image's class label. Specifically, we designate the patch $y\cdot\bmu$ as the signal patch that correlates to the data's label, while $\bxi$ represents the noise patch, which is unrelated to the label and therefore irrelevant for prediction.  In order to clearly distinguish the role of gradient regularization in learning signals/noises, we assume that the noise patch is generated from a Gaussian distribution $\cN(\zero, \sigma_p^2 \cdot (\Ib - \bmu\bmu^\top/\|\bmu\|^2))$ to ensure the noise vector is orthogonal to the signal vector $\bmu$ for simplicity.

\textbf{Two-layer CNN} 
We investigate a two-layer convolutional neural network that applies its filters to the two patches, $\xb^{(1)}$ and $\xb^{(2)}$, separately. Additionally, we fix the second layer parameters of the network as +1/m and -1/m, respectively. Then the network can be written as $f(\Wb,\xb)  = \sum_{j\in\{\pm1\}}j\cdot F_j(\Wb_j,\xb)$, where
\begin{align*}
F_j(\Wb_j,\xb)=\frac{1}{m}\sum_{r=1}^m \Big(\sigma( \la \wb_{j,r}, \xb^{(1)} \ra )+\sigma( \la \wb_{j,r}, \xb^{(2)} \ra )\Big).
\end{align*}


Here, $\sigma(z)= \ReLU^2(z)$ is the activation function, $\wb_{j,r}$, $j\in\{\pm1\}$ and $r\in[m]$, are first-layer parameter vectors with positive and negative second-layer parameters respectively, and $\xb=[\xb^{(1)\top},\xb^{(2)\top}]^\top\in\RR^{2m}$. Note that using polynomial activation is commonly applied in studying feature learning of deep learning models. Our analysis can also be applied to $\sigma(z)=\ReLU^{q}(z)$ with $q>2$  by some additional treatments. We denote by $\Wb\in\RR^{2m}$ the collection of all $\wb_{j,r}$, $j\in\{\pm1\}$ and $r\in[m]$. 
For each data point $\xb_i=[y_i\bmu^\top,\bxi_i^\top]^\top$,   $\bmu$ is  the signal and $\bxi_i$ is the noise. 

Given $n$ training data points $(\xb_i,y_i)$, $i\in[n]$, we define the empirical cross-entropy loss as
\begin{align*}
    L_S(\Wb) = \frac{1}{n}\sum_{i=1}^n \ell[y_i\cdot f(\Wb,\xb_i)],
\end{align*}
where $\ell(z) = \log(1 + \exp(-z))$. Then the loss function with gradient regularization is given as
\begin{align*}
    \tilde{L}(\Wb) = L_S(\Wb) + \frac{\lambda}{2n} \cdot \sum_{i=1}^n\| \nabla_{\Wb} \ell[y_i\cdot f(\Wb,\xb_i)] \|_F^2,
\end{align*}
where $\lambda$ is the regularization parameter. It is worth noting that the gradient regularization is calculated by averaging the gradient norms when fed into different training data points, which can be also understood as a kind of variance control penalty over the training data. 

We consider training $f(\Wb,\xb)$ by minimizing $\tilde{L}(\Wb)$ 
with gradient descent $\wb_{j,r}^{(t+1)}=\wb_{j,r}^{(t)}-\eta\nabla_{\wb_{j,r}^{(t)}}\tilde{L}(\Wb)$, and we are particularly interested in analyzing the test error differences between the  case with gradient regularization ($\lambda > 0$) and the case without regularization ($\lambda = 0$). However, note that the direct comparison between methods with and without regularization may not be fair. Therefore we compare the following two cases:
\begin{itemize}
    \item Implementing gradient regularization in the first phase of training: with some appropriately chosen $\tT_1$, we use $\lambda > 0$ for iterations $0\leq t\leq \tT_1$, and set $\lambda=0$ for iterations $ t >\tT_1$.
    \item Using no regularization throughout training:  we use $\lambda = 0$ for all iterations $t\geq 0$.
\end{itemize}
We aim to show that both methods above can minimize the training loss $L_S(\Wb)$ while achieving different prediction accuracies on test data. Note that we consider gradient descent starting from Gaussian initialization, where each entry of $\Wb_{+1}$ and $\Wb_{-1}$ is sampled from a Gaussian distribution $\cN(0,\sigma_0^2)$ and $\sigma_0^2$ is the variance.

\section{Main Results}
\label{sec:mainresults}
In this section, we present main theoretical results, which rely on a signal-noise decomposition of the filters in the CNN trained by gradient descent.
Our results are based on the following conditions on the dimension $d$, sample size $n$, neural network width $m$, learning rate $\eta$ and initialization scale $\sigma_0$.
\begin{condition}
\label{condition:condition}
Define a small constant $0<\alpha<0.001$. Suppose that
\begin{enumerate}
    \item  Dimension $d$ is large: $d\geq \tilde{\Omega}(m^2n^{(2+2\alpha)})$. 
    \item The signal and noise levels satisfy $(\|\bmu\|+\|\bmu\|^4)\ll\sigma_p\sqrt{d}$ and $\sigma_p\sqrt{d}\to+\infty$.
    \item The standard deviation of Gaussian initialization $\sigma_0$ is sufficiently small. We assume $\sigma_0\leq O\Big(\frac{1}{\sigma_p^{2}d\cdot (nm)^{2\alpha}}\Big)$. 
    \item The learning rate $\eta  = \tilde O \big(\frac{nm}{ \sigma_p^2d}\big)$.
    \item Training sample size $n$ and neural network width $m$ satisfy $m,n=\Omega(\polylog(d,\sigma_0^{-1}))$. 
\end{enumerate}
\end{condition}
We set $\lambda=\sigma_p^{-1}d^{-1/2}$. Below, we provide some remarks on Condition~\ref{condition:condition}. The condition on $d$ ensures that the learning takes place in a sufficiently over-parameterized setting, and similar conditions have been imposed in \citep{chatterji2021finite,cao2021risk}. Additionally, the condition of the signal and noise levels in the data assumes that the signal level is considerably smaller than the noise level. This is to ensure that we are focusing on a relatively difficult learning problem for which learning methods without gradient regularization may fail. Moreover, the conditions on $\sigma_0$ and $\eta$ are technical assumptions we make
to ensure the convergence of gradient descent. Finally, we only require that the sample size $n$ and neural network width $m$ be at least $\polylog(d,\sigma_0^{-1})$, which are very mild assumptions. 

\begin{theorem}
\label{thm:withregularization}
Consider implementing gradient regularization in the first phase of training. Specifically, set $\lambda=\sigma_p^{-1}d^{-1/2}$ for iterations $0\leq t\leq \tT_1$ and $\lambda=0$ for iterations $ t >\tT_1$, where $\tT_1$ satisfies that
$\tT_1=\frac{m}{\eta\|\bmu\|^2}\log\Big(\frac{4}{\sqrt{2\log(8m/\delta)}\sigma_0\|\bmu\|\cdot\log{(n)}}\Big)$. 
 Then under Condition~\ref{condition:condition}, for any $\varepsilon \geq 0$, there exists $\tT_1\leq t\leq \tT_1+\tilde{\Omega}(\frac{2nm\sigma_p^2d}{\eta \varepsilon\|\bmu\|^2})$, such that:
 \begin{enumerate}
     \item The training loss and its gradient converge below  $\varepsilon$: $L_S(\Wb^{(t)})\leq \varepsilon, \|\nabla_{\Wb}L_S(\Wb)|_{\Wb=\Wb^{(t)}}\|_F^2\leq \varepsilon.$
     \item The test error converges to $0$: For any new data $(\xb,y)$, $\PP(yf(\Wb^{(t)},\xb)<0)\leq \frac{1}{\poly(n)}$.
 \end{enumerate}
\end{theorem}

Theorem~\ref{thm:withregularization} characterizes the case of signal learning. It shows that if we add gradient regularization in the beginning of the training and close it at appropriate time, the neural network can learn the signal, and then achieve small testing error and training gradient. To  demonstrate   the  benefits of gradient regularization,  we also  present  the  theorem under the case that there is no gradient  regularization during the whole training process.
\begin{theorem}
    \label{thm:withoutregularization}
    Consider using no regularization throughout training, i.e., $\lambda = 0$ for all $t\geq 0$. Then under Condition~\ref{condition:condition}, for any $\varepsilon >0$, there exists $0\leq t\leq \tilde{O} \big( \frac{nm}{\eta \sigma_p^2d} + \eta^{-1}\varepsilon^{-1}m^3n\big)$ such that:
    \begin{enumerate}
        \item The training loss and its gradient converge below $\varepsilon$: $L_S(\Wb^{(t)})\leq \varepsilon, \|\nabla_{\Wb}L_S(\Wb)|_{\Wb=\Wb^{(t)}}\|_F^2\leq \varepsilon.$
        \item The test error is large: For any new data $(\xb,y)$, $\PP(yf(\Wb^{(t)},\xb)<0)\geq \frac{1}{2.01}$.
    \end{enumerate}
\end{theorem}
Clearly, Theorem~\ref{thm:withoutregularization} is the case without gradient regularization. In this case, the CNN trained 
by gradient descent mainly memorizes noises in the training data and does not learn enough signal. This, together with Theorem~\ref{thm:withregularization}, gives a clear statement:
\begin{enumerate}
    \item[$\bullet$] By modifying the learning algorithm, the per-example gradient regularization can improve the learning of significant patterns within the data (often referred to as ``signal'') while simultaneously discouraging the memorizing of irrelevant or random variation (known as ``noise''). This approach enhances the model's ability to extract relevant features and generalize well to new data.
\end{enumerate}

In this study, we examine the per-example gradient regularization (PEGR), which calculates the gradient of the loss function for each training data point. The full version of gradient regularization (FGR) considers interactions between different training data, which can make the dynamics more complex. We show in the experiments that the full version of gradient descent is unable to improve signal learning while preventing the memorization of noise.

\begin{remark}
\label{remark:GD}
Based on the analysis in Section~\ref{subsec:whyPEGR}, we see that the addition of PEGR to the loss function gives a suppressing term for both noise memorization and signal learning. This suppression term is particularly effective in the presence of high levels of noise, reducing the impact of noise memorization. Conversely, the subtracted term in signal learning is relatively small, allowing for signal learning to progress without significant interference. 
\end{remark}

\section{Experiments}
\label{sec:experiments}
In this section, we conduct numeric experiments and real data experiments, in Subsection~\ref{subsec:numericexperiments} and ~\ref{subsec:realdataexperiment} separately. Both experiments show that PEGR will improve the test accuracy.
\subsection{Numerical experiments}
\label{subsec:numericexperiments}
In this section, we perform numerical experiments on several synthetic data sets, which take different values on the $\sigma_p$, to verify our theoretical results. The synthetic data is generated according to Definition~\ref{def:Data_distribution}. In particular, we set $\|\bmu\|^2=1$, data dimension $d=400$, neural network width $m=10$ and training data size $n=20$. We train all the data sets in total $1500$ epochs with learning rate $\eta=0.02$. In the case of Theorem~\ref{thm:withregularization}, we close the PEGR at epoch $800$. The tuning parameter $\lambda$ are set to be $0.01$. Note that the value of $\sigma_0$ we set is based on the value of $\sigma_p$. When $\sigma_p=0.5$ or $\sigma_p=1$, we set $\sigma_0=0.01$; when $\sigma_p=1.5$, we set $\sigma_0=0.001$. 
We present the results of our experiments in Figure~\ref{fig:trainloss}-\ref{fig:testacc}. We compare the performance of per-example gradient regularization (PEGR), full gradient regularization (FGR), and standard training through numerical experiments. FGR is represented by the following formula:
\begin{align*}
    \hat{L}(\Wb)=L_S(\Wb)+\lambda \bigg\|\frac{1}{n}\sum_{i=1}^n \nabla_{\Wb}\ell(y_i\cdot f(\Wb,\xb_i)) \bigg\|_F^2.
\end{align*}
The standard learning is just $\lambda=0$. \newline

\noindent\textbf{Training Loss:}
\begin{figure}[t!]
\centering
\subfigure[PEGR]{
\includegraphics[width=0.33\columnwidth]{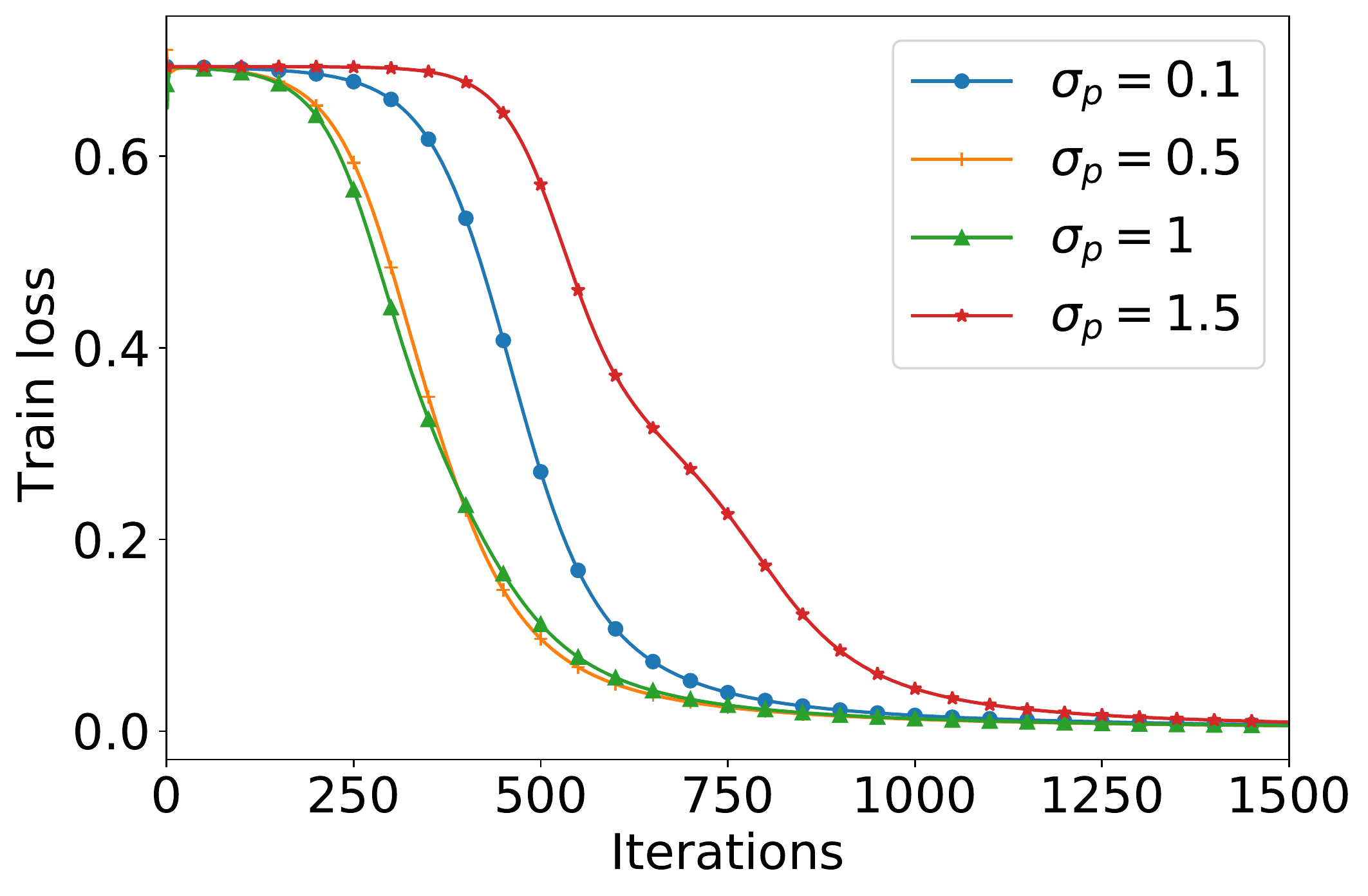}
}%
\subfigure[FGR]{
\includegraphics[width=0.33\columnwidth]{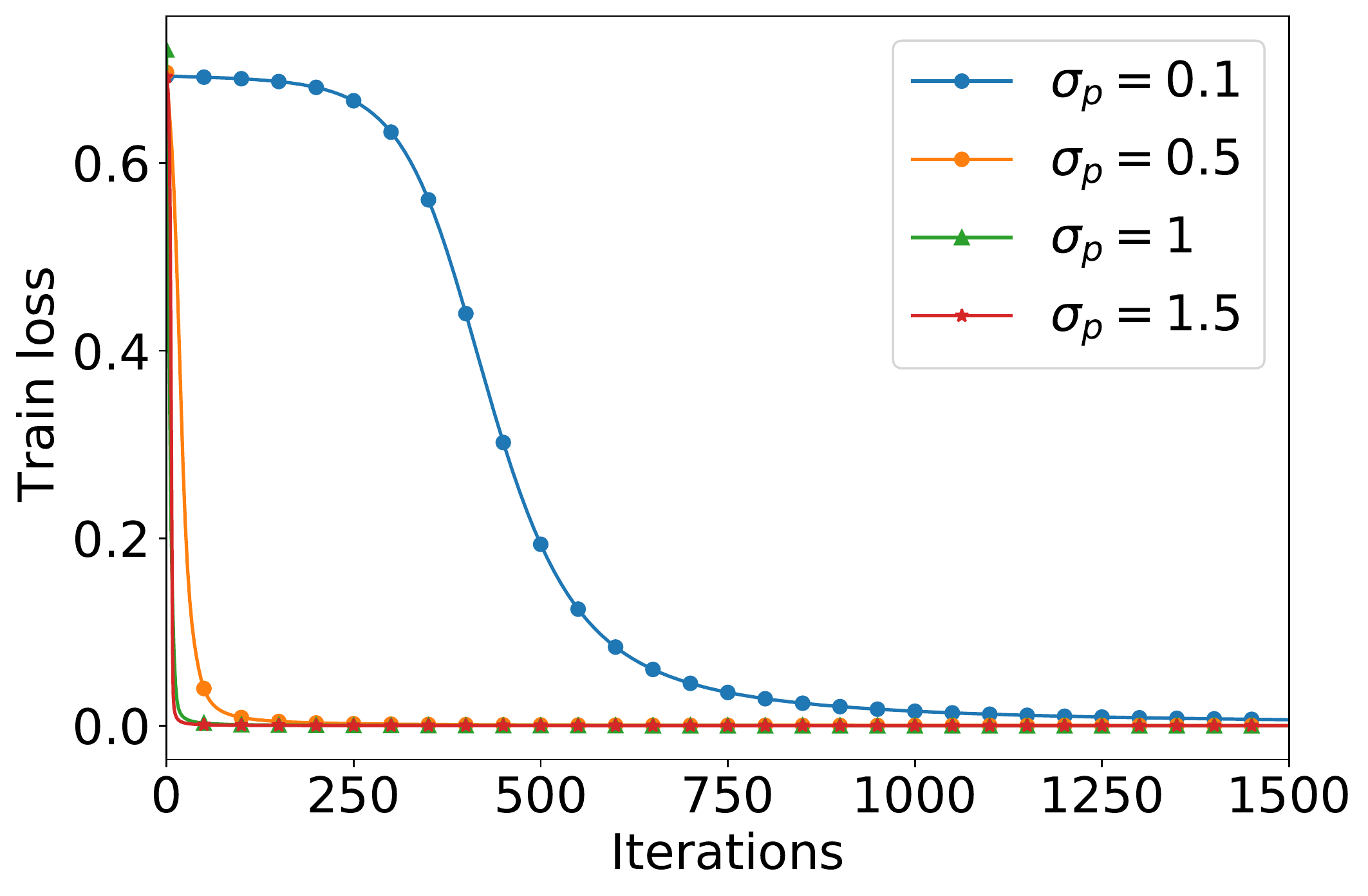}
}%
\subfigure[Standard]{
\includegraphics[width=0.33\columnwidth]{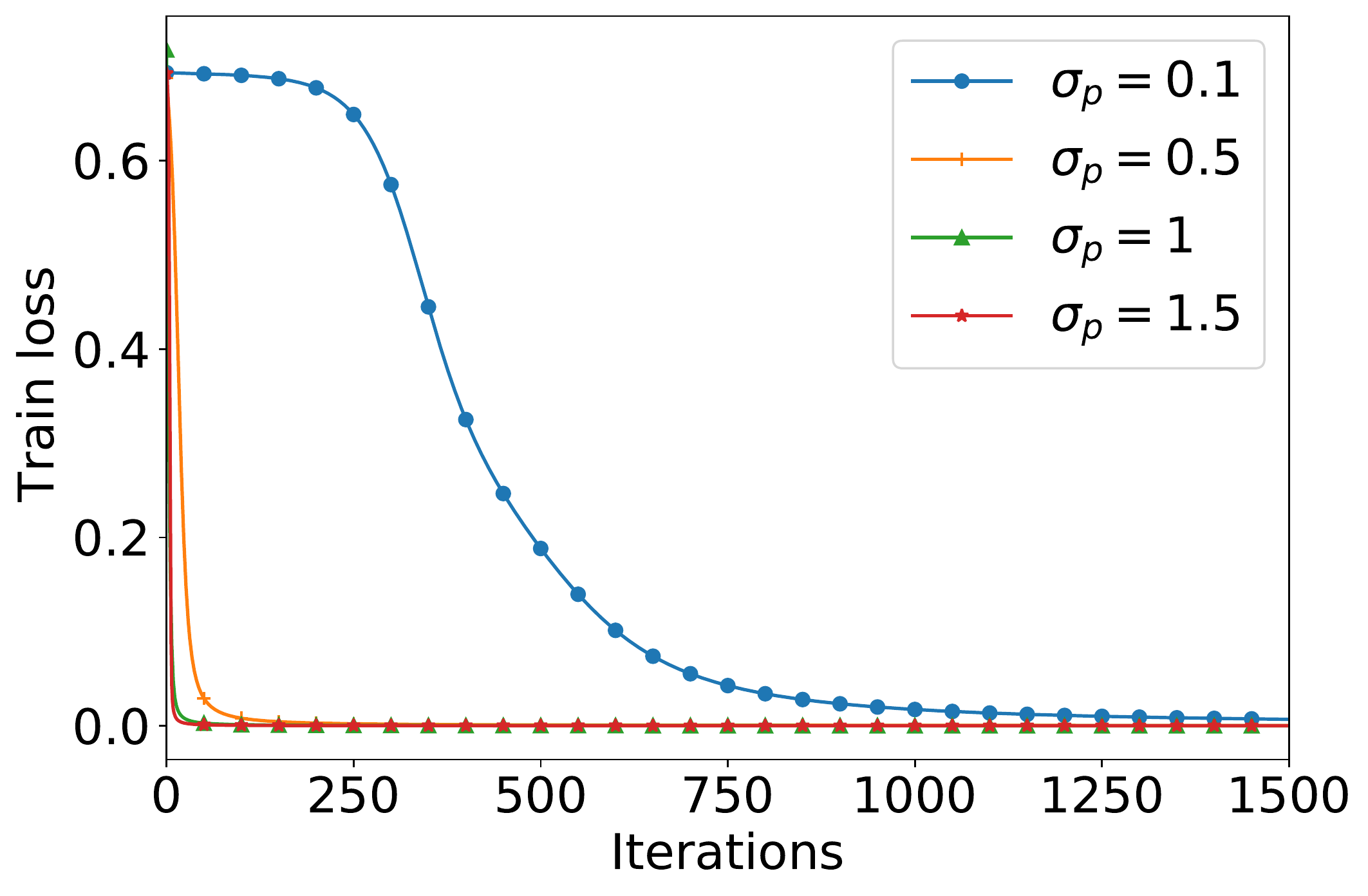}
}
\caption{~Training loss under different algorithms. (a) gives the training loss curve under per-example gradient regularization; (b) gives the training loss curve under full gradient regularization; (c) gives the training loss curve under standard training.\label{fig:trainloss}}
\end{figure}
 The training loss is denoted by $L_S{(\Wb^{(t)})}$. It's easy to observe that when training begins, the training loss is approximately $\log(2)\approx 0.69$. Theorems~\ref{thm:withregularization} and \ref{thm:withoutregularization} provide proofs for the convergence of the training loss. As depicted in Figure~\ref{fig:trainloss}, when we eliminate regularization, all training losses converge to $0$. \newline

\noindent\textbf{Signal Learning:}
\begin{figure}[t!]
\centering
\subfigure[PEGR]{
\includegraphics[width=0.33\columnwidth]{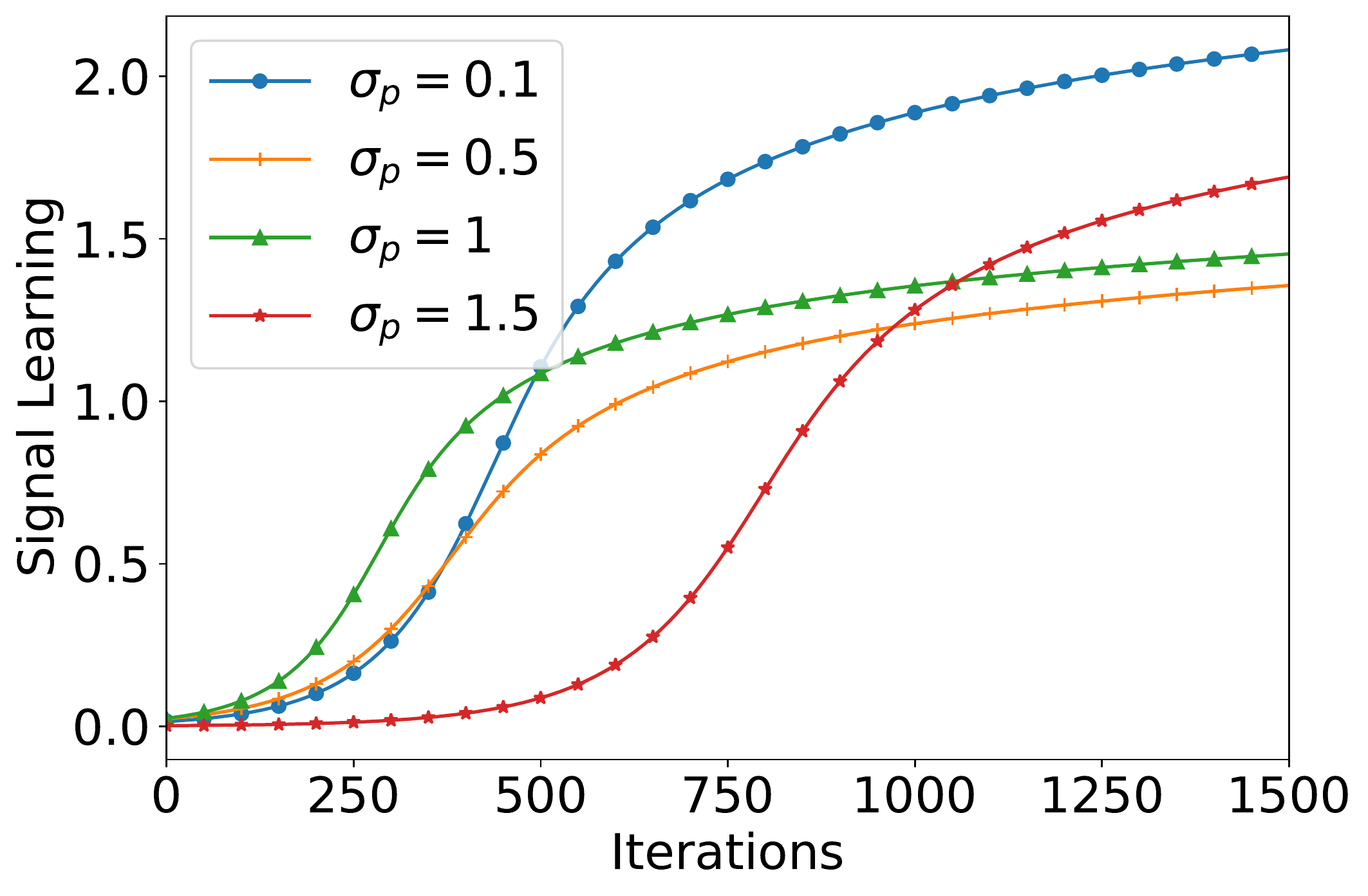}
}%
\subfigure[FGR]{
\includegraphics[width=0.33\columnwidth]{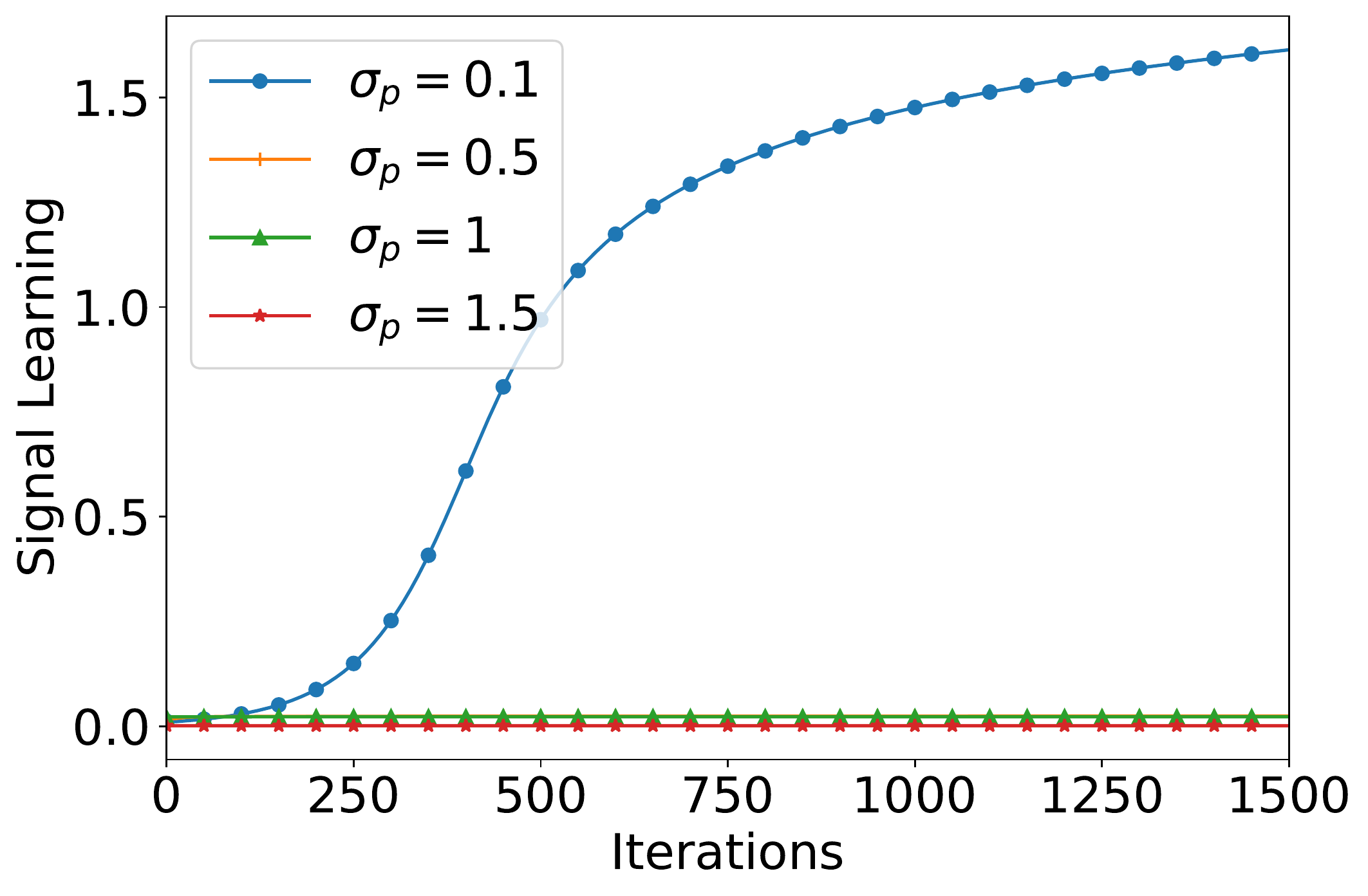}
}%
\subfigure[Standard]{
\includegraphics[width=0.33\columnwidth]{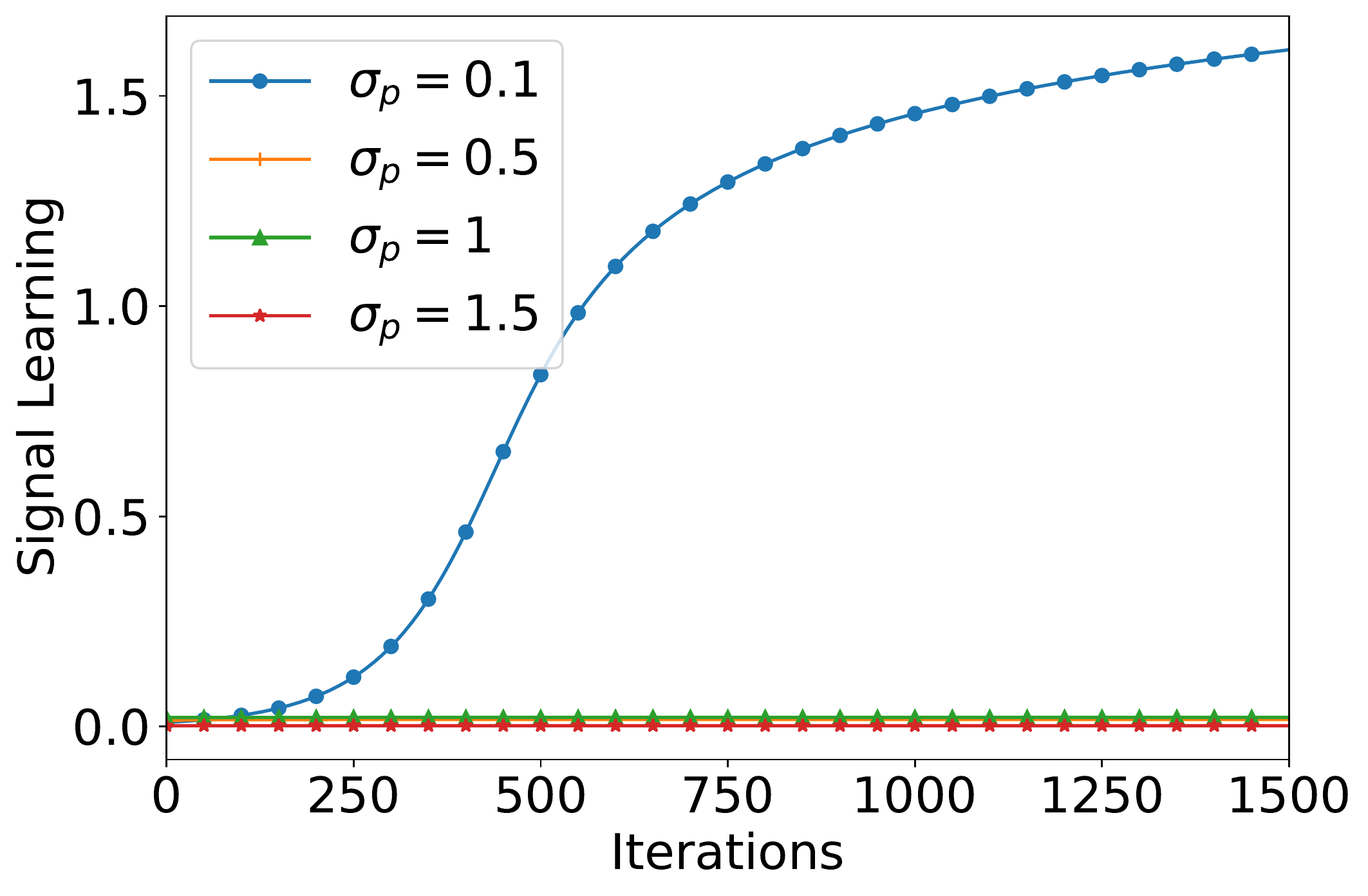}
}
\caption{~Signal learning where $y$-axis is $\max_{j,r}|\la \wb_{j,r},\bmu\ra|$ under different algorithms. (a) gives the signal learning curve under per-example gradient regularization; (b) gives the signal learning curve under full gradient regularization; (c) gives the signal learning curve under standard training.\label{fig:signal}}
\end{figure}
We define the signal learning by
$
    \text{Signal}=\max_{j,r}|\la \wb_{j,r},\bmu\ra|. 
$
The results shown in Figure~\ref{fig:signal} demonstrate that PEGR effectively promotes signal learning, while FGR and standard training fail to achieve the same results. It's important to note that when $\sigma_p=0.1$, signal learning increases in all cases due to the lower level of noise compared to when $\sigma_p=0.5, 1, 1.5$. In scenarios where the data SNR is high, standard training or FGR can effectively learn the signal. Additinally, it's worth noting that PEGR does not suppress the signal even when it is strong.\newline

\noindent\textbf{Noise Memorizing:}
\begin{figure}[t!]
\centering
\subfigure[PEGR]{
\includegraphics[width=0.33\columnwidth]{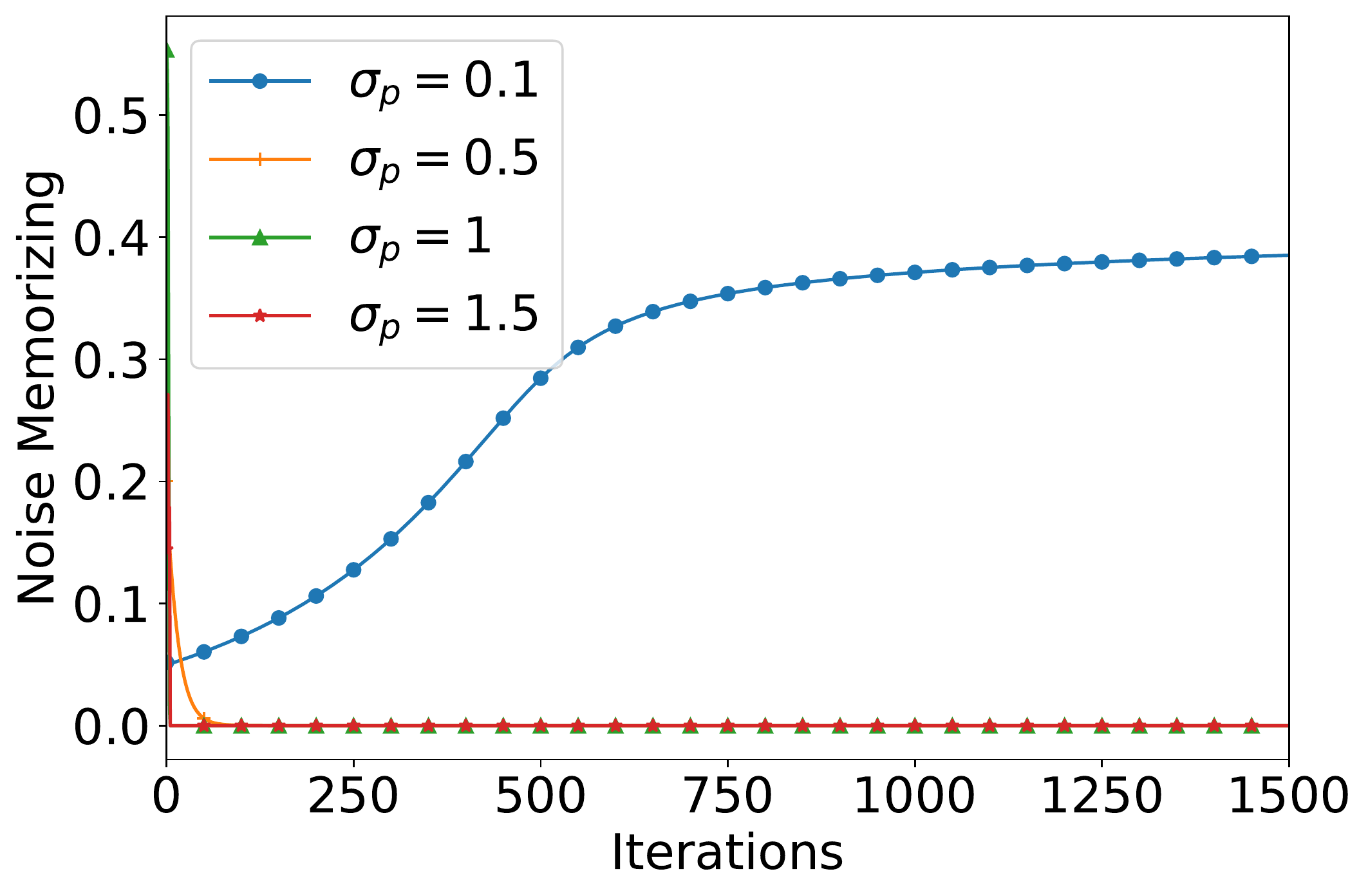}
}%
\subfigure[FGR]{
\includegraphics[width=0.33\columnwidth]{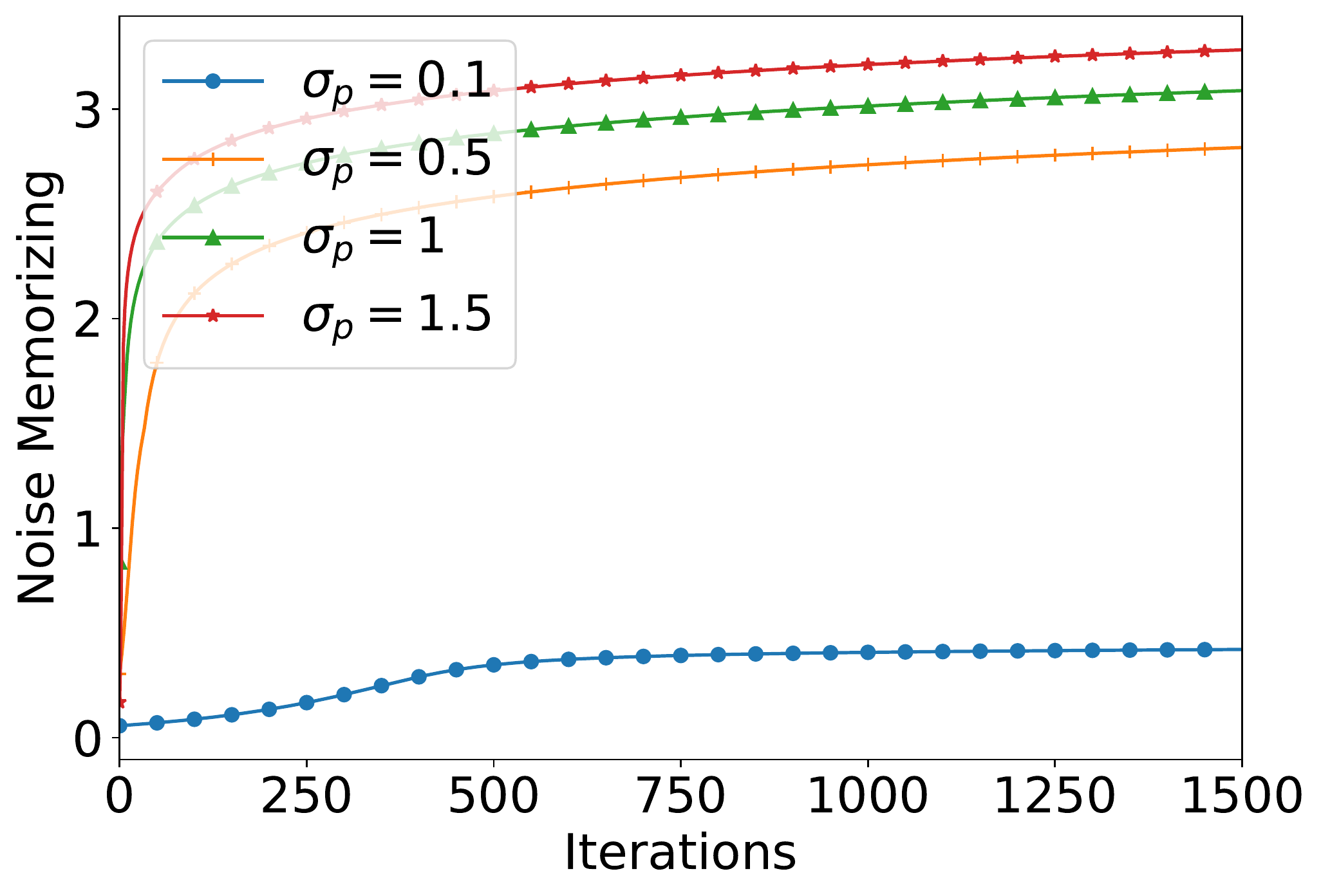}
}%
\subfigure[Standard]{
\includegraphics[width=0.33\columnwidth]{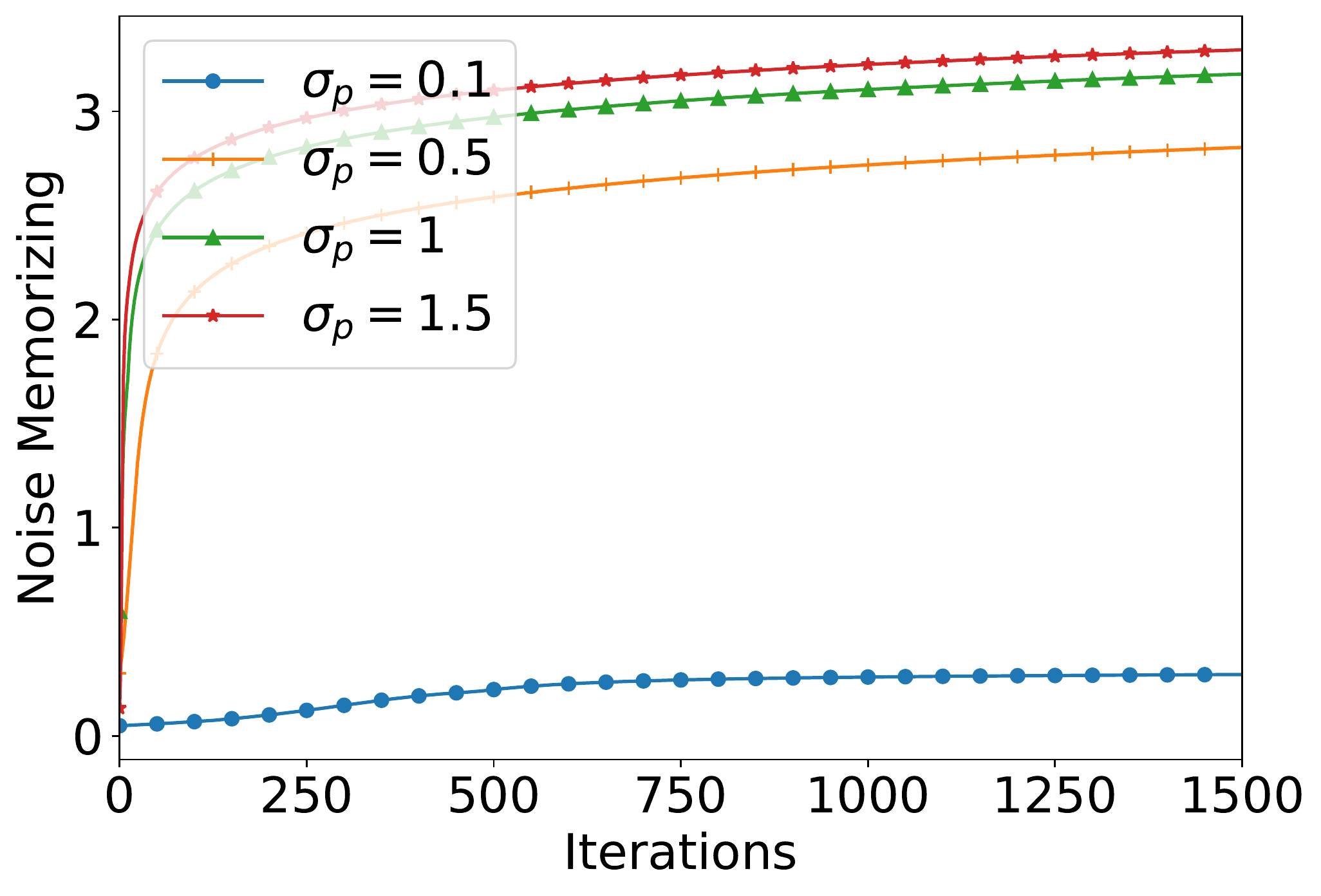}
}
\caption{~Noise memorizing where $y$-axis is $\max_{j,r,i}\la \wb_{j,r},\bxi_i\ra$ under different algorithms. (a) gives the noise  memorizing curve under per-example gradient regularization; (b) gives the noise  memorizing curve under full gradient regularization; (c) gives the noise  memorizing curve under standard training.\label{fig:noise}}
\end{figure}
We define the noise  memorizing by
$
    \text{Noise}=\max_{j,r,i}\la \wb_{j,r},\bxi_i\ra. 
$
The findings depicted in Figure~\ref{fig:noise} demonstrate that PEGR is highly effective at suppressing noise memorization, while  FGR and standard training  fail to produce the same outcomes. It is noteworthy that when $\sigma_p=0.1$, noise memorization increases in all cases, likely due to the low level of noise. Nevertheless, it is important to note that noise memorization does not impede signal learning when $\sigma_p=0.1$. In situations where the data SNR is low,  FGR or standard training  may not be capable of effectively learning the signal. In such scenarios, PEGR suppresses noise and facilitates signal learning. \newline

\noindent \textbf{Test Accuracy:}
\begin{figure}[t!]
\centering
\subfigure[PEGR]{
\includegraphics[width=0.33\columnwidth]{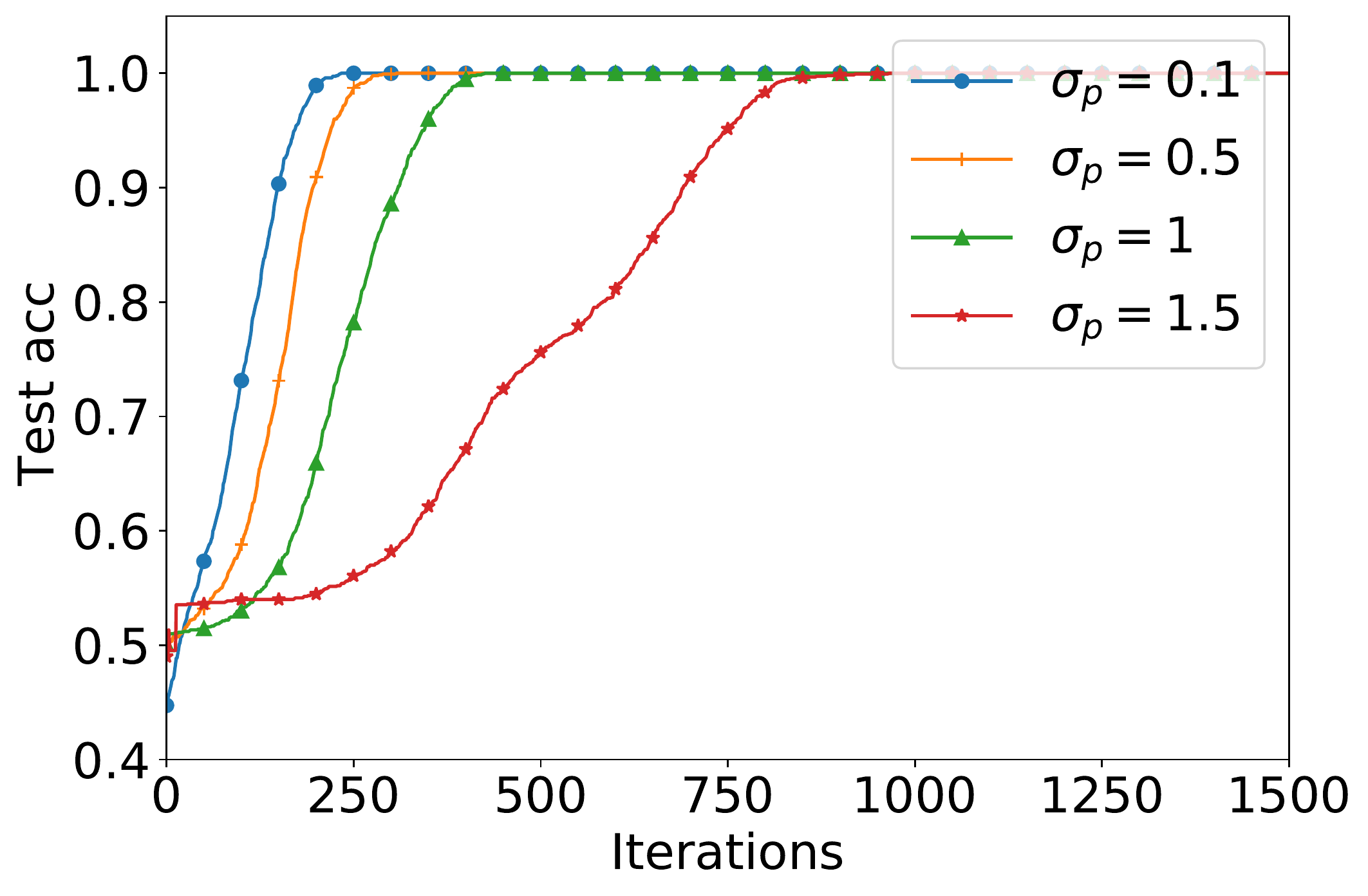}
}%
\subfigure[FGR]{
\includegraphics[width=0.33\columnwidth]{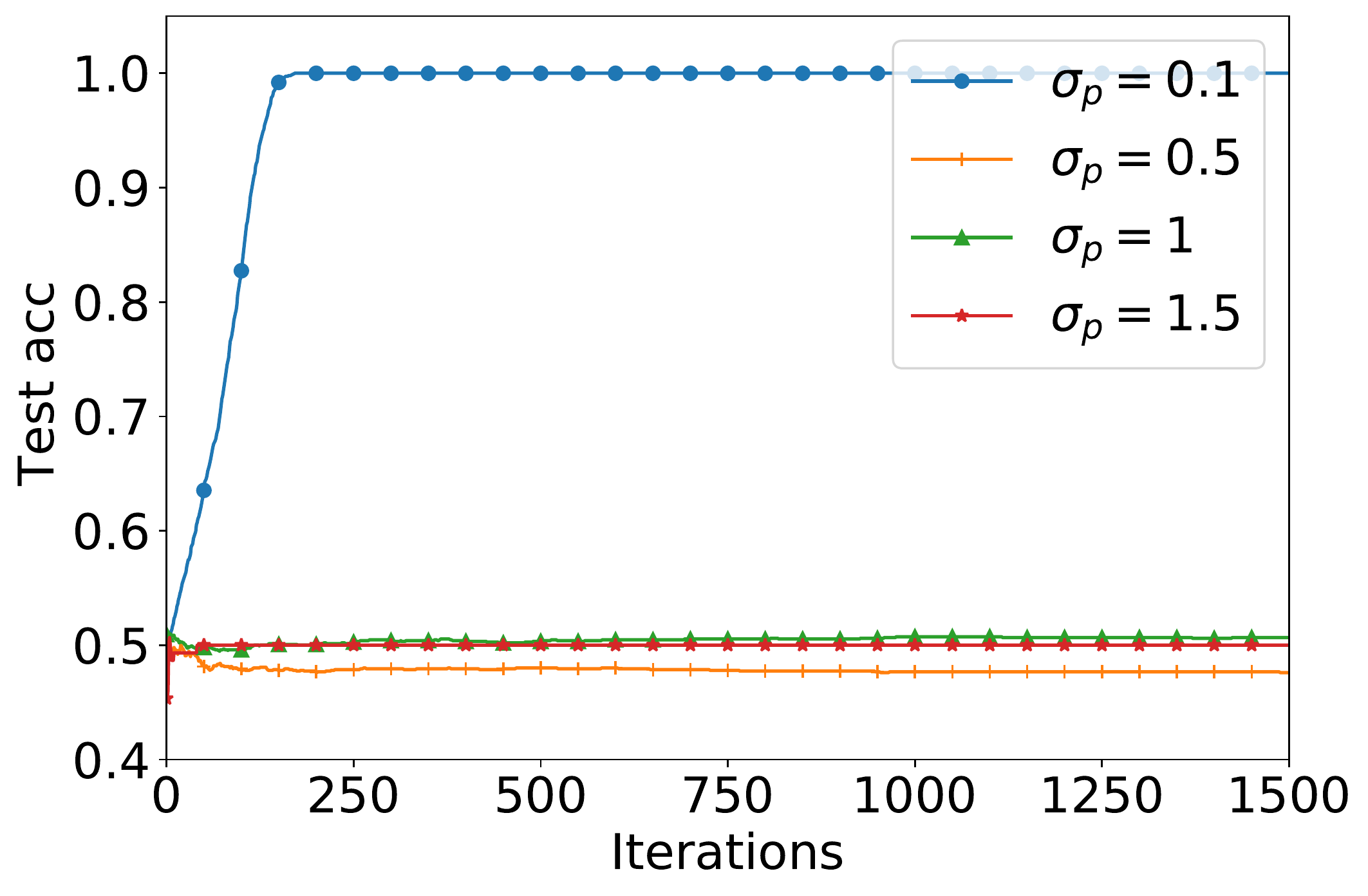}
}%
\subfigure[Standard]{
\includegraphics[width=0.33\columnwidth]{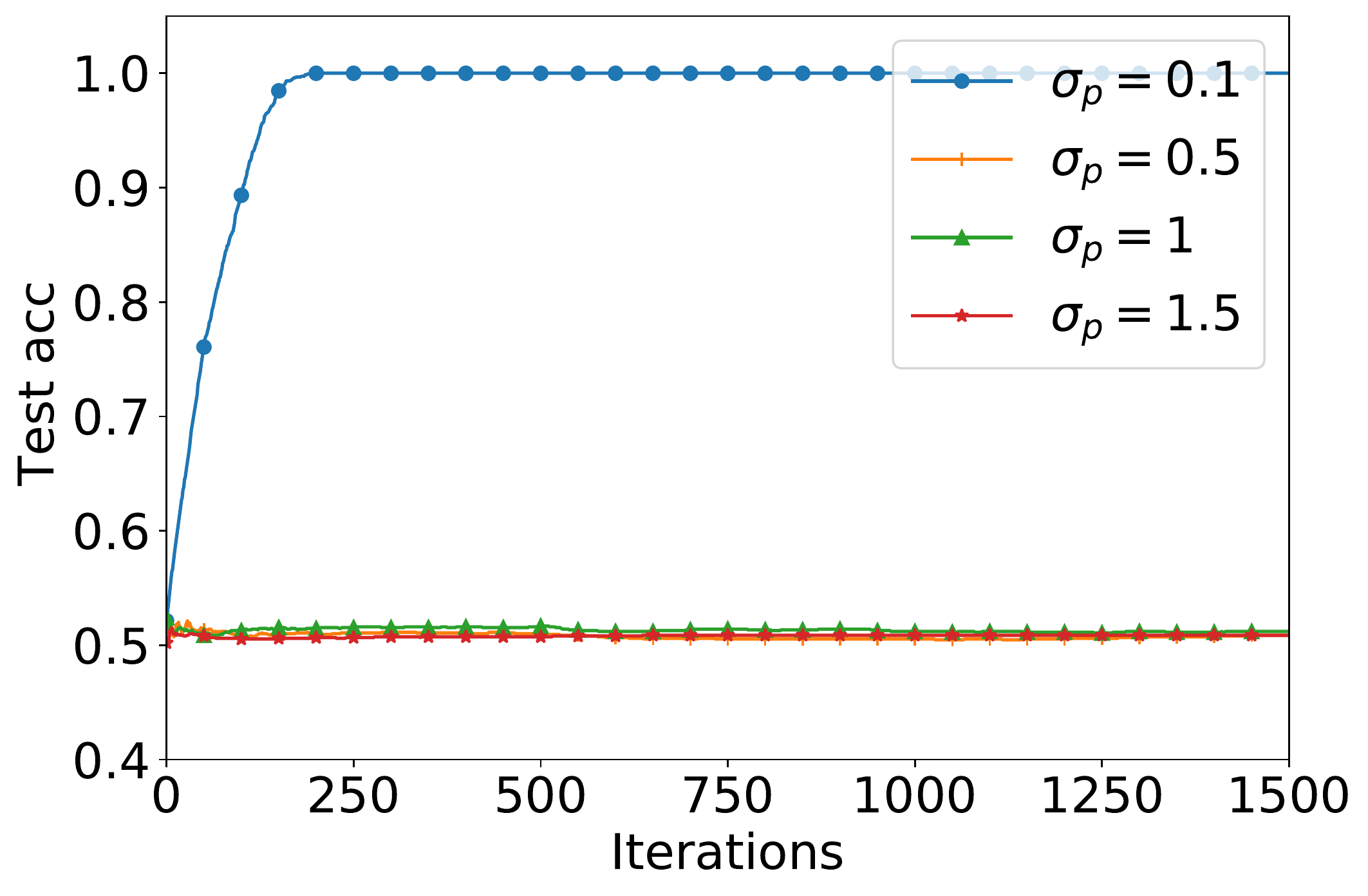}
}
\caption{~Test accuracy under different algorithms. (a) gives the test accuracy curve under per-example gradient regularization; (b) gives the test accuracy curve under full gradient regularization; (c) gives the test accuracy curve under standard training.\label{fig:testacc}}
\end{figure}
PEGR undoubtedly achieves the highest test accuracy by enhancing signal learning and suppressing noise memorization. Figure~\ref{fig:testacc} shows that the test accuracies in PEGR tend to $1$ under different values of $\sigma_p$, as proven in Theorem~\ref{thm:withregularization} and Theorem~\ref{thm:withoutregularization}. It is worth noting that, while all algorithms have a test accuracy tending to $1$ when $\sigma_p=0.1$ due to the high SNR, only PEGR has all test accuracy tend to $1$ when compared with FGR and standard training.

In summary, the numerical experiments show that per-example gradient regularization (PEGR) enhances signal learning while suppressing noise memorization compared to full gradient regularization (FGR) and standard learning approaches. Furthermore, the training loss of all methods converges to zero. Importantly, the test accuracy achieved by PEGR is significantly higher than the test accuracy obtained by the other algorithms. These findings provide strong empirical evidence in support of our theoretical demonstration, indicating that PEGR holds promise as an effective approach for improving learning performance in neural networks.

\subsection{Real data experiments}
\label{subsec:realdataexperiment}
In this section, we report on real data experiments conducted on the MNIST dataset using Lenet5 as the model architecture to assess the effectiveness of PEGR in suppressing noise memorization and enhancing signal learning. To achieve this, we added standard noise with a mean of 0 and varied the standard deviation of the noise, $\sigma_p$, to 0, 5, 10, and 15, respectively. The test data was kept free from noise to evaluate how well the signal was learned. Our experiments were conducted with the following parameters: a learning rate $\eta$ of 0.01, initialization $\sigma_0$ of 0.1, batch size of 64, PEGR closed after 25 epochs, and a tuning parameter of $\lambda=0.1$. The training loss and test accuracy results are presented below.\newline

\noindent\textbf{Training loss:}
\begin{figure}[t!]
\centering
\subfigure[$\sigma_p=0$]{
\includegraphics[width=0.25\columnwidth]{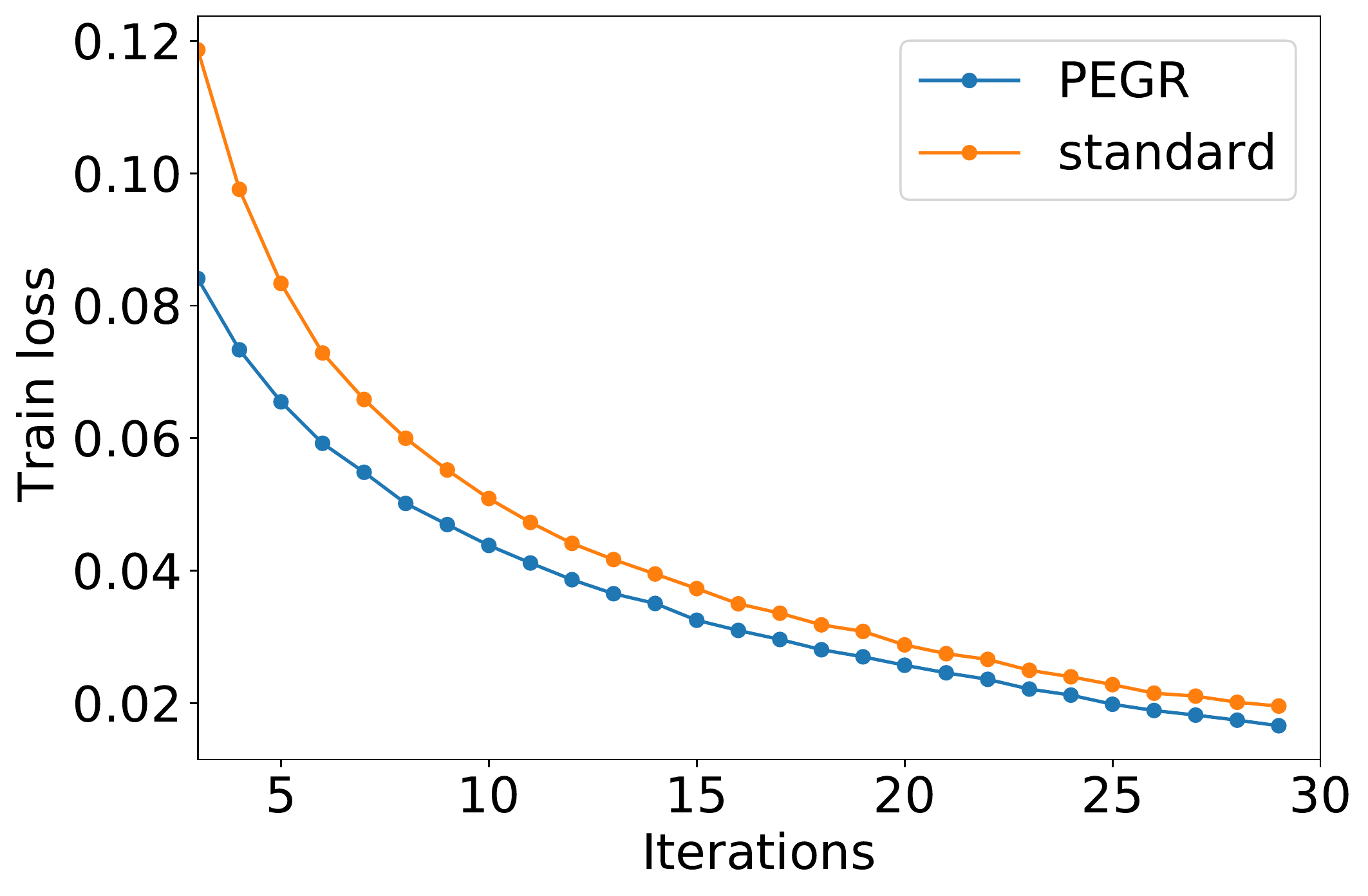}
}%
\subfigure[$\sigma_p=5$]{
\includegraphics[width=0.25\columnwidth]{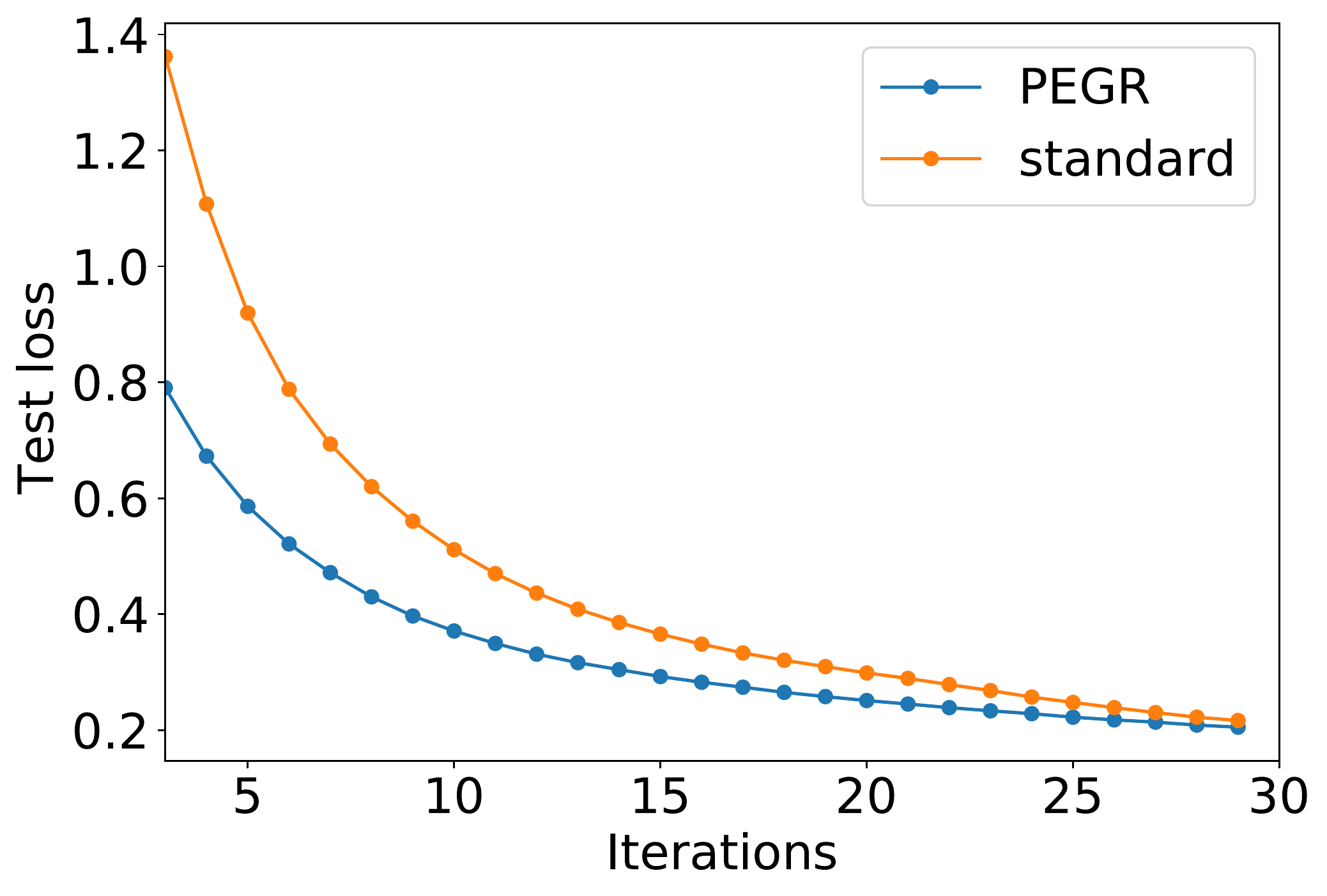}
}%
\subfigure[$\sigma_p=10$]{
\includegraphics[width=0.25\columnwidth]{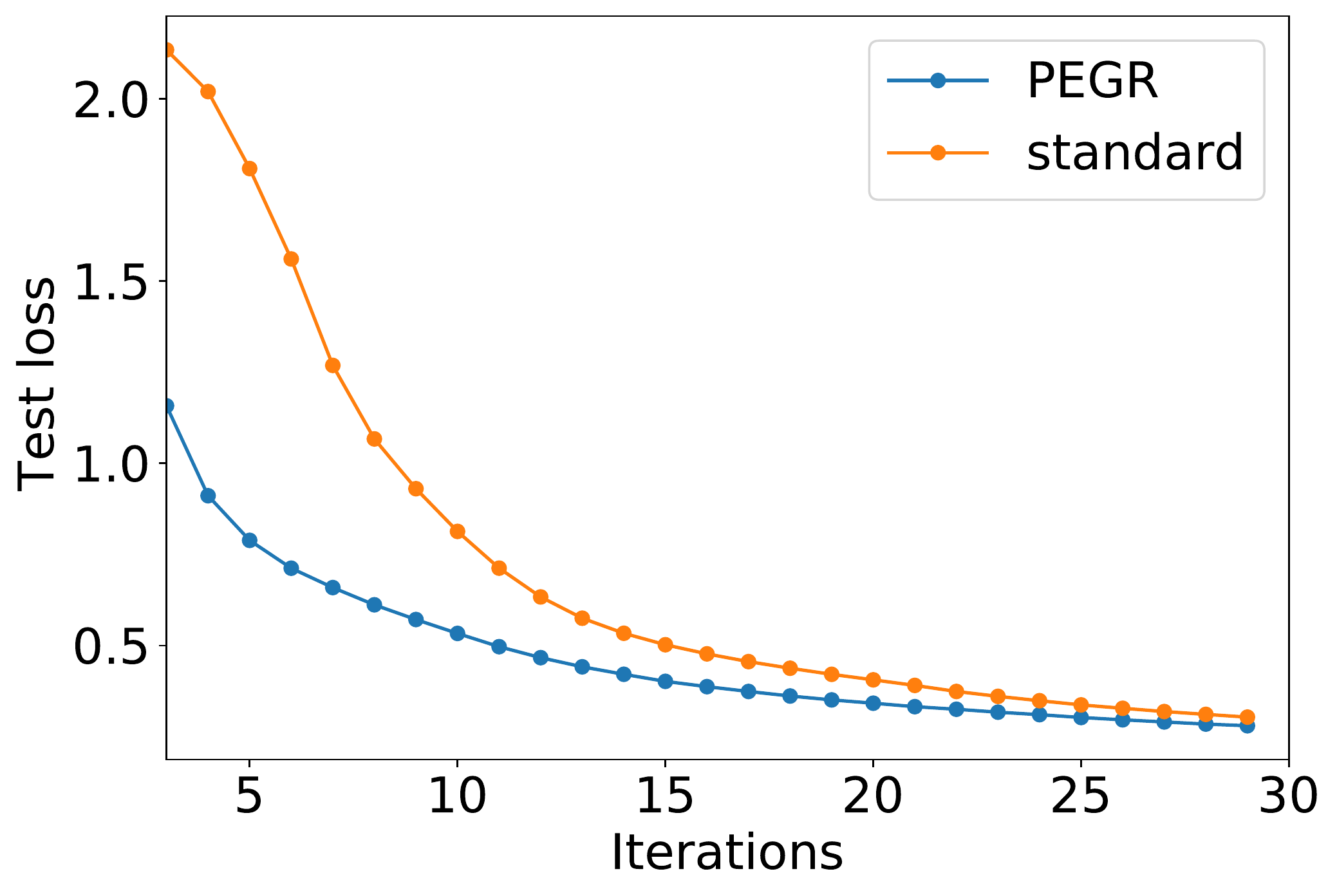}
}%
\subfigure[$\sigma_p=15$]{
\includegraphics[width=0.25\columnwidth]{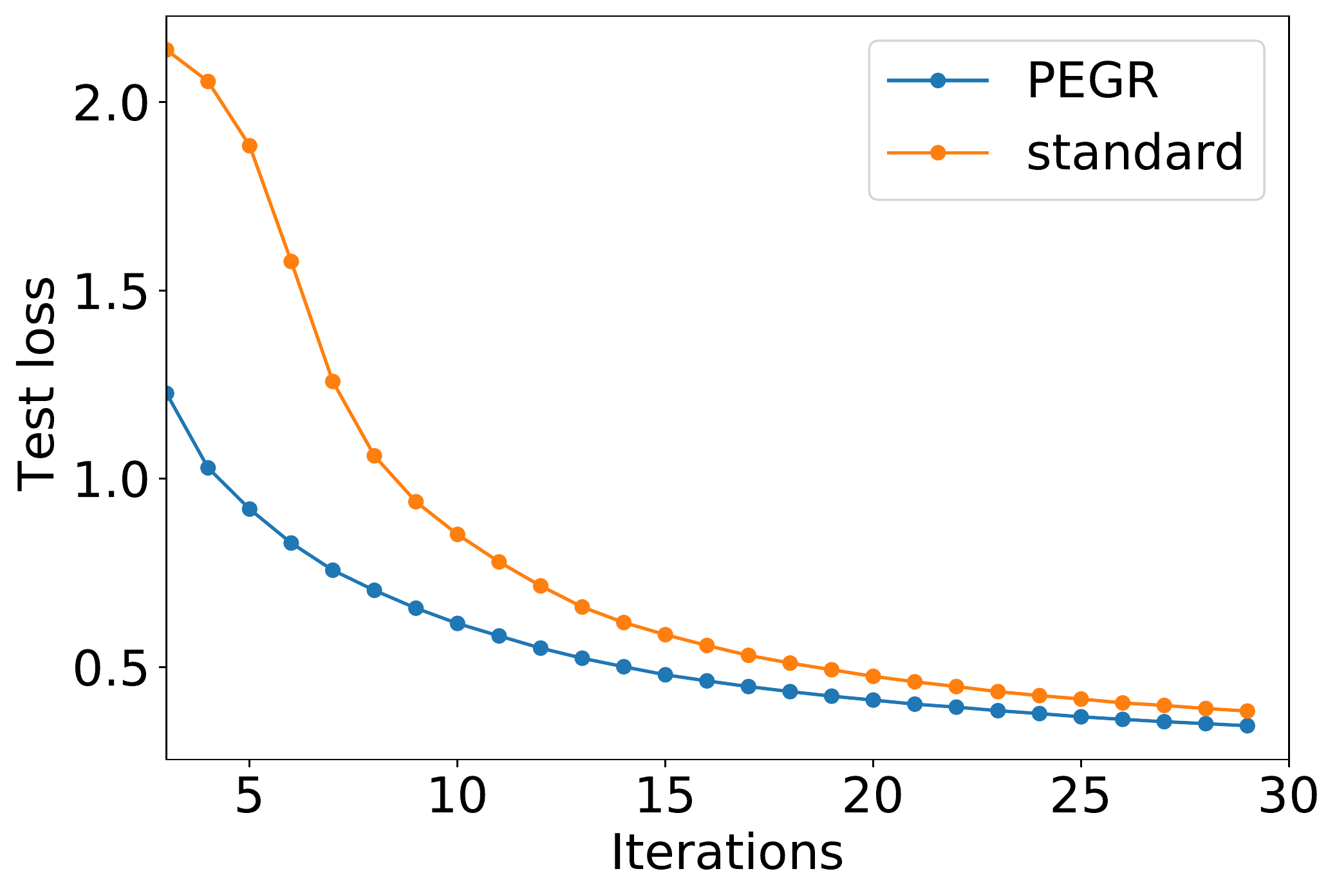}
}
\caption{~Training loss under different $\sigma_p$. (a) gives the training loss curve under $\sigma_p=0$; (b) gives the training loss curve under $\sigma_p=5$; (c) gives the training loss curve under $\sigma_p=10$; (d) gives the training loss curve under $\sigma_p=15$.\label{fig:trueloss}}
\end{figure}
As dipicted in Figure~\ref{fig:trueloss}, the training loss for both PEGR and standard training approaches show a consistent decrease with increasing iterations, implying that the training process is stable and allowing us to evaluate the efficacy of PEGR in real data training. The observed stable training status allows for reliable comparison of the two approaches and enables us to assess the potential of PEGR in enhancing signal learning in real-world data sets.\newline

\noindent \textbf{Test accuracy:}
\begin{figure}[t!]
\centering
\subfigure[$\sigma_p=0$]{
\includegraphics[width=0.25\columnwidth]{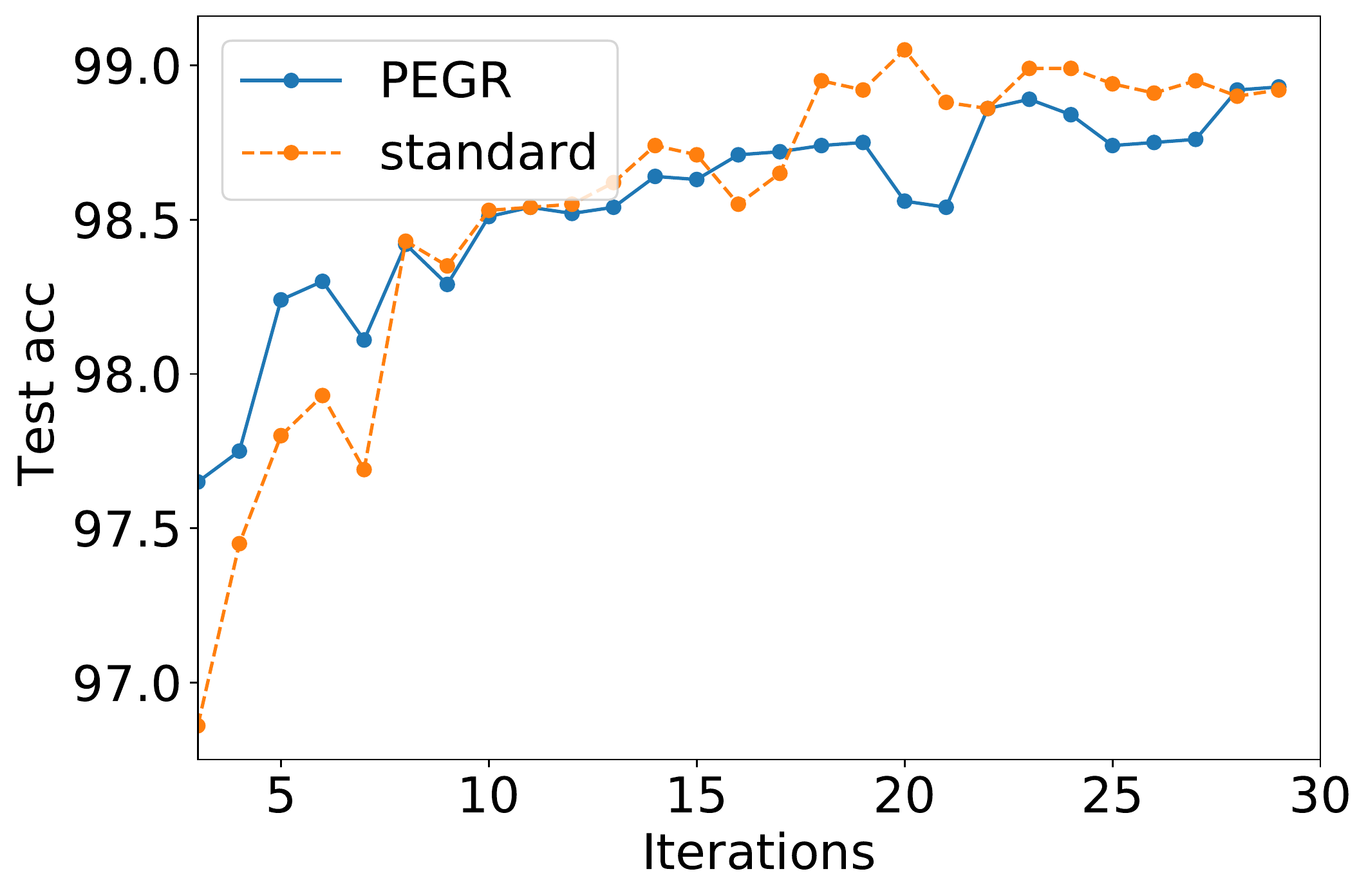}
}%
\subfigure[$\sigma_p=5$]{
\includegraphics[width=0.25\columnwidth]{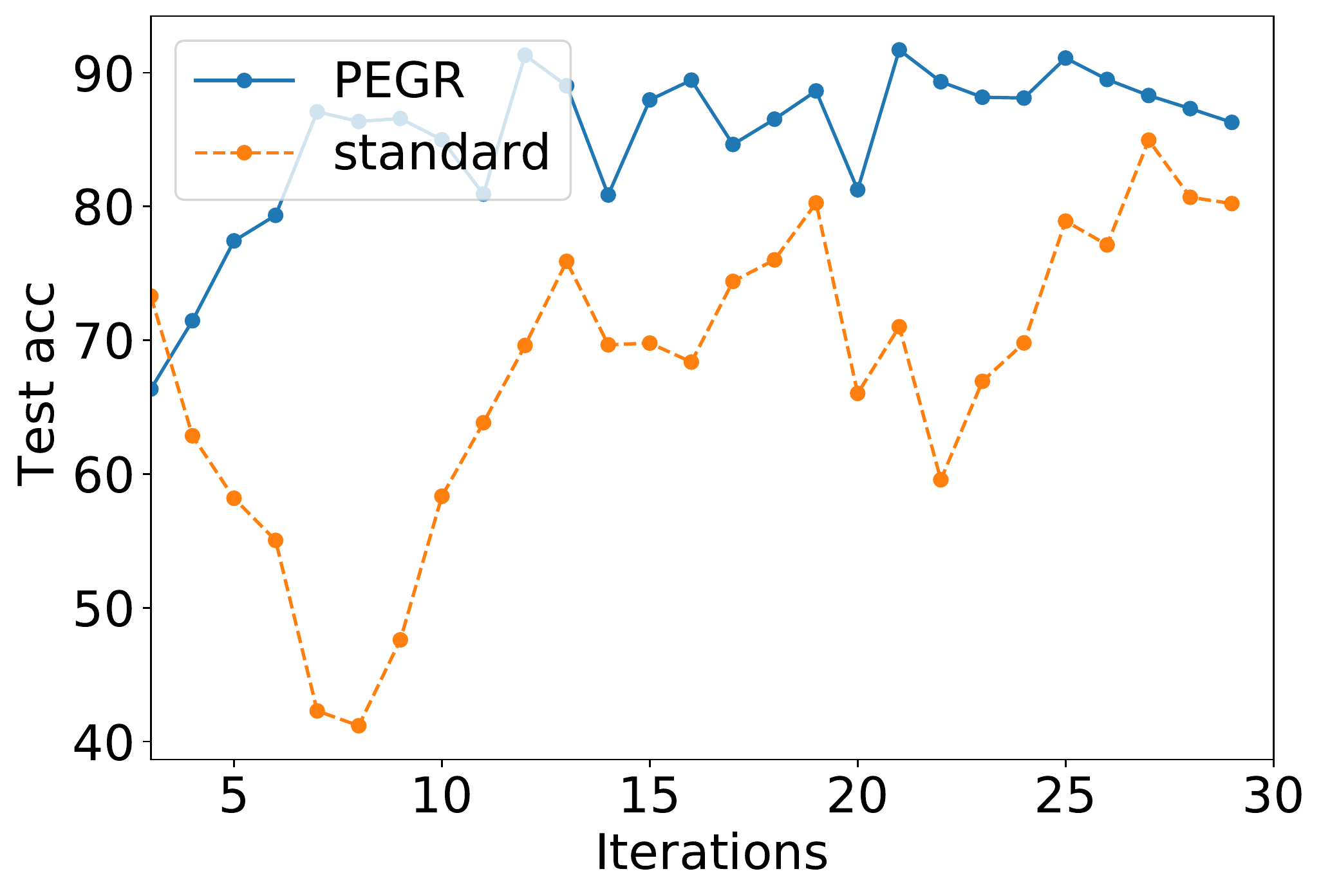}
}%
\subfigure[$\sigma_p=10$]{
\includegraphics[width=0.25\columnwidth]{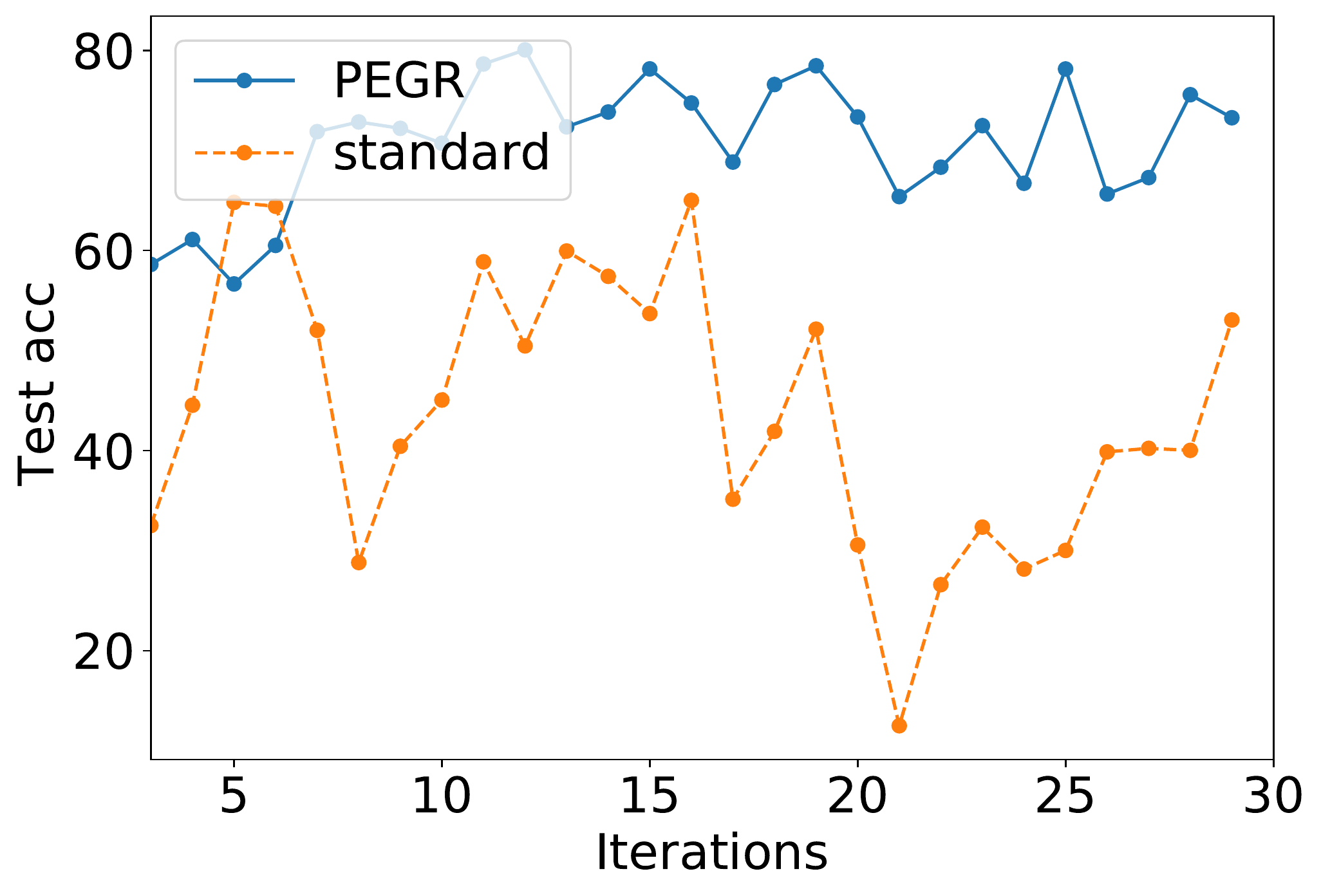}
}%
\subfigure[$\sigma_p=15$]{
\includegraphics[width=0.25\columnwidth]{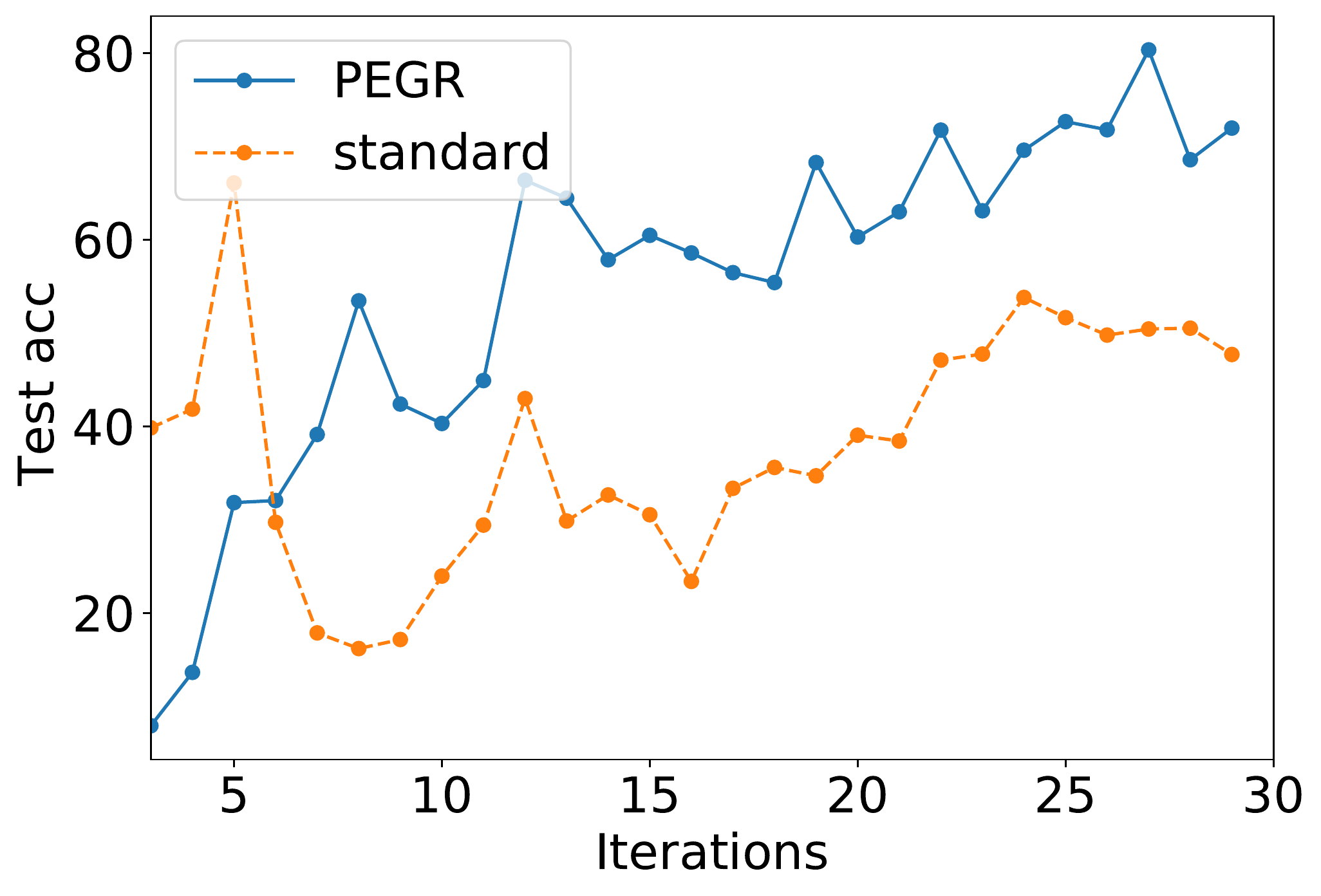}
}
\caption{~Test accuracy under different $\sigma_p$. (a) gives the test accuracy curve under $\sigma_p=0$; (b) gives the test accuracy curve under $\sigma_p=5$; (c) gives the test accuracy curve under $\sigma_p=10$; (d) gives the test accuracy curve under $\sigma_p=15$.\label{fig:trueacc}}
\end{figure}
As shown in Figure~\ref{fig:trueacc}, PEGR demonstrates its ability to enhance signal learning during real data training by yielding higher test accuracy compared to standard training under different values of $\sigma_p$ and training algorithms. The decrease in test accuracy for both approaches as the value of $\sigma_p$ increases indicates that noise obscures the signal, but PEGR is still able to distinguish the signal from noise and gradually increase the test accuracy. Moreover, as shown in Figure~\ref{fig:trueacc} (a), PEGR does not prevent signal learning when there is no noise. Hence, these results suggest that PEGR can be effectively applied to improve the generalization of deep neural networks in real-world settings where training data sets contain significant noise.

In conclusion, our real data experiments confirm the effectiveness of PEGR in enhancing signal learning. As the level of noise in the training data increases, both PEGR and standard training approaches achieve a small training loss. However, PEGR is able to separate the signal from the noise better than standard training, resulting in higher test accuracy. These results, combined with our theoretical guarantees in Theorems~\ref{thm:withregularization} and \ref{thm:withoutregularization}, suggest that the PEGR technique can be widely applied when training data sets contain significant noise and the signal is obscured.

\section{Overview of Proof Technique}
\label{sec:overviewproof}
In this section, we present the main technique in the study of CNN training under our setting. The maximum admissible iterations we set in the paper is $T^*=\eta^{-1}\poly({n,m,d,\varepsilon^{-1},\sigma_0^{-1},\sigma_p\sqrt{d}})$.  The complete proofs of all the results are given in the appendix.

\subsection{Why PEGR works better than standard training?}
\label{subsec:whyPEGR}
In this section, we give the equation of the gradient, and briefly discuss about why PEGR works better than standard training.
We further denote by $\zeta_i^{(t)}=\sum_{j'}\sum_{r'=1}^m\big(\ReLU^2( \la \wb_{j',r'}^{(t)},  y_i\bmu \ra )\|\bmu\|^2+\ReLU^2( \la \wb_{j',r'}^{(t)}, \bxi_i \ra )\|\bxi_i\|^2\big)$, $\ell'^{(t)}_{i}={\ell'}[y_{i}\cdot f(\Wb^{(t)},\xb_{i})],~\ell''^{(t)}_{i}={\ell''}[y_{i}\cdot f(\Wb^{(t)},\xb_{i})]$, and show the update rule in the following lemma.
\begin{lemma}
\label{lemma:update_rule}
Given $\tilde{L}(\Wb)$ above, we have
\begin{align*}
    &\wb_{j,r}^{(t+1)}=\wb_{j,r}^{(t)}-\frac{2\eta}{nm}\sum_{i=1}^{n}{\ell'^{(t)}_i}\bigg(1+\frac{4\lambda\zeta_i^{(t)}}{m}\cdot\ell''^{(t)}_i\bigg)\cdot jy_i \big(\ReLU ( \la\wb_{j,r}^{(t)}, y_i\bmu \ra)y_i\bmu+\ReLU (\la \wb_{j,r}^{(t)}, \bxi_i \ra) \bxi_i\big)\\
    &\qquad-\frac{4\lambda\eta}{nm^2}\sum_{i=1}^{n}{\ell'^{2(t)}_i}\big(\ReLU  ( \la \wb_{j,r}^{(t)}, y_i\bmu \ra)\|\bmu\|^2 y_i\bmu+\ReLU  ( \la \wb_{j,r}^{(t)}, \bxi_i \ra)\|\bxi_i\|^2\bxi_i\big).
\end{align*}
\end{lemma}
The proof for Lemma~\ref{lemma:update_rule} can be found in Appendix~\ref{subsec:proofPEGR}. With the update rule established, we can now proceed to analyze the disparity in test error between cases with and without gradient regularization. By using Lemma~\ref{lemma:update_rule}, we can observe how per-example gradient regularization affects the learning of both noise and feature. Specifically, in the early stages of training where $|\wb_{j,t}^{(t)}|\ll 1$, we can ignore the high-order terms $\zeta_i^{(t)}$, and we can approximate $\ell_i^{(t)}\sim -0.5$ and $\ell_{i}'^{2(t)}\sim 1/4$. Moreover,  we assume that $\bxi_i$ is nearly orthogonal to all other noise vectors. These simplifications allow us to discuss the training with PEGR and the standard training method separately.

In standard training, we set $\lambda=0$, and the update rule becomes 
    \begin{align*}
        \wb_{j,r}^{(t+1)}=\wb_{j,r}^{(t)}-\frac{2\eta}{nm}\sum_{i=1}^{n}{\ell'^{(t)}_i}  \big(\ReLU ( \la\wb_{j,r}^{(t)}, y_i\bmu \ra)j \bmu+\ReLU (\la \wb_{j,r}^{(t)}, \bxi_i \ra) jy_i\bxi_i\big).
    \end{align*}
    If the noise is significantly stronger than the signal, the signal and noise  can grow at a rate of $\Theta(\eta\|\bmu\|^2)$ and $\Theta(\eta\|\bxi_i\|^2/n)$ respectively, indicating a faster rate of noise memorization. As training progresses, the noise memorization grows significantly faster than the signal learning, leading to a rapid decrease in the loss function. Consequently, the signal learning will no longer improve despite the small loss.
    
Similar in the PEGR training, we also ignore the higher-order terms $\zeta_i^{(t)}$, and the update rule can be approximated by
\begin{align*}
&\wb_{j,r}^{(t+1)}\approx\wb_{j,r}^{(t)}-\frac{2\eta}{nm}\sum_{i=1}^{n}{\ell'^{(t)}_i}\big(\ReLU ( \la\wb_{j,r}^{(t)}, y_i\bmu \ra)j\bmu+\ReLU (\la \wb_{j,r}^{(t)}, \bxi_i \ra) jy_i\bxi_i\big)\\
    &\qquad-\frac{4\lambda\eta}{nm^2}\sum_{i=1}^{n}{\ell'^{2(t)}_i}\big(\ReLU  ( \la \wb_{j,r}^{(t)}, y_i\bmu \ra)\|\bmu\|^2 y_i\bmu+\ReLU  ( \la \wb_{j,r}^{(t)}, \bxi_i \ra)\|\bxi_i\|^2\bxi_i\big).
\end{align*}
The per-example gradient regularization suppresses the learning of both features and noises, driving them back to zero at suppression speeds approximately $\lambda\cdot\Theta(\eta\|\bmu\|^4)$ and $\lambda\cdot\Theta(\eta\|\bxi_i\|^4/n)$, respectively. Recall the growing rate $\Theta(\eta\|\bmu\|^2)$ and $\Theta(\eta\|\bxi_i\|^2/n)$, we can tune $\lambda$ to make the suppression speeds $\Theta(\eta\|\bmu\|^4/\|\bxi_i\|)$ and $\Theta(\eta\|\bxi_i\|^3/n)$, which indicates that the regularization imposes a stronger suppression on noise memorization while still promoting signal learning. The analysis presented above suggests that feature learning can outperform noise memorization when the signal-to-noise ratio (SNR) is small. In contrast, standard training is likely to fail in such a low-SNR regime. By appropriately adjusting the regularization parameter, we can enhance the suppression of noise memorization and promote the learning of features, which can lead to improved performance in low-SNR scenarios.


In fact, the reason behind this is that the PEGR implicitly controls the variance of the gradients when using different training data points. As a consequence, the network is encouraged to make use of the \textbf{signal component} rather than the \textbf{noise component} to fit the training data, since the signal component, which is shared for all data in the same class, has a significantly smaller variance than  that of the noise components, which vary for different data.

\subsection{Analysis in the maximum admissible iterations.}

In this section, we give the analysis in the maximum admissible iterations $T^*$.
Since the vectors $\bmu$ and $\bxi_i$, where $i\in[n]$, are linearly independent with probability 1, we can introduce the following definition based on the gradient descent update rule in Lemma~\ref{lemma:update_rule}. 
\begin{definition}
\label{def:signal_to_noise_decomp}
Let $\wb_{j,r}^{(t)} $ for $j\in\{\pm1\}$, $r\in[m]$ be the convolution filter of the CNN at the $t$-th iteration of gradient descent. Then there exists unique coefficients $\gamma_{j,r}^{(t)}$ and $\rho_{j,r,i}^{(t)}$ such that
\begin{align}
\label{eq:signal_to_noise_decomp}
    \wb_{j,r}^{(t)}&=\wb_{j,r}^{(0)}+j\cdot\gamma_{j,r}^{(t)}\frac{\bmu}{\|\bmu\|^2}+\sum_{i=1}^n\rho_{j,r,i}^{(t)}\frac{\bxi_i}{\|\bxi_i\|^2}.
\end{align}
\end{definition}
Easy to see $\gamma_{j,r}^{(0)}=\rho_{j,r,i}^{(0)}=0$. We refer to \eqref{eq:signal_to_noise_decomp} as the signal-noise decomposition of $\wb_{j,r}^{(t)}$. The uniqueness  can be found in the proof of Lemma~\ref{lemma:decomposition_coef_dynamic_appendix} in Appendix~\ref{sec:update_rule}. In this decomposition, $\gamma_{j,r}^{(t)}$ characterizes the progress of learning the signal vector $\bmu$, and $\rho_{j,r,i}^{(t)}$ depicts the degree of noise memorization by the filter. Based on the decomposition, if some $\gamma_{j,r}^{(t)}$  is large enough while $|\rho_{j,r,i}^{(t)}|$ is small, then the CNN will achieve small training gradient and test error, if some $|\rho_{j,r,i}^{(t)}|$  is large enough while all $\gamma_{j,r}^{(t)}$ are small, then the CNN will achieve a small training gradient but a large test error. Thus, Definition~\ref{def:signal_to_noise_decomp} provides a method for us to study the convergence of the training procedure and the property of the test error with or without gradient regularization under SGD.
Note that the dynamic analysis is related to a non convex optimization problem, and the key technique is the signal-noise decomposition in Definition~\ref{def:signal_to_noise_decomp}, which is further investigated in the following lemma.

\begin{lemma}
\label{lemma:decomposition_coef_dynamic}
The coefficients $\gamma_{j,r}^{(t)},\rho_{j,r,i}^{(t)}$ in Definition~\ref{def:signal_to_noise_decomp} satisfy the following equations:
\begin{align*}
\zeta_i^{(t)}&=\sum_{j'}\sum_{r'=1}^m\big(\ReLU^2( \la \wb_{j',r'}^{(t)},  y_i\bmu \ra )\|\bmu\|^2+\ReLU^2( \la \wb_{j',r'}^{(t)}, \bxi_i \ra )\|\bxi_i\|^2\big),\\
    \gamma_{j,r}^{(t+1)}&=\gamma_{j,r}^{(t)}-\frac{2\eta}{nm}\sum_{i=1}^{n}\bigg[{\ell'^{(t)}_i}\bigg(1+\frac{4\lambda\zeta_i^{(t)}}{m}\cdot\ell''^{(t)}_i\bigg) \ReLU ( \la\wb_{j,r}^{(t)}, y_i\cdot\bmu \ra)\|\bmu\|^2\\
    &\quad+\frac{  2\lambda}{m}{\ell'^{2(t)}_i}\ReLU  ( \la \wb_{j,r}^{(t)}, y_i\cdot\bmu \ra)\|\bmu\|^4\cdot jy_i\bigg],\\
    \rho_{j,r,i}^{(t+1)}&=\rho_{j,r,i}^{(t)}-\frac{2\eta}{nm}\bigg[{\ell'^{(t)}_i}\bigg(1+\frac{4\lambda \zeta_i^{(t)}}{m}\cdot\ell''^{(t)}_i\bigg) \ReLU ( \la\wb_{j,r}^{(t)}, \bxi_i \ra)\|\bxi_i\|^2\cdot jy_i\\
    &\quad+\frac{  2\lambda}{m}{\ell'^{2(t)}_i}\ReLU  ( \la \wb_{j,r}^{(t)}, \bxi_i \ra)\|\bxi_i\|^4\bigg],
\end{align*}
the initialization condition is $\gamma_{j,r}^{(0)}=\rho_{j,r,i}^{(0)}=0$.
\end{lemma}
Lemma~\ref{lemma:decomposition_coef_dynamic} gives the update rule of the coefficients $\gamma_{j,r}^{(t)}$ and $\rho_{j,r,i}^{(t)}$, which enables us further analyze the training process. We further define  
\begin{align*}
    \overrho_{j,r,i}^{(t)}=\rho_{j,r,i}^{(t)}\one\{j=y_i \},\quad\underrho_{j,r,i}^{(t)}=\rho_{j,r,i}^{(t)}\one\{j\not=y_i \}.
\end{align*}
Our proof then focuses on a careful assessment of $\gamma_{j,r}^{(t)}$  $\overrho_{j,r,i}^{(t)}$ and $\underrho_{j,r,i}^{(t)}$ throughout training. To prepare more detailed proof, we give the update rule of $\gamma_{j,r}^{(t)}$  $\overrho_{j,r,i}^{(t)}$ and $\underrho_{j,r,i}^{(t)}$ with $\lambda=0$.
\begin{equation}
\label{eq:Update_eq}
    \begin{split}
\gamma_{j,r}^{(t+1)}&=\gamma_{j,r}^{(t)}-\frac{2\eta}{nm}\sum_{i=1}^{n}{\ell'^{(t)}_i}\ReLU \bigg( \la \wb_{j,r}^{(0)},  y_i\bmu \ra+jy_i\cdot\gamma_{j,r}^{(t)} \bigg)\|\bmu\|^2,\\
\overrho_{j,r,i}^{(t+1)}&=\overrho_{j,r,i}^{(t)}-\frac{2\eta}{nm}{\ell'^{(t)}_i}\ReLU \bigg( \la \wb_{j,r}^{(0)}, \bxi_i \ra +\sum_{i'=1}^n\overrho_{j,r,i'}^{(t)}\frac{\la\bxi_i,\bxi_{i'}\ra}{\|\bxi_{i'}\|^2}+\sum_{i'=1}^n\underrho_{j,r,i'}^{(t)}\frac{\la\bxi_i,\bxi_{i'}\ra}{\|\bxi_{i'}\|^2}\bigg)\|\bxi_i\|^2,\\
\underrho_{j,r,i}^{(t+1)}&=\underrho_{j,r,i}^{(t)}+\frac{2\eta}{nm}{\ell'^{(t)}_i}\ReLU \bigg( \la \wb_{j,r}^{(0)}, \bxi_i \ra +\sum_{i'=1}^n\overrho_{j,r,i'}^{(t)}\frac{\la\bxi_i,\bxi_{i'}\ra}{\|\bxi_{i'}\|^2}+\sum_{i'=1}^n\underrho_{j,r,i'}^{(t)}\frac{\la\bxi_i,\bxi_{i'}\ra}{\|\bxi_{i'}\|^2}\bigg)\|\bxi_i\|^2.
    \end{split}
\end{equation}
With \eqref{eq:Update_eq} above, we give the following proposition which holds during in the whole training process, no matter there exists gradient regularization or not. 
\begin{proposition}
\label{prop:admissible_time_bound}
Under Condition~\ref{condition:condition}, if $\gamma_{j,r}^{(t)},\overrho_{j,r,i}^{(t)}$ and $\underrho_{j,r,i}^{(t)}$ satisfy the update rule \eqref{eq:Update_eq},
$|\gamma_{j,r}^{(0)}|=O(1)$ and $|\rho_{j,r,i}^{(0)}|\leq 8\sqrt{\log(8mn/\delta)}\cdot\sigma_0\sigma_p\sqrt{d}$. Then it holds that
\begin{align*}
    &0\leq\gamma_{j,r}^{(t)},\overrho_{j,r,i}^{(t)}\leq 4\log(T^*),\\
    &\underrho_{j,r,i}^{(t)}\geq -\beta-64n\sqrt{\frac{\log(4n^2/\delta)}{d}}\log(T^*)\geq-4\log(T^*)
\end{align*}
for any $t\in[T^*]$, where $\beta=2\max_{i,j,r}\{|\la w_{j,r}^{(0)},\bmu\ra|,|\la w_{j,r}^{(0)},\bxi_i\ra|\}$.
\end{proposition}
The initialization conditions for $\gamma_{j,r}^{(t)}$ and $\rho_{j,r,i}^{(t)}$ in the following proposition ensure that the inequalities in the proposition will always hold, regardless of whether we are dealing with Theorem~\ref{thm:withregularization} or Theorem~\ref{thm:withoutregularization}. We will prove that when we close the gradient regularization in the case of Theorem~\ref{thm:withregularization}, $\gamma_{j,r}^{(t)}$ and $\rho_{j,r,i}^{(t)}$ still satisfy the initialization conditions specified in Proposition~\ref{prop:admissible_time_bound}. After we prove that Proposition~\ref{prop:admissible_time_bound} holds in both cases, we have the following lemma, which gives an upper bound of the training gradient.
\begin{lemma}
    \label{lemma:admissible_trainloss}
    Under Condition~\ref{condition:condition}, for $0\leq t\leq T^*$ where $T^*=\eta^{-1}\poly({n,m,d,\varepsilon^{-1},\sigma_0^{-1},\sigma_p\sqrt{d}})$ is the maximum admissible iterations. the following result holds.
    \begin{align*}
        \| \nabla_{\Wb} L_S(\Wb)|_{\Wb=\Wb^{(t)}}\|_F^2\leq 72\sigma_p^2d L_S(\Wb^{(t)}).
    \end{align*}
\end{lemma}
With this upper bound, it is clear that if we provide a sharp bound for $L_S(\Wb^{(t)})$ at some iteration $t$, the training gradient will be bounded.

\subsection{A Two-Stage Analysis in Theorem~\ref{thm:withregularization}}
\label{sec:overview-2}
We utilize a two-stage analysis to decouple the complicated relation among the coefficients $\gamma_{j,r}^{(t)}$, $\overrho_{j,r,i}^{(t)}$, and $\underrho_{j,r,i}^{(t)}$. In the first stage of the training process, we set the initial neural network weights to be small enough such that $\ell'^{(0)}_i \approx -1/2$, and we assign $\ell'^{(t)}_i = \ell'(y_if(\Wb^{(t)},\xb_i)) \approx -1/2$ for all $i \in [n]$. We then show that there exists a significant scale difference among the values of $\max_{r}\gamma_{j,r}^{(t)}$, $\max_{r}\overrho_{j,r,i}^{(t)}$, and $\max_{r}|\underrho_{j,r,i}^{(t)}|$ at the final iteration of the first stage. Based on these findings, we proceed to the second stage of the training process, where the loss derivatives are modified.

\textbf{Stage 1} 
It can be shown that when  $\gamma_{j,r}^{(t)}$ and $\rho_{j,r,i}^{(t)}$ reach the order of $1/\polylog(n)$, the value of $\ell'^{(t)}_i$ remains around $1/2$. This observation allows us to simplify the dynamics of the coefficients in \eqref{eq:Update_eq} by using upper and lower bounds in place of $\ell'^{(t)}_i$. Based on these findings, we can summarize our main conclusion in the first stage of training with gradient regularization as follows:
\begin{proposition}
\label{prop:Phase_I_prop2}
Under the same conditions as Theorem~\ref{thm:withregularization}, define $\tT_{1}$  satisfies
\begin{align*}
    T_1=\frac{m}{\eta\|\bmu\|^2}\log\Big(\frac{4}{\sigma_0\|\bmu\|\cdot\log{(n)}\sqrt{2\log(8m/\delta)}}\Big).
\end{align*}
Then the following facts hold:
\begin{enumerate}
    \item For any $i\in[n]$, $|\Upsilon_i^{(t)}|=|\ell'^{(t)}_i+\frac{1}{2}|= O(\frac{1}{\log^2{(n)}})$.
    \item $\gamma_{j,r}^{(t)}\leq 5/\log(n)$ for all $j\in\{\pm1\}$, $r\in[m]$ and $t\in[\tT_1]$.
    \item For  each $j$, 
     there exists $c_1>0$ such that $\max_{r}\gamma_{j,r}^{(\tT_1)}\geq  1/(\sqrt{2\log(8m/\delta)}\log(n) )$. Moreover, for $\rho_{j,r,i}$ we have
    \begin{align*}
    0\geq\rho_{j,r,i}^{(\tT_1)}\geq -8\sqrt{\log(8mn/\delta)}\cdot\sigma_0\sigma_p\sqrt{d}.
\end{align*}
\end{enumerate}
\end{proposition}
Proposition~\ref{prop:Phase_I_prop2} demonstrates that CNN can effectively capture the signal under training with gradient regularization. At the end of this stage, $\max_{r}\gamma_{j,r}^{(t)}$ reaches a value on the order of $1/\polylog(n)$, which is sufficiently larger than $\rho_{j,r,i}^{(t)}$. In the next stage, we show that reaching this threshold for $\max_{r}\gamma_{j,r}^{(t)}$ is sufficient to prove that the test error and training gradient are small.


\textbf{Stage 2} In  this  stage,  we  take  into  full  consideration  the  exact  deﬁnition  $\ell'^{(t)}$, and show that the training loss will converge to any $\varepsilon>0$. We give the following proposition, which directly shows the results in Theorem~\ref{thm:withregularization}. 
\begin{proposition}
\label{prop:Phase_II_prop}
 Under Condition~\ref{condition:condition}, for any $\varepsilon>0$, define $\varepsilon_0=1-e^{-\varepsilon}$, 
and 
\begin{align*}
    \tT_2=\frac{2nm}{\eta \varepsilon_0\|\bmu\|^2}\log\big(\sqrt{2\log(8m/\delta)}\cdot\log(n)d\big)
\end{align*}
then there exists $t\in [\tT_1,\tT_1+\tT_2]$ such that
\begin{align*}
    L_S{(\Wb^{(t)})}\leq \varepsilon.
\end{align*}
\end{proposition}
We prove Propositiuon~\ref{prop:Phase_II_prop} by contradiction.  With Proposition~\ref{prop:admissible_time_bound}, Proposition~\ref{prop:Phase_I_prop2} and Proposition~\ref{prop:Phase_II_prop}, we can prove Theorem~\ref{thm:withregularization}. Details can be found in Appendix~\ref{sec:thmProof1}.

\subsection{Implications to full  gradient regularization (FGR)}
\label{subsec:discussFGR}
In this section, we provide a brief discussion on the performance of full gradient regularization (FGR). Our goal is to train the function $f(\Wb,\xb)$ by minimizing the cross-entropy loss with gradient regularization. Specifically, we seek to minimize the function 
\begin{align*}
    \tilde{L}(\Wb) = L_S(\Wb) + \frac{\lambda}{2} \| \nabla_{\Wb} L_S(\Wb) \|_F^2
\end{align*} 
using gradient descent. Although we do not present a rigorous analysis of FGR, for the sake of clarity, we make the simplifying assumption that $\la \bmu,\bxi_i\ra$ and $\la \bxi_i,\bxi_{i'}\ra$ are all equal to $0$ for $i\neq i'$. With this assumption, we derive the following update rule.
\begin{lemma}
    \label{lemma:update_rule_FGR}
    Under the full gradient regularization, suppose that $\la \bmu,\bxi_i\ra = \la \bxi_i,\bxi_{i'}\ra=0$ for all $i\neq i'$, then the update rule is
\begin{align*}
    &\wb_{j,r}^{(t+1)}=\wb_{j,r}^{(t)}-\frac{2\eta}{nm}\sum_{i=1}^{n}{\ell'^{(t)}_i} \big(\ReLU ( \la\wb_{j,r}^{(t)}, y_i\bmu \ra)j\bmu+\ReLU (\la \wb_{j,r}^{(t)}, jy_i\bxi_i \ra) \bxi_i\big)\\
   &\quad -\frac{4\lambda\eta}{n^2m^2}\sum_{j'}\sum_{r'=1}^m\bigg(\sum_{i=1}^{n}{\ell'^{(t)}_i}\ReLU( \la \wb_{j',r'}^{(t)}, y_i\bmu \ra )\bigg)\cdot \bigg(\sum_{i=1}^{n}{\ell''^{(t)}_i}\Big(\ReLU( \la \wb_{j,r}^{(t)}, y_i\bmu \ra )j\bmu+\ReLU( \la \wb_{j,r}^{(t)}, \bxi_i \ra )jy_i\bxi_i\Big)\\
   &\qquad\qquad\cdot \ReLU( \la \wb_{j',r'}^{(t)}, y_i\bmu \ra )\bigg) \|\bmu\|^2 +\bigg(\sum_{i=1}^{n}{\ell'^{(t)}_i}\ReLU( \la \wb_{j,r}^{(t)}, y_i\bmu \ra )\bigg)\cdot\bigg(\sum_{i=1}^{n}{\ell'^{(t)}_i}\one( \la \wb_{j,r}^{(t)}, y_i\bmu \ra>0 )y_i\bmu\bigg) \|\bmu\|^2\\
   &\qquad -\frac{4\lambda\eta}{n^2m^2}\sum_{i=1}^n\ell'^{2(t)}_i\ReLU( \la \wb_{j,r}^{(t)}, \bxi_i  \ra )\|\bxi_i\|^2\bxi_i \\
   &\qquad -\frac{4\lambda\eta}{n^2m^2}\sum_{j'}\sum_{r'=1}^m \sum_{i=1}^n\ell'^{(t)}_i\ell''^{(t)}_i\cdot \Big(\ReLU( \la \wb_{j,r}^{(t)}, \bxi_i  \ra )j\bmu+\ReLU( \la \wb_{j,r}^{(t)}, \bxi_i \ra )jy_i\bxi_i\Big)\cdot \ReLU^2( \la \wb_{j',r'}^{(t)}, \bxi_i \ra )\|\bxi_i\|^2.
\end{align*}
\end{lemma}
The proof of Lemma~\ref{lemma:update_rule_FGR} can be found in Appendix~\ref{subsec:proofFGR}. To provide a high-level discussion on the performance of FGR, we omit all the high-order terms in Lemma~\ref{lemma:update_rule_FGR} and focus on the dominant ones. Accordingly, we simplify the update rule as follows:
\begin{align*}
    &\wb_{j,r}^{(t+1)}\approx\wb_{j,r}^{(t)}-\frac{2\eta}{nm}\sum_{i=1}^{n}{\ell'^{(t)}_i} \big(\ReLU ( \la\wb_{j,r}^{(t)}, y_i\bmu \ra)j\bmu+\ReLU (\la \wb_{j,r}^{(t)}, jy_i\bxi_i \ra) \bxi_i\big)\\
    &\qquad-\frac{4\lambda\eta}{n^2m^2}\bigg(\sum_{i=1}^{n}{\ell'^{(t)}_i}\ReLU( \la \wb_{j,r}^{(t)}, y_i\bmu \ra )\bigg)\cdot\bigg(\sum_{i=1}^{n}{\ell'^{(t)}_i}\one( \la \wb_{j,r}^{(t)}, y_i\bmu \ra>0 )y_i\bmu\bigg) \|\bmu\|^2\\
   &\qquad -\frac{4\lambda\eta}{n^2m^2}\sum_{i=1}^n\ell'^{2(t)}_i\ReLU( \la \wb_{j,r}^{(t)}, \bxi_i  \ra )\|\bxi_i\|^2\bxi_i.
\end{align*}
As discussed  in Lemma~\ref{lemma:update_rule}, we can approximate $\ell_i^{(t)}\sim -0.5$, $\ell_{i}''^{(t)}\sim 1/4$, and $\ell_{i}'^{2(t)}\sim 1/4$. Furthermore, we assume that $\bxi_i$ is orthogonal to all other noise vectors. Employing the previously mentioned update rule, we note that regularization propels signal learning and noise memorization back to zero at suppression speeds of approximately $\lambda\cdot\Theta(\eta\|\bmu\|^4)$ and $\lambda\cdot\Theta(\eta\|\bxi_i\|^4/n^2)$, respectively. In contrast, as discussed in Section \ref{subsec:whyPEGR}, PEGR has a stronger suppression speed for noise memorization, i.e., $\lambda\cdot\Theta(\eta\|\bxi_i\|^4/n^2)$. Additionally, recalling the growth rates under standard training is $ \Theta(\eta\|\bmu\|^2)$ and $\Theta(\eta\|\bxi_i\|^2/n)$, if we adjust $\lambda$ and establish the suppression speeds as $\Theta(\eta\|\bmu\|^4/\|\bxi_i\|)$ and $\Theta(\eta\|\bxi_i\|^3/n^2)$, more stringent conditions (compared to PEGR) will be required to inhibit noise memorization and encourage signal learning. Additionally, the intricate dynamics under FGR may also hinder the successful learning of the signal, as demonstrated by the numerical experiments in Section~\ref{subsec:numericexperiments}. Consequently, in practical applications, we advocate for the utilization of PEGR over FGR to enhance signal learning.

\section{Conclusion and Future Work}
\label{sec:conclusion}
In this paper, we employ a signal-noise decomposition framework to investigate the impact of gradient regularization on the training of a two-layer convolutional neural network (CNN). Specifically, we provide precise conditions under which the CNN will prioritize learning signals over memorizing noises, and demonstrate the benefits of gradient regularization in inhibiting the memorizing of noises while encouraging the learning of signals. Our results theoretically demonstrate the effectiveness of gradient regularization in facilitating the learning of signals. As a next step, we aim to extend our analysis to deep convolutional neural networks and investigate the signal learning dynamics when using mini-batch stochastic gradients, which represent a critical area of ongoing research.

\section*{Acknowledgement}
We would like to thank Yuan Hua for a helpful discussion about the experiments.

\bibliographystyle{ims}
\bibliography{deeplearningreference}

\appendix

\section{Calculation for the gradient}
\label{sec:update_rule}
In this section, we give the calculation of the gradient in details. 

\subsection{The calculation of PEGR}
\label{subsec:proofPEGR}
We remind readers that 
\begin{align*}
    L_S(\Wb) = \frac{1}{n}\sum_{i=1}^n \ell[y_i\cdot f(\Wb,\xb_i)],
\end{align*}
where $\ell(z) = \log(1 + \exp(-z))$. We consider training $f(\Wb,\xb)$ by minimizing the cross-entropy loss with gradient regularization. Specifically, we minimize
\begin{align*}
    \tilde{L}(\Wb) = L_S(\Wb) + \frac{\lambda}{2n} \cdot \sum_i\| \nabla_{\Wb} L_i(\Wb) \|_F^2,
\end{align*}
with gradient descent. Here, we denote by $L_i(\Wb) = \ell(y_i\cdot f(\Wb,\xb_i))$. With the notations above, we first have $\ell'(z)=-1/(1+\exp{(z)})$ and $\ell''(z)=\exp{(z)}/(1+\exp{(z)})^2$. 
Given $\tilde{L}(\Wb)$ above, we restate Lemma~\ref{lemma:update_rule} below:

\begin{lemma}[Restatement of Lemma~\ref{lemma:update_rule}]
\label{lemma:update_rule_appendix}
Given $\tilde{L}(\Wb)$ above, we have
\begin{align*}
    &\wb_{j,r}^{(t+1)}=\wb_{j,r}^{(t)}-\frac{2\eta}{nm}\sum_{i=1}^{n}{\ell'^{(t)}_i}\bigg(1+\frac{4\lambda\zeta_i^{(t)}}{m}\cdot\ell''^{(t)}_i\bigg)\cdot jy_i \big(\ReLU ( \la\wb_{j,r}^{(t)}, y_i\bmu \ra)y_i\bmu+\ReLU (\la \wb_{j,r}^{(t)}, \bxi_i \ra) \bxi_i\big)\\
    &\qquad-\frac{4\lambda\eta}{nm^2}\sum_{i=1}^{n}{\ell'^{2(t)}_i}\big(\ReLU  ( \la \wb_{j,r}^{(t)}, y_i\bmu \ra)\|\bmu\|^2 y_i\bmu+\ReLU  ( \la \wb_{j,r}^{(t)}, \bxi_i \ra)\|\bxi_i\|^2\bxi_i\big).
\end{align*}
\end{lemma}
\begin{proof}[Proof of Lemma~\ref{lemma:update_rule_appendix}]
For $j\in\{+1,-1\}$, the gradient for $L_S(\Wb)$ and $L_i(\Wb)$is
\begin{align*}
    \nabla_{\wb_{j,r}}L_S(\Wb)&=\frac{1}{n}\sum_{i=1}^{n}\nabla_{\wb_{j,r}}\ell[y_i\cdot f(\Wb,\xb_i)]\\
    &=\frac{1}{n}\sum_{i=1}^{n}{\ell'}[y_i\cdot f(\Wb,\xb_i)]\cdot\nabla_{\wb_{j,r}}[y_i\cdot f(\Wb,\xb_i)]\\
    &=\frac{1}{nm}\sum_{i=1}^{n}{\ell'}[y_i\cdot f(\Wb,\xb_i)]\cdot jy_i\nabla_{\wb_{j,r}}[ \sigma( \la \wb_{j,r}, y_i\bmu \ra )+\sigma( \la \wb_{j,r}, \bxi_i \ra )]\\
    &=\frac{1}{nm}\sum_{i=1}^{n}{\ell'}[y_i\cdot f(\Wb,\xb_i)]\cdot jy_i\big(\sigma'( \la \wb_{j,r}, y_i\bmu \ra )y_i\bmu+\sigma'( \la \wb_{j,r}, \bxi_i \ra )\bxi_i\big),\\
    \nabla_{\wb_{j,r}}L_i(\Wb)&=\nabla_{\wb_{j,r}}\ell[y_i\cdot f(\Wb,\xb_i)]\\
    &={\ell'}[y_i\cdot f(\Wb,\xb_i)]\cdot\nabla_{\wb_{j,r}}[y_i\cdot f(\Wb,\xb_i)]\\
    &=\frac{1}{m}{\ell'}[y_i\cdot f(\Wb,\xb_i)]\cdot jy_i\big(\sigma'( \la \wb_{j,r},  y_i\bmu \ra ) y_i\bmu+\sigma'( \la \wb_{j,r}, \bxi_i \ra )\bxi_i\big).
\end{align*}
Moreover, if we further denote by ${\ell'}[y_{i}\cdot f(\Wb,\xb_{i})]=\ell'_{i},~{\ell''}[y_{i}\cdot f(\Wb,\xb_{i})]=\ell''_{i}$, we can easily see that 
\begin{align*}
    &\nabla_{\wb_{j,r}}\ell'^2_i=2\ell'_i\ell''_ijy_i\big(\sigma'( \la \wb_{j,r},  y_i\bmu \ra ) y_i\bmu+\sigma'( \la \wb_{j,r}, \bxi_i \ra )\bxi_i\big),\\
        &\| \nabla_{\Wb} L_i(\Wb) \|_F^2=\frac{\ell'^2_i}{m^2}\sum_{j'}\sum_{r'=1}^m\big(\sigma'^2( \la \wb_{j',r'},  y_i\bmu \ra )\|\bmu\|^2+\sigma'^2( \la \wb_{j',r'}, \bxi_i \ra )\|\bxi_i\|^2\big).
\end{align*}
Here, we utilize the fact that $\la\bmu,\bxi_i\ra=0$. Then by the equations above, the gradient descent of $\| \nabla_{\Wb} L_i(\Wb) \|_F^2$ can be expressed as 
\begin{align*}
    &\nabla_{\wb_{j,r}}\| \nabla_{\Wb} L_i(\Wb) \|_F^2\\
    &=\frac{1}{m^2}\sum_{j'}\sum_{r'=1}^m\nabla_{\wb_{j,r}}\big[\ell'^2_i\big(\sigma'^2( \la \wb_{j',r'},  y_i\bmu \ra )\|\bmu\|^2\notag+\sigma'^2( \la \wb_{j',r'}, \bxi_i \ra )\|\bxi_i\|^2\big)\big]\\
    &=\frac{1}{m^2}\sum_{j'}\sum_{r'=1}^m\Big\{\ell'^2_i\nabla_{\wb_{j,r}}\big[\big(\sigma'^2( \la \wb_{j',r'},  y_i\bmu \ra )\|\bmu\|^2+\sigma'^2( \la \wb_{j',r'}, \bxi_i \ra )\|\bxi_i\|^2\big)\big]\\
    &\quad+ \nabla_{\wb_{j,r}}\ell'^2_i\cdot\big[\big(\sigma'^2( \la \wb_{j',r'},  y_i\bmu \ra )\|\bmu\|^2+\sigma'^2( \la \wb_{j',r'}, \bxi_i \ra )\|\bxi_i\|^2\big]\Big\}\\
    &=\frac{8}{m^2}\Big\{\ell'^2_i \ReLU  (\la \wb_{j,r},  y_i\bmu \ra)\|\bmu\|^2 y_i\bmu+\ell'^2_i\ReLU  (\la \wb_{j,r}, \bxi_i \ra)\|\bxi_i\|^2\bxi_i\Big\}\\
    &\qquad+\frac{16\ell'_i\ell''_ijy_i}{m^2}\zeta_i\big(\ReLU ( \la \wb_{j,r},  y_i\bmu \ra ) y_i\bmu+\ReLU ( \la \wb_{j,r}, \bxi_i \ra )\bxi_i\big) \Big\}.
\end{align*}
Here, we define $\zeta_i=\sum_{j'}\sum_{r'=1}^m\big(\ReLU^2( \la \wb_{j',r'},  y_i\bmu \ra )\|\bmu\|^2+\ReLU^2( \la \wb_{j',r'}, \bxi_i \ra )\|\bxi_i\|^2\big)$ in the last equality. Therefore the gradient flow could be written as 
\begin{align*}
    &\wb_{j,r}^{(t+1)}=\wb_{j,r}^{(t)}-\frac{2\eta}{nm}\sum_{i=1}^{n}{\ell'^{(t)}_i}\cdot jy_i \big(\ReLU ( \la\wb_{j,r}^{(t)}, y_i\bmu \ra)y_i\bmu+\ReLU (\la \wb_{j,r}^{(t)}, \bxi_i \ra) \bxi_i\big)\\
    &\quad-\frac{4\lambda\eta}{nm^2}\sum_{i=1}^{n}{\ell'^{2(t)}_i}\big(\ReLU  ( \la \wb_{j,r}^{(t)}, y_i\bmu \ra)\|\bmu\|^2 y_i\bmu+\ReLU  ( \la \wb_{j,r}^{(t)}, \bxi_i \ra)\|\bxi_i\|^2\bxi_i\big)\\
    &\quad-\frac{8\lambda \eta}{nm^2}\sum_{i=1}^n\ell'^{(t)}_i\ell''^{(t)}_i\cdot jy_i\zeta_i^{(t)}\big(\ReLU ( \la \wb_{j,r},  y_i\bmu \ra ) y_i\bmu+\ReLU ( \la \wb_{j,r}, \bxi_i \ra )\bxi_i \big)\\
    &=\wb_{j,r}^{(t)}-\frac{2\eta}{nm}\sum_{i=1}^{n}{\ell'^{(t)}_i}\bigg(1+\frac{4\lambda \zeta_i^{(t)}}{m}\cdot\ell''^{(t)}_i\bigg)\cdot jy_i \big(\ReLU ( \la\wb_{j,r}^{(t)}, y_i\bmu \ra)y_i\bmu+\ReLU (\la \wb_{j,r}^{(t)}, \bxi_i \ra) \bxi_i\big)\\
    &\quad-\frac{4\lambda\eta}{nm^2}\sum_{i=1}^{n}{\ell'^{2(t)}_i}\big(\ReLU  ( \la \wb_{j,r}^{(t)}, y_i\bmu \ra)\|\bmu\|^2 y_i\bmu+\ReLU  ( \la \wb_{j,r}^{(t)}, \bxi_i \ra)\|\bxi_i\|^2\bxi_i\big),
\end{align*}
which gives Lemma~\ref{lemma:update_rule}.
\end{proof}

Lemma~\ref{lemma:update_rule_appendix} above clearly gives the update rule of $\wb_{j,r}^{(t)}$. According to the update rule, we can further gives the decomposition of $\wb_{j,r}^{(t)}$ as shown in the next lemma.
\begin{lemma}[Restatement of Lemma~\ref{lemma:decomposition_coef_dynamic}]
\label{lemma:decomposition_coef_dynamic_appendix}
The coefficients $\gamma_{j,r}^{(t)},\rho_{j,r,i}^{(t)}$ in Definition~\ref{def:signal_to_noise_decomp} satisfy the following equations:
\begin{align*}
\zeta_i^{(t)}&=\sum_{j'}\sum_{r'=1}^m\big(\ReLU^2( \la \wb_{j',r'}^{(t)},  y_i\bmu \ra )\|\bmu\|^2+\ReLU^2( \la \wb_{j',r'}^{(t)}, \bxi_i \ra )\|\bxi_i\|^2\big),\\
    \gamma_{j,r}^{(t+1)}&=\gamma_{j,r}^{(t)}-\frac{2\eta}{nm}\sum_{i=1}^{n}\bigg[{\ell'^{(t)}_i}\bigg(1+\frac{4\lambda\zeta_i^{(t)}}{m}\cdot\ell''^{(t)}_i\bigg) \ReLU ( \la\wb_{j,r}^{(t)}, y_i\cdot\bmu \ra)\|\bmu\|^2\\
    &\quad+\frac{  2\lambda}{m}{\ell'^{2(t)}_i}\ReLU  ( \la \wb_{j,r}^{(t)}, y_i\cdot\bmu \ra)\|\bmu\|^4\cdot jy_i\bigg],\\
    \rho_{j,r,i}^{(t+1)}&=\rho_{j,r,i}^{(t)}-\frac{2\eta}{nm}\bigg[{\ell'^{(t)}_i}\bigg(1+\frac{4\lambda \zeta_i^{(t)}}{m}\cdot\ell''^{(t)}_i\bigg) \ReLU ( \la\wb_{j,r}^{(t)}, \bxi_i \ra)\|\bxi_i\|^2\cdot jy_i\\
    &\quad+\frac{  2\lambda}{m}{\ell'^{2(t)}_i}\ReLU  ( \la \wb_{j,r}^{(t)}, \bxi_i \ra)\|\bxi_i\|^4\bigg],
\end{align*}
the initialization condition is $\gamma_{j,r}^{(0)}=\rho_{j,r,i}^{(0)}=0$.
\end{lemma}
\begin{proof}[Proof of Lemma~\ref{lemma:decomposition_coef_dynamic_appendix}]
From Lemma~\ref{lemma:update_rule_appendix}, we first have the update rule:
\begin{align*}
    &\wb_{j,r}^{(t+1)}=\wb_{j,r}^{(t)}-\frac{2\eta}{nm}\sum_{i=1}^{n}{\ell'^{(t)}_i}\bigg(1+\frac{4\lambda \zeta_i^{(t)}}{m}\cdot\ell''^{(t)}_i\bigg)\cdot jy_i \big(\ReLU ( \la\wb_{j,r}^{(t)}, y_i\bmu \ra)y_i\bmu+\ReLU (\la \wb_{j,r}^{(t)}, \bxi_i \ra) \bxi_i\big)\\
    &\qquad-\frac{4\lambda\eta}{nm^2}\sum_{i=1}^{n}{\ell'^{2(t)}_i}\big(\ReLU  ( \la \wb_{j,r}^{(t)}, y_i\bmu \ra)\|\bmu\|^2 y_i\bmu+\ReLU  ( \la \wb_{j,r}^{(t)}, \bxi_i \ra)\|\bxi_i\|^2\bxi_i\big).
\end{align*}
Here, we remind readers that 
\begin{align*}
    \zeta_i^{(t)}=\sum_{j'}\sum_{r'=1}^m\big(\ReLU^2( \la \wb_{j',r'}^{(t)},  y_i\bmu \ra )\|\bmu\|^2+\ReLU^2( \la \wb_{j',r'}^{(t)}, \bxi_i \ra )\|\bxi_i\|^2\big)>0.
\end{align*}
It is clear that with probability $1$, the vectors $\bmu$ and $\bxi_i$, $i\in[n]$ are linearly independent. Therefore, the decomposition \eqref{eq:signal_to_noise_decomp} is unique. Now consider $\tgamma_{j,r}^{(0)},\trho_{j,r,i}^{(0)}=0$ and 
\begin{align*}
    \tgamma_{j,r}^{(t+1)}&=\tgamma_{j,r}^{(t)}-\frac{2\eta}{nm}\sum_{i=1}^{n}\bigg[{\ell'^{(t)}_i}\bigg(1+\frac{4\lambda\zeta_i^{(t)}}{m}\cdot\ell''^{(t)}_i\bigg) \ReLU ( \la\wb_{j,r}^{(t)}, y_i\cdot\bmu \ra)\|\bmu\|^2\\
    &\quad+\frac{  2\lambda}{m}{\ell'^{2(t)}_i}\ReLU  ( \la \wb_{j,r}^{(t)}, y_i\cdot\bmu \ra)\|\bmu\|^4\cdot jy_i\bigg],\\
    \trho_{j,r,i}^{(t+1)}&=\trho_{j,r,i}^{(t)}-\frac{2\eta}{nm}\bigg[{\ell'^{(t)}_i}\bigg(1+\frac{4\lambda\zeta_i^{(t)}}{m}\cdot\ell''^{(t)}_i\bigg) \ReLU ( \la\wb_{j,r}^{(t)}, \bxi_i \ra)\|\bxi_i\|^2\cdot jy_i\\
    &\quad+\frac{  2\lambda}{m}{\ell'^{2(t)}_i}\ReLU  ( \la \wb_{j,r}^{(t)}, \bxi_i \ra)\|\bxi_i\|^4\bigg],
\end{align*}
It can be easily checked from Lemma~\ref{lemma:update_rule} that 
\begin{align*}
    \wb_{j,r}^{(t)}&=\wb_{j,r}^{(0)}+j\cdot\tgamma_{j,r}^{(t)}\frac{\bmu}{\|\bmu\|^2}+\sum_{i=1}^n\trho_{j,r,i}^{(t)}\frac{\bxi_i}{\|\bxi_i\|^2}.
\end{align*}
Hence by the uniqueness of the decomposition we have $\gamma_{j,r}^{(t)}=\tgamma_{j,r}^{(t)}$ and $\rho_{j,r,i}^{(t)}=\trho_{j,r,i}^{(t)}$.
\end{proof}
If we further plug in the signal noise decomposition \eqref{eq:signal_to_noise_decomp} into the iterative formulas in Lemma~\ref{lemma:decomposition_coef_dynamic_appendix}, by the first equation in Lemma~\ref{lemma:decomposition_coef_dynamic_appendix} we have
\begin{equation}
    \label{eq:zeta_t_express}
    \small{
    \begin{aligned}
           \zeta_i^{(t)}&=\sum_{j'}\sum_{r'=1}^m\bigg(\ReLU^2\Big( \la \wb_{j',r'}^{(0)},  y_i\bmu \ra+j'y_i\cdot\gamma_{j',r'}^{(t)} \Big)\|\bmu\|^2+\ReLU^2\Big( \la \wb_{j',r'}^{(0)}, \bxi_i \ra +\sum_{i'=1}^n\rho_{j',r',i'}^{(t)}\frac{\la\bxi_i,\bxi_{i'}\ra}{\|\bxi_{i'}\|^2}\Big)\|\bxi_i\|^2\bigg).
    \end{aligned}}
\end{equation}
The second equation in Lemma~\ref{lemma:decomposition_coef_dynamic_appendix} gives us 
\begin{equation}
    \label{eq:gamma_t_express}
    \begin{split}
    \gamma_{j,r}^{(t+1)}&=\gamma_{j,r}^{(t)}-\frac{2\eta}{nm}\sum_{i=1}^{n}\bigg[{\ell'^{(t)}_i}\bigg(1+\frac{4\lambda\zeta_i^{(t)}}{m}\cdot\ell''^{(t)}_i\bigg) \ReLU \Big( \la \wb_{j,r}^{(0)},  y_i\bmu \ra+jy_i\cdot\gamma_{j,r}^{(t)} \Big)\|\bmu\|^2\\
    &\quad+\frac{  2\lambda}{m}{\ell'^{2(t)}_i}\ReLU  \Big( \la \wb_{j,r}^{(0)},  y_i\bmu \ra+jy_i\cdot\gamma_{j,r}^{(t)} \Big)\|\bmu\|^4\cdot jy_i\bigg],
    \end{split}
\end{equation}
Moreover, the third equation in Lemma~\ref{lemma:decomposition_coef_dynamic_appendix} indicates that 
\begin{equation}
    \label{eq:rho_t_express}
    \begin{split}
\rho_{j,r,i}^{(t+1)}&=\rho_{j,r,i}^{(t)}-\frac{2\eta}{nm}\bigg[{\ell'^{(t)}_i}\bigg(1+\frac{4\lambda\zeta_i^{(t)}}{m}\cdot\ell''^{(t)}_i\bigg) \ReLU \Big( \la \wb_{j,r}^{(0)}, \bxi_i \ra +\sum_{i'=1}^n\rho_{j,r,i'}^{(t)}\frac{\la\bxi_i,\bxi_{i'}\ra}{\|\bxi_{i'}\|^2}\Big)\|\bxi_i\|^2\cdot jy_i\\
    &\quad+\frac{  2\lambda}{m}{\ell'^{2(t)}_i}\ReLU  \Big( \la \wb_{j,r}^{(0)}, \bxi_i \ra +\sum_{i'=1}^n\rho_{j,r,i'}^{(t)}\frac{\la\bxi_i,\bxi_{i'}\ra}{\|\bxi_{i'}\|^2}\Big)\|\bxi_i\|^4\bigg],
    \end{split}
\end{equation}

 \subsection{The calculation of FGR discussed in Section~\ref{subsec:discussFGR}}
\label{subsec:proofFGR}
We have $\ell'(z)=-1/(1+\exp{(z)})$. For $j\in\{+1,-1\}$, similar to the analysis above, we have
\begin{align*}
\nabla_{\wb_{j,r}}L_S(\Wb)&=\frac{1}{n}\sum_{i=1}^{n}\nabla_{\wb_{j,r}}\ell[y_i\cdot f(\Wb,\xb_i)]
    &=\frac{2}{nm}\sum_{i=1}^{n}{\ell'_i}\big(\ReLU( \la \wb_{j,r}, y_i\bmu \ra )j\bmu+\ReLU( \la \wb_{j,r}, \bxi_i \ra )jy_i\bxi_i\big).
\end{align*}
With expression of $\nabla_{\wb_{j,r}}L(\Wb)$ above, we can first give the formula of $\| \nabla_{\Wb} L(\Wb) \|_F^2$:
\begin{align*}
    &\| \nabla_{\Wb} L_S(\Wb) \|_F^2=\sum_{j'}\sum_{r'=1}^m\|  \nabla_{\wb_{j',r'}}L_S(\Wb) \|_2^2\\
    &=\frac{4}{n^2m^2}\sum_{j'}\sum_{r'=1}^m\bigg(\sum_{i=1}^{n}{\ell'_i}\ReLU( \la \wb_{j',r'}, y_i\bmu \ra )\bigg)^2\|\bmu\|^2+\frac{4}{n^2m^2}\sum_{j'}\sum_{r'=1}^m \sum_{i=1}^n\ell'^{2}_i\ReLU^2( \la \wb_{j',r'}, \bxi_i \ra )\|\bxi_i\|^2.
\end{align*}
Here, the second equality is by  $\la \bmu,\bxi_i\ra=\la \bxi_i,\bxi_{i'}\ra=0$ for all $i\not=i'$. Therefore we have
\begin{align*}
   &\nabla_{\wb_{j,r}}\frac{1}{2} \| \nabla_{\Wb} L(\Wb) \|_F^2\\
   &\quad =\frac{4}{n^2m^2}\sum_{j'}\sum_{r'=1}^m\bigg(\sum_{i=1}^{n}{\ell'_i}\ReLU( \la \wb_{j',r'}, y_i\bmu \ra )\bigg)\cdot \nabla_{\wb_{j,r}}\bigg(\sum_{i=1}^{n}{\ell'_i}\ReLU( \la \wb_{j',r'}, y_i\bmu \ra )\bigg) \|\bmu\|^2 \\
   &\quad\quad +\frac{2}{n^2m^2}\sum_{i=1}^n\ell'^{2}_i\nabla_{\wb_{j,r}}\ReLU^2( \la \wb_{j,r}, \bxi_i  \ra )\|\bxi_i\|^2+\frac{2}{n^2m^2}\sum_{j'}\sum_{r'=1}^m \sum_{i=1}^n\nabla_{\wb_{j,r}}\ell'^{2}_i\cdot \ReLU^2( \la \wb_{j',r'}, \bxi_i  \ra )\|\bxi_i\|^2.
\end{align*}
With further calculation, we have that 
\begin{align*}
    &\nabla_{\wb_{j,r}}\frac{1}{2} \| \nabla_{\Wb} L(\Wb) \|_F^2\\
   &\quad =\frac{4}{n^2m^2}\sum_{j'}\sum_{r'=1}^m\bigg(\sum_{i=1}^{n}{\ell'_i}\ReLU( \la \wb_{j',r'}, y_i\bmu \ra )\bigg)\cdot \bigg(\sum_{i=1}^{n}{\ell''_i}\Big(\ReLU( \la \wb_{j,r}, y_i\bmu \ra )j\bmu+\ReLU( \la \wb_{j,r}, \bxi_i \ra )jy_i\bxi_i\Big)\\
   &\qquad\qquad\cdot \ReLU( \la \wb_{j',r'}, y_i\bmu \ra )\bigg) \|\bmu\|^2 +\bigg(\sum_{i=1}^{n}{\ell'_i}\ReLU( \la \wb_{j,r}, y_i\bmu \ra )\bigg)\cdot\bigg(\sum_{i=1}^{n}{\ell'_i}\one( \la \wb_{j,r}, y_i\bmu \ra>0 )y_i\bmu\bigg) \|\bmu\|^2\\
   &\qquad +\frac{4}{n^2m^2}\sum_{i=1}^n\ell'^{2}_i\ReLU( \la \wb_{j,r}, \bxi_i  \ra )\|\bxi_i\|^2\bxi_i \\
   &\qquad +\frac{4}{n^2m^2}\sum_{j'}\sum_{r'=1}^m \sum_{i=1}^n\ell'_i\ell''_i\cdot \Big(\ReLU( \la \wb_{j,r}, \bxi_i  \ra )j\bmu+\ReLU( \la \wb_{j,r}, \bxi_i \ra )jy_i\bxi_i\Big)\cdot \ReLU^2( \la \wb_{j',r'}, \bxi_i \ra )\|\bxi_i\|^2.
\end{align*}
Hence, the update rule under FGR is completed under such calculation.

\section{Concentration inequalities}
\label{sec:concentration_appendix}
In this section, we present some trivial  lemmas which give important properties of the data and the neural networks at their random initialization.
\begin{lemma}
\label{lemma:data_count_yi=1}
Suppose that $\delta>0$ and $n\ge8\log(4/\delta)$. Then with probability at least $1-\delta$, 
\begin{align*}
    n/2+\sqrt{\log(4/\delta)/2\cdot n}\geq|\{i\in[n]:y_i=-1\}|,\quad|\{i\in[n]:y_i=1\}|\geq n/2-\sqrt{\log(4/\delta)/2\cdot n}.
\end{align*}
\end{lemma}
\begin{proof}[Proof of Lemma~\ref{lemma:data_count_yi=1}]
    By Hoeffding’s  inequality, with probability at least $1-\delta/2$, we have
    \begin{align*}
        \bigg| \frac{1}{n}\sum_{i=1}^n\one\{y_i=1\}-\frac{1}{2}\bigg|\leq\sqrt{\frac{\log(4/\delta)}{2n}}.
    \end{align*}
    Therefore, as long as $n\geq8\log(4/\delta)$, we have 
    \begin{align*}
        |\{i\in[n]:y_i=1\}|=\sum_{i=1}^n\one\{y_i=1\}\geq \frac{n}{2}-n\cdot \sqrt{\frac{\log(4/\delta)}{2n}}=n/2-\sqrt{\log(4/\delta)/2\cdot n}.
    \end{align*}
    Similarly we have $|\{i\in[n]:y_i=1\}|\leq n/2+\sqrt{\log(4/\delta)/2\cdot n}$. The proof for $|\{i\in[n]:y_i=1\}|$ is exactly the same.
\end{proof}

The following lemma estimates the norms of the noise vectors $\bxi_i$, $i\in[n]$, and gives an upper
bound of their inner products with each other.

\begin{lemma}
\label{lemma:data_noise_concentration}
Suppose that $\delta>0$ and $d\gg\log(4n/\delta)$. Then with probability at least $1-\delta$, 
\begin{align*}
    &\sigma_p^2d/2\leq\sigma_p^2d-O(\sigma_p^2\sqrt{d\log(4n/\delta)})\leq\| \bxi_i\|^2\leq\sigma_p^2d+O(\sigma_p^2\sqrt{d\log(4n/\delta)})\leq3\sigma_p^2d/2,\\
    &|\la\bxi_i,\bxi_{i'}\ra|\leq2\sigma_p^2\cdot\sqrt{d\log(4n^2/\delta)}
\end{align*}
for all $i,i'\in[n]$.
\end{lemma}
\begin{proof}[Proof of Lemma~\ref{lemma:data_noise_concentration}]
By Bernstein’s inequality, with probability at least $1-\delta/(2n)$ we have
\begin{align*}
    \big|\|\bxi_i\|^2-\sigma_p^2d\big|=O\big(\sigma_p^2\cdot \sqrt{d\log(4n/\delta)}\big).
\end{align*}
Therefore,  as  long  as $ d  =  \Omega(\log(4n/\delta))$,  we  have
\begin{align*}
    &\sigma_p^2d/2\leq\sigma_p^2d-O(\sigma_p^2\sqrt{d\log(4n/\delta)})\leq\| \bxi_i\|^2\leq\sigma_p^2d+O(\sigma_p^2\sqrt{d\log(4n/\delta)})\leq3\sigma_p^2d/2,\\
&|\la\bxi_i,\bxi_{i'}\ra|\leq2\sigma_p^2\cdot\sqrt{d\log(4n^2/\delta)}.
\end{align*}
Moreover,   clearly   $\la \bxi_i,\bxi_{i'}\ra$  has  mean  zero.    For  any   $i,i'$ with $i\not=i'$ ,   by  Bernstein’s  inequality,   with 
probability  at  least  $1-\delta/(2n^2)$  we  have
\begin{align*}
    |\la\bxi_i,\bxi_{i'}\ra|\leq2\sigma_p^2\cdot\sqrt{d\log(4n^2/\delta)}.
\end{align*}
The union bound completes the proof.
\end{proof}

The following lemma gives the bound of the inner probuct between the initialized CNN convolutional filter $\wb_{j,r}^{(0)}$ and the  signal/noise vectors in the training data.

\begin{lemma}
\label{lemma:data_CNN_concentration}
Suppose that $\delta>0$, $d\gg\log(mn/\delta)$ and $m\gg\log(1/\delta)$. Then with probability at least $1-\delta$,
\begin{align*}
    &|\la\wb_{j,r}^{(0)},\bmu\ra|\leq\sqrt{2\log(8m/\delta)}\cdot\sigma_0\|\bmu\|,\\
    &|\la\wb_{j,r}^{(0)},\bxi_i\ra|\leq2\sqrt{\log(8mn/\delta)}\cdot\sigma_0\sigma_p\sqrt{d}
\end{align*}
for all $r\in[m]$, $j\in\{\pm1\}$ and $i\in[n]$. Moreover, 
\begin{align*}
    &\max_{r\in[m]}j\cdot\la\wb_{j,r}^{(0)},\bmu\ra\geq \sigma_0\|\bmu\|/2,\\
    &\max_{r\in[m]}j\cdot\la\wb_{j,r}^{(0)},\bxi_i\ra\geq \sigma_0\sigma_p\sqrt{d}/4.
\end{align*}
\end{lemma}
\begin{proof}[Proof of Lemma~\ref{lemma:data_CNN_concentration}]
    For each $r\in[m]$, $j\cdot\la\wb_{j,r}^{(0)},\bmu \ra$ is a Gaussian random variable with mean zero and variance $\sigma_0^2\|\bmu\|^2$. Therefore by Gaussian tail bound and union bound, with probability at least $1-\delta/4$, 
    \begin{align*}
    j\cdot\la\wb_{j,r}^{(0)},\bmu\ra\leq|\la\wb_{j,r}^{(0)},\bmu\ra|\leq \sqrt{2\log(8m/\delta)}\cdot\sigma_0\|\bmu\|.
    \end{align*}
    Moreover, $\PP(\sigma_0\|\bmu\|/2>j\cdot\la\wb_{j,r}^{(0)},\bmu\ra)$ is an absolute constant smaller than $1$, and therefore by the condition on $m$, we have
    \begin{align*}
        \PP(\sigma_0\|\bmu\|/2\leq\max_{r\in[m]}j\cdot\la\wb_{j,r}^{(0)},\bmu\ra)&=1-\PP(\sigma_0\|\bmu\|/2>\max_{r\in[m]}j\cdot\la\wb_{j,r}^{(0)},\bmu\ra)\\
        &=1-\PP(\sigma_0\|\bmu\|/2>j\cdot\la\wb_{j,r}^{(0)},\bmu\ra)^{2m}
        \geq 1-\delta/4.
    \end{align*}
    By Lemma~\ref{lemma:data_noise_concentration}, with probability at least $1-\delta/4$, $\sigma\sqrt{d}/\sqrt{2}\leq \|\bxi_i\|\leq\sqrt{3/2}\sigma\sqrt{d}$ for all $i\in[n]$. Therefore the results for  $j\cdot\la\wb_{j,r}^{(0)},\bxi_i\ra$ follows the same proof as $j\cdot\la\wb_{j,r}^{(0)},\bmu\ra$.
\end{proof}

\section{General properties for both cases}
\label{sec:General_properties}
In this section, we consider the update rule with $\lambda=0$. Note that in both cases in Theorem~\ref{thm:withregularization} and Theorem~\ref{thm:withoutregularization}, we will set $\lambda=0$ in the final stage, so it is suitable to present some universal properties here which will help us handle both cases. When $\lambda=0$,  from Lemma~\ref{lemma:decomposition_coef_dynamic_appendix}, 
the coefficients $\gamma_{j,r}^{(t)},\rho_{j,r,i}^{(t)}$ in Definition~\ref{def:signal_to_noise_decomp} satisfy the following equations:
\begin{align*}
    \gamma_{j,r}^{(t+1)}&=\gamma_{j,r}^{(t)}-\frac{2\eta}{nm}\sum_{i=1}^{n}{\ell'^{(t)}_i} \ReLU ( \la\wb_{j,r}^{(t)}, y_i\cdot\bmu \ra)\|\bmu\|^2,\\
    \rho_{j,r,i}^{(t+1)}&=\rho_{j,r,i}^{(t)}-\frac{2\eta}{nm}{\ell'^{(t)}_i} \ReLU ( \la\wb_{j,r}^{(t)}, \bxi_i \ra)\|\bxi_i\|^2\cdot jy_i.
\end{align*}
We define
\begin{align*}
\overrho_{j,r,i}^{(t)}=\rho_{j,r,i}^{(t)}\one(j=y_i),\quad \underrho_{j,r,i}^{(t)}=\rho_{j,r,i}^{(t)}\one(j\not=y_i),
\end{align*}
the update rule can be further written as 
\begin{equation}
\label{eq:update_rule_lbd=0}
    \begin{split}
\gamma_{j,r}^{(t+1)}&=\gamma_{j,r}^{(t)}-\frac{2\eta}{nm}\sum_{i=1}^{n}{\ell'^{(t)}_i}\ReLU \bigg( \la \wb_{j,r}^{(0)},  y_i\bmu \ra+jy_i\cdot\gamma_{j,r}^{(t)} \bigg)\|\bmu\|^2,\\
\overrho_{j,r,i}^{(t+1)}&=\overrho_{j,r,i}^{(t)}-\frac{2\eta}{nm}{\ell'^{(t)}_i}\ReLU \bigg( \la \wb_{j,r}^{(0)}, \bxi_i \ra +\sum_{i'=1}^n\overrho_{j,r,i'}^{(t)}\frac{\la\bxi_i,\bxi_{i'}\ra}{\|\bxi_{i'}\|^2}+\sum_{i'=1}^n\underrho_{j,r,i'}^{(t)}\frac{\la\bxi_i,\bxi_{i'}\ra}{\|\bxi_{i'}\|^2}\bigg)\|\bxi_i\|^2,\\
\underrho_{j,r,i}^{(t+1)}&=\underrho_{j,r,i}^{(t)}+\frac{2\eta}{nm}{\ell'^{(t)}_i}\ReLU \bigg( \la \wb_{j,r}^{(0)}, \bxi_i \ra +\sum_{i'=1}^n\overrho_{j,r,i'}^{(t)}\frac{\la\bxi_i,\bxi_{i'}\ra}{\|\bxi_{i'}\|^2}+\sum_{i'=1}^n\underrho_{j,r,i'}^{(t)}\frac{\la\bxi_i,\bxi_{i'}\ra}{\|\bxi_{i'}\|^2}\bigg)\|\bxi_i\|^2.
    \end{split}
\end{equation}
The update rule clearly shows that $\gamma_{j,r}^{(t)},\overrho_{j,r,i}^{(t)}$ are non decreasing, $\underrho_{j,r,i}^{(t)}$ is non increasing with respect to $t$.
With this update rule, we now show that the parameter in the signal-noise decomposition will stay in a reasonable scale during a long time training without  the regularization. Let us consider the learning period $\tT_1\leq t\leq T^*$, where $T^*=\eta^{-1}\poly({n,m,d,\varepsilon^{-1},\sigma_0^{-1},\sigma_p\sqrt{d}})$ is the maximum admissible iterations. Suppose that Condition~\ref{condition:condition} holds, we have
\begin{proposition}[Restatement of Proposition~\ref{prop:admissible_time_bound}]
\label{prop:admissible_time_bound_appendix}
Under Condition~\ref{condition:condition}, if $\gamma_{j,r}^{(t)},\overrho_{j,r,i}^{(t)}$ and $\underrho_{j,r,i}^{(t)}$ satisfy the update rule \eqref{eq:update_rule_lbd=0},
$|\gamma_{j,r}^{(0)}|\leq O(1)$ and $|\rho_{j,r,i}^{(0)}|\leq 8\sqrt{\log(8mn/\delta)}\cdot\sigma_0\sigma_p\sqrt{d}$. Then it holds that
\begin{align}
    &\gamma_{j,r}^{(t)},\overrho_{j,r,i}^{(t)}\leq 4\log(T^*),\label{eq:admissible_gamma_overrho}\\
    &\underrho_{j,r,i}^{(t)}\geq -\beta-64n\sqrt{\frac{\log(4n^2/\delta)}{d}}\log(T^*)\geq-4\log(T^*)\label{eq:admissible_underrho}
\end{align}
for any $t\in[T^*]$, where $\beta=2\max_{i,j,r}\{|\la w_{j,r}^{(0)},\bmu\ra|,|\la w_{j,r}^{(0)},\bxi_i\ra|\}$.
\end{proposition}
To prove Proposition~\ref{prop:admissible_time_bound_appendix}, we will employ the method of induction. This proposition demonstrates that $\gamma_{j,r}^{(t)},\overrho_{j,r,i}^{(t)}$ and $\underrho_{j,r,i}^{(t)}$ satisfy \eqref{eq:admissible_gamma_overrho} and \eqref{eq:admissible_underrho} throughout the entire training process, regardless of whether gradient regularization is utilized or not. Prior to presenting the proof, we will introduce several technical inequalities that will be used in support of Proposition~\ref{prop:admissible_time_bound_appendix}.

\begin{lemma}[Technical Inequalities]
\label{lemma:admissible_lemma}
Under Condition~\ref{condition:condition}, suppose that \eqref{eq:admissible_gamma_overrho} and \eqref{eq:admissible_underrho} hold at iteration $t$, then the following inequalities hold:
\begin{enumerate}
    \item For all $i\in[n]$, $r\in[m]$ and $j\in\{\pm1\}$,
    \begin{align*}
        &\underrho_{j,r,i}^{(t)}-32n\sqrt{\frac{\log(4n^2/\delta)}{d}}\log(T^*)\leq \la w_{j,r}^{(t)}-w_{j,r}^{(0)},\bxi_i\ra\leq \underrho_{j,r,i}^{(t)}+32n\sqrt{\frac{\log(4n^2/\delta)}{d}}\log(T^*), j\not=y_i,\\
        &\overrho_{j,r,i}^{(t)}-32n\sqrt{\frac{\log(4n^2/\delta)}{d}}\log(T^*)\leq \la w_{j,r}^{(t)}-w_{j,r}^{(0)},\bxi_i\ra\leq \overrho_{j,r,i}^{(t)}+32n\sqrt{\frac{\log(4n^2/\delta)}{d}}\log(T^*), j=y_i.
    \end{align*}
    \item For all $r\in[m]$, $i\in[n]$ and $j\not=y_i$, 
    \begin{align*}
        &\la w_{j,r}^{(t)},y_i\bmu\ra\leq \la w_{j,r}^{(0)},y_i\bmu\ra,\quad F_j(W_j^{(t)},\xb_i)\leq O(1), \\
        &\la w_{j,r}^{(t)},\bxi_i\ra\leq \la w_{j,r}^{(0)},\bxi_i\ra+32n\sqrt{\frac{\log(4n^2/\delta)}{d}}\log(T^*).
    \end{align*}
    \item For all $r\in[m]$, $i\in[n]$ and $j=y_i$,
    \begin{align*}
        &\la w_{j,r}^{(t)},y_i\bmu\ra= \la w_{j,r}^{(0)},y_i\bmu\ra+\gamma_{j,r}^{(t)},\\
        &\la w_{j,r}^{(t)},\bxi_i\ra\leq \la w_{j,r}^{(0)},\bxi_i\ra+\overrho_{j,r,i}^{(t)}+32n\sqrt{\frac{\log(4n^2/\delta)}{d}}\log(T^*).
    \end{align*}
    Moreover, if $\max\{\gamma_{j,r}^{(t)},\rho_{j,r,i}^{(t)}\}=O(1)$, we have $F_j(W_j^{(t)},\xb_i)=O(1)$.
\end{enumerate}
\end{lemma}
\begin{proof}[Proof of Lemma~\ref{lemma:admissible_lemma}]
\begin{enumerate}
    \item For $j\not=y_i$, we have $\overrho_{j,r,i}^{(t)}=0$ and 
    \begin{align*}
        \la \wb_{j,r}^{(t)}-\wb_{j,r}^{(0)},\bxi_i\ra&=\sum_{i'=1}^n\overrho_{j,r,i'}^{(t)}\|\bxi_{i'}\|^{-2}\cdot\la\bxi_i,\bxi_{i'}\ra+\sum_{i'=1}^n\underrho_{j,r,i'}^{(t)}\|\bxi_{i'}\|^{-2}\cdot\la\bxi_i,\bxi_{i'}\ra\\
        &\leq 4\sqrt{\frac{\log(4n^2/\delta)}{d}}\sum_{i'\not=i}\big|\overrho_{j,r,i'}^{(t)} \big|+4\sqrt{\frac{\log(4n^2/\delta)}{d}}\sum_{i'\not=i}\big|\underrho_{j,r,i'}^{(t)} \big|+\underrho_{j,r,i}^{(t)}\\
        &\leq \underrho_{j,r,i}^{(t)}+32n\sqrt{\frac{\log(4n^2/\delta)}{d}}\log(T^*),
    \end{align*}
    where the first inequality is by Lemma~\ref{lemma:data_noise_concentration} and the last inequality is by $\big|\overrho_{j,r,i'}^{(t)} \big|,\big|\underrho_{j,r,i'}^{(t)} \big|\leq 4\log(T^{*})$ in \eqref{eq:admissible_gamma_overrho} and \eqref{eq:admissible_underrho}. Similarly, for $j=y_i$, we have that $\underrho_{j,r,i}^{(t)}=0$ and
        \begin{align*}
        \la \wb_{j,r}^{(t)}-\wb_{j,r}^{(0)},\bxi_i\ra&=\sum_{i'=1}^n\overrho_{j,r,i'}^{(t)}\|\bxi_{i'}\|^{-2}\cdot\la\bxi_i,\bxi_{i'}\ra+\sum_{i'=1}^n\underrho_{j,r,i'}^{(t)}\|\bxi_{i'}\|^{-2}\cdot\la\bxi_i,\bxi_{i'}\ra\\
        &\leq \overrho_{j,r,i}^{(t)}+4\sqrt{\frac{\log(4n^2/\delta)}{d}}\sum_{i'\not=i}\big|\overrho_{j,r,i'}^{(t)} \big|+4\sqrt{\frac{\log(4n^2/\delta)}{d}}\sum_{i'\not=i}\big|\underrho_{j,r,i'}^{(t)} \big|\\
        &\leq \underrho_{j,r,i}^{(t)}+32n\sqrt{\frac{\log(4n^2/\delta)}{d}}\log(T^*),
    \end{align*}
    where the first inequality is by Lemma~\ref{lemma:data_noise_concentration} and the last inequality is by $\big|\overrho_{j,r,i'}^{(t)} \big|,\big|\underrho_{j,r,i'}^{(t)} \big|\leq 4\log(T^{*})$ in \eqref{eq:admissible_gamma_overrho} and \eqref{eq:admissible_underrho}. Similarly, we can show that 
    $\la\wb_{j,r}^{(t)}-\wb_{j,r}^{(0)},\bxi_i \ra\geq \underrho_{j,r,i}^{(t)}-32n\sqrt{\frac{\log(4n^2/\delta)}{d}}\log(T^*)$ and $\la\wb_{j,r}^{(t)}-\wb_{j,r}^{(0)},\bxi_i \ra\geq \underrho_{j,r,i}^{(t)}-32n\sqrt{\frac{\log(4n^2/\delta)}{d}}\log(T^*)$, which completes the proof.
    \item In this part of proof, $y_i\not= j$. We have 
    \begin{align}
        \la\wb_{j,r}^{(t)},y_i\bmu \ra=\la\wb_{j,r}^{(0)},y_i\bmu \ra+y_i\cdot j\cdot\gamma_{j,r}^{(t)}\leq \la\wb_{j,r}^{(0)},y_i\bmu \ra,\label{eq:addmissible_tech2_1}
    \end{align}
    where the equality is by \eqref{def:signal_to_noise_decomp}, and the inequality is by $\gamma_{j,r}^{(t)}\geq 0$ and $y_i\cdot j=-1$. In addition, we have 
    \begin{align}
        \la\wb_{j,r}^{(t)},\bxi_i \ra\leq\la\wb_{j,r}^{(0)},\bxi_i \ra +\underrho_{j,r,i}^{(t)}+32n\sqrt{\frac{\log(4n^2/\delta)}{d}}\log(T^*)\leq \la\wb_{j,r}^{(0)},\bxi_i \ra +32n\sqrt{\frac{\log(4n^2/\delta)}{d}}\log(T^*),\label{eq:addmissible_tech2_2}
    \end{align}
    where the first inequality is proved in the first part, and the second inequality is by $\underrho_{j,r,i}^{(t)}\leq0$. Then we can get that
    \begin{align*}
        F_j(\Wb_j^{(t)},\xb_i)&=\frac{1}{m}\sum_{r=1}^m[\sigma(\la\wb_{j,r}^{(t)},-j\bmu  \ra)+\sigma(\la\wb_{j,r}^{(t)},\bxi_i  \ra)]\\
        &\leq 2^3\max_{j,r,i}\bigg\{|\la\wb_{j,r}^{(0)},\bmu \ra|,|\la\wb_{j,r}^{(0)},\bxi_i \ra|, 32n\sqrt{\frac{\log(4n^2/\delta)}{d}}\log(T^*)\bigg\}^2\leq O(1).
    \end{align*}
    Here, the first inequality is by \eqref{eq:addmissible_tech2_1} and \eqref{eq:addmissible_tech2_2}, and the second inequality is by the condition $\gamma_{j,r}^{(0)}=O(1)$.  
\item In this part, $j=y_i$. Similar to \eqref{eq:addmissible_tech2_1}, we have
\begin{align}
        \la\wb_{j,r}^{(t)},y_i\bmu \ra=\la\wb_{j,r}^{(0)},y_i\bmu \ra+y_i\cdot j\cdot\gamma_{j,r}^{(t)}=\la\wb_{j,r}^{(0)},y_i\bmu \ra+\gamma_{j,r}^{(t)}.\label{eq:addmissible_tech2_3}
    \end{align}
    The second inequality in part 3 can be directly obtained from part 1. As for the last equality $F_j(W_j^{(t)},\xb_i)=O(1)$, it comes from 
    \begin{align*}
        F_j(W_j^{(t)},\xb_i)&=\frac{1}{m}\sum_{r=1}^m[\sigma(\la\wb_{j,r}^{(t)},-j\bmu  \ra)+\sigma(\la\wb_{j,r}^{(t)},\bxi_i  \ra)]\\
        &\leq 2\cdot 3^2\max_{j,r,i}\bigg\{\gamma_{j,r}^{(t)},\overrho_{j,r,i}^{(t)},|\la\wb_{j,r}^{0},\bmu \ra|,|\la\wb_{j,r}^{0},\bxi_i \ra|, 32n\sqrt{\frac{\log(4n^2/\delta)}{d}}\log(T^*)\bigg\}^2=O(1).
    \end{align*}
\end{enumerate}
Hence, the proof of Lemma~\ref{lemma:admissible_lemma} is completed.
\end{proof}
We are now well prepared for the proof of Proposition~\ref{prop:admissible_time_bound_appendix}.
\begin{proof}[Proof of Proposition~\ref{prop:admissible_time_bound_appendix}]
The proof is based on mathematical induction. For $t=0$, the conditions for $\gamma_{j,r}^{(0)}$  and $\rho_{j,r,i}^{(0)}$ clearly show that the results in Proposition~\ref{prop:admissible_time_bound_appendix} hold. Suppose that there exists $T_0\leq T^*$ such that the results in Proposition~\ref{prop:admissible_time_bound_appendix} hold for all time $\tT_1\leq t\leq T_0-1$, we aim to prove they also hold for $t=T_0$.

We first prove that \eqref{eq:admissible_underrho} holds for $\underrho_{j,r,i}^{(t)}$ when $t=T_0$. We only need to consider $j\not= y_i$ due to $\underrho_{j,r,i}^{(t)}=0$ for $\forall j=y_i$. Recall $\beta=2\max_{i,j,r}\{|\la w_{j,r}^{(0)},\bmu\ra|,|\la w_{j,r}^{(0)},\bxi_i\ra|\}$, when $\underrho_{j,r,i}^{(t)}\leq -0.5\beta-32n\sqrt{\frac{\log(4n^2/\delta)}{d}}\log(T^*)$, from the first term in Lemma~\ref{lemma:admissible_lemma}, we have
\begin{align*}
    \la \wb_{j,r}^{(T_0-1)},\bxi_i\ra\leq \underrho_{j,r,i}^{(T_0-1)}+\la \wb_{j,r}^{(0)},\bxi_i\ra+32n\sqrt{\frac{\log(4n^2/\delta)}{d}}\log(T^*)\leq 0,
\end{align*}
and thus 
\begin{align*}
    \underrho_{j,r,i}^{(T_0)}&=\underrho_{j,r,i}^{(T_0-1)}+\frac{2\eta}{nm}{\ell'^{(t)}_i} \ReLU( \la\wb_{j,r}^{(T_0-1)},\bxi_i\ra)=\underrho_{j,r,i}^{(T_0-1)}\|\bxi_i\|^2\\
    &\geq -\beta-64n\sqrt{\frac{\log(4n^2/\delta)}{d}}\log(T^*).
\end{align*}
When $\underrho_{j,r,i}^{(t)}\geq -0.5\beta-32n\sqrt{\frac{\log(4n^2/\delta)}{d}}\log(T^*)$, we have 
\begin{align*}
    \underrho_{j,r,i}^{(T_0)}&=\underrho_{j,r,i}^{(T_0-1)}+\frac{2\eta}{nm}{\ell'^{(t)}_i} \ReLU( \la\wb_{j,r}^{(T_0-1)},\bxi_i\ra)\\
    &\geq -0.5\beta-32n\sqrt{\frac{\log(4n^2/\delta)}{d}}\log(T^*)-O\bigg(\frac{2\eta\sigma_p^2d}{nm} \bigg)\bigg(0.5\beta+32n\sqrt{\frac{\log(4n^2/\delta)}{d}}\log(T^*)\bigg)\\
    &\geq -\beta-64n\sqrt{\frac{\log(4n^2/\delta)}{d}}\log(T^*).
\end{align*}
Here the first inequality comes from the second term in Lemma~\ref{lemma:admissible_lemma}, and the last inequality utilizes the fact that $\eta=o(\frac{nm}{\sigma_p^2d})$ in Condition~\ref{condition:condition}.

Next we prove \eqref{eq:admissible_gamma_overrho} at iteration $T_0$.  Recall the update rule 
\begin{align*}
    \gamma_{j,r}^{(t+1)}&=\gamma_{j,r}^{(t)}-\frac{2\eta}{nm}\sum_{i=1}^{n}{\ell'^{(t)}_i} \ReLU( \la\wb_{j,r}^{(t)}, y_i\cdot\bmu \ra)\|\bmu\|^2,\\
    \overrho_{j,r,i}^{(t+1)}&=\overrho_{j,r,i}^{(t)}-\frac{2\eta}{nm}{\ell'^{(t)}_i} \ReLU( \la\wb_{j,r}^{(t)}, \bxi_i \ra)\|\bxi_i\|^2,
\end{align*}
and the loss function 
\begin{align}
    |\ell'^{(t)}_i|&=\frac{1}{1+\exp\{y_i[F_{+1}(W_{+1}^{(t)},\xb_i)-F_{-1}(W_{-1}^{(t)},\xb_i)]\}}\nonumber\\
    &\leq \exp\{-y_i[F_{+1}(W_{+1}^{(t)},\xb_i)-F_{-1}(W_{-1}^{(t)},\xb_i)]\}\leq \exp\{-F_{y_i}(W_{y_i}^{(t)},\xb_i)+O(1)\},\label{eq:admissible_loss_bound}
\end{align}
where the last inequality  is due to the second term in Lemma~\ref{lemma:admissible_lemma}. Let $t_{i,j,r}$ be the last time $t<T^*$ that $\overrho_{j,r,i}^{(t)}\leq 2\log(T*)$, then we have
\begin{align*}
    \overrho_{j,r,i}^{(T_0)}&=\overrho_{j,r,i}^{(t_{j,r,i})}-\underbrace{\frac{2\eta}{nm}{\ell'^{(t_{j,r,i})}_i} \ReLU( \la\wb_{j,r}^{(t_{j,r,i})}, \bxi_i \ra)\|\bxi_i\|^2\one\{j=y_i\}}_{I_1}\\
    &\quad \underbrace{-\frac{2\eta}{nm}\sum_{t_{j,r,i}<t<T_0}{\ell'^{(t)}_i} \ReLU( \la\wb_{j,r}^{(t)}, \bxi_i \ra)\|\bxi_i\|^2\one\{j=y_i\}}_{I_2}.
\end{align*}
We first bound $I_1$. From the third term in Lemma~\ref{lemma:admissible_lemma}, we have 
\begin{align*}
    |I_1|\leq \frac{4\eta\sigma_p^2d}{nm}\bigg|\la w_{j,r}^{(0)},\bxi_i\ra+\overrho_{j,r,i}^{(t)}+32n\sqrt{\frac{\log(4n^2/\delta)}{d}}\log(T^*) \bigg|\leq 16\log(T^*)\cdot\frac{\eta\sigma_p^2d}{nm}\ll \log(T^*).
\end{align*}
The last inequality utilizes $\eta=o(\frac{nm}{\sigma_p^2d})$ in Condition~\ref{condition:condition}. Second,  we bound $I_2$. For $t_{j,r,i}<t<T_0$ and $j=y_i$, we can first lower bound $\la w_{j,r}^{(t)},\bxi_i\ra$,
\begin{align*}
    \la w_{j,r}^{(t)},\bxi_i\ra&\geq\overrho_{j,r,i}^{(t)}-32n\sqrt{\frac{\log(4n^2/\delta)}{d}}\log(T^*)+\la w_{j,r}^{(0)},\bxi_i\ra\\
    &\geq 2\log(T^*)-0.5 \log(T^*)-0.5\log(T^*)=\log(T^*),
\end{align*}
the first inequality comes from the first term in \ref{lemma:admissible_lemma} and the second inequality comes from Condition~\ref{condition:condition} which assume $d$ large enough. We can then similarly give the upper bound of $\la w_{j,r}^{(t)},\bxi_i\ra$, that is 
\begin{align*}
    \la w_{j,r}^{(t)},\bxi_i\ra&\leq\overrho_{j,r,i}^{(t)}+32n\sqrt{\frac{\log(4n^2/\delta)}{d}}\log(T^*)+\la w_{j,r}^{(0)},\bxi_i\ra\\
    &\leq 4\log(T^*)+0.5 \log(T^*)+0.5\log(T^*)=5\log(T^*).
\end{align*}
Plugging the upper bound and lower bound into $I_2$ gives 
\begin{align*}
    |I_2|&\leq \frac{2\eta}{nm}\sum_{t_{j,r,i}<t<T_0}{\exp\{-\ReLU^2(\la\wb_{j,r}^{(t)}, \bxi_i \ra)+O(1)\}} \ReLU( \la\wb_{j,r}^{(t)}, \bxi_i \ra)\|\bxi_i\|^2\\
    &\leq O\bigg(\frac{\eta T^*}{nm}\cdot \exp\{-\log^{1.5}(T^*)\}\cdot 5\log(T^*)\sigma_p^2d\bigg)\ll \log(T^*).
\end{align*}
Here, the first inequality is by \eqref{eq:admissible_loss_bound}, the second inequality is by the upper and lower bound of $\la w_{j,r}^{(t)},\bxi_i\ra$, and the last inequality is by $\eta=\tilde{O}(\frac{nm}{\sigma_p^2d})$ in Condition~\ref{condition:condition}.
With $I_1,I_2$ bounded by $\log(T^*)$, we conclude that $\overrho_{j,r,i}^{(T_0)}\leq\overrho_{j,r,i}^{(t_{j,r,i})}+|I_1|+|I_2|\leq 4\log(T^*)$. The proof for $\overrho_{j,r,i}^{(t)}\leq 4\log(T^*)$ is completed. Similarly, we can prove $\gamma_{j,r}^{(t)}\leq 4\log(T^*)$ with condition $\eta=o(\frac{nm}{\|\bmu\|^2})$, therefore Proposition~\ref{prop:admissible_time_bound_appendix} holds at time $T_0$, which completes
the induction.
\end{proof}
With Proposition~\ref{prop:admissible_time_bound_appendix} above, we can also have the following lemma, which gives an upper bound of the training gradient.
\begin{lemma}[Restatement of Lemma~\ref{lemma:admissible_trainloss}]
    \label{lemma:admissible_trainloss_appendix}
    Under Condition~\ref{condition:condition}, for $0\leq t\leq T^*$ where $T^*=\eta^{-1}\poly({n,m,d,\varepsilon^{-1},\sigma_0^{-1},\sigma_p\sqrt{d}})$ is the maximum admissible iterations. the following result holds.
    \begin{align*}
        \| \nabla_{\Wb} L_S(\Wb)|_{\Wb=\Wb^{(t)}}\|_F^2\leq 72\sigma_p^2d L_S(\Wb^{(t)}).
    \end{align*}
\end{lemma}
\begin{proof}[Proof of Lemma~\ref{lemma:admissible_trainloss_appendix}]
We first prove that
\begin{align*}
    -\ell'^{(t)}(y_if(\Wb^{(t),\xb_i}))\cdot \| \nabla f(\Wb^{(t)},\xb_i)\|_F^2\leq 72\sigma_p^2d.
\end{align*}
From \eqref{eq:admissible_withoutregu1} and \eqref{eq:admissible_withoutregu2}, Lemma~\ref{lemma:admissible_lemma} still holds under this case. Therefore, for $j\not=y_i$, $F_j(\Wb_{j}^{(t)},\xb_i)\leq 1$. Without loss of generality, we suppose that $y_i=1$ and $\xb_i=[\bmu,\bxi_i]$. Then we have that 
\begin{align*}
    \| \nabla f(\Wb^{(t)},\xb_i)\|_F&\leq \frac{1}{m}\sum_{j,r}\Big\|\sigma'(\la\wb_{j,r}^{(t)},\bmu\ra)\bmu+ \sigma'(\la\wb_{j,r}^{(t)},\bxi_i\ra)\bxi_i \Big\|\\
    &\leq \frac{1}{m}\sum_{j,r}\sigma'(\la\wb_{j,r}^{(t)},\bmu\ra)\|\bmu\|+ \sigma'(\la\wb_{j,r}^{(t)},\bxi_i\ra)\|\bxi_i\|\\
    &\leq 4\big(F_{+1}(\Wb_{+1}^{(t)},\xb_i)\big)^{1/2}\cdot 1.5\sigma_p\sqrt{d}+4\big(F_{-1}(\Wb_{-1}^{(t)},\xb_i)\big)^{1/2}\cdot 1.5\sigma_p\sqrt{d}\\
    &\leq 4\big[\big(F_{+1}(\Wb_{+1}^{(t)},\xb_i)\big)^{1/2}+1\big]\cdot 1.5\sigma_p\sqrt{d},
\end{align*}
where the first and second inequality is by triangle inequality, the third inequality is by Jensen inequality and Lemma~\ref{lemma:data_noise_concentration}, and the last inequality is by $F_{-1}(\Wb_{-1}^{(t)},\xb_i)\leq 1$. Denote by   $A=F_{+1}(\Wb_{+1}^{(t)},\xb_i)$, and besides, $F_{-1}(\Wb_{-1}^{(t)},\xb_i)\leq 1$ from Lemma~\ref{lemma:admissible_lemma}, we have
\begin{align}
    &-\ell'^{(t)}(y_if(\Wb^{(t)},\xb_i))\cdot \| \nabla f(\Wb^{(t)},\xb_i)\|_F^2\nonumber\\
    &\qquad \leq -\ell'(A-1)\cdot36(A^{1/2}+1)^2\cdot\sigma_p^2d\leq 72\sigma_p^2d.\label{eq:admissble_l'}
\end{align}
Here, the last inequality is by the fact $\exp(-A+1)/(1+\exp(-A+1))\cdot(\sqrt{A}+1)^2\leq2$ for any $A>0$. Now, we can upper bound the gradient norm as follows.
\begin{align*}
    \| \nabla_{\Wb} L_S(\Wb)|_{\Wb=\Wb^{(t)}}\|_F^2&\leq \bigg[\frac{1}{n}\sum_{i=1}^n\ell'^{(t)}\big(y_if(\Wb^{(t)},\xb_i)\big)\| \nabla f(\Wb^{(t)},\xb_i)\|_F \bigg]^2\\
    &\leq \frac{72\sigma_p^2d}{n^2}\bigg[\sum_{i=1}^n\sqrt{-\ell'^{(t)}\big(y_if(\Wb^{(t)},\xb_i)\big)}\bigg]^2\\
    &\leq \frac{72\sigma_p^2d}{n}\sum_{i=1}^n -\ell'^{(t)}\big(y_if(\Wb^{(t)},\xb_i)\big)\\
    &\leq \frac{72\sigma_p^2d}{n}\sum_{i=1}^n \ell^{(t)}\big(y_if(\Wb^{(t)},\xb_i)\big)=72\sigma_p^2d L_S(\Wb^{(t)}),
\end{align*}
where the first inequality is by triangle inequality, the second inequality is by \eqref{eq:admissble_l'}, the third inequality is by Cauchy-Schwartz inequality  and the last inequality is due to the inequality $-\ell'\leq\ell$. 
\end{proof}

\section{Analysis with gradient regularization}
\label{sec:withregular}
In this section, we present the results on the signal-noise decomposition. All the results are hold with high probability. Denote by $\cE_{\pre}$ the event that all the results in Section~\ref{sec:concentration_appendix} hold. Then for simplicity and clarity, all the results in this and following sections are conditional on the event $\cE_{\pre}$. 
\subsection{Stage 1: Signal learning and noise memorizing}

We will now show that under some assumptions and conditions, the parameter of the signal-noise decomposition will stay a reasonable scale during a long time of training. 
Let $\lambda=\sigma_p^{-1}d^{-1/2}$, and note that the following conditions hold: $\|\bmu\|^{-1}\lambda^{-1/2}\to+\infty$.


\begin{proposition}
\label{prop:Phase_I_prop}
 Under Condition~\ref{condition:condition}, if 
 \begin{align}
     0\leq\gamma_{j,r}^{(t)}\leq \Theta\bigg(\frac{1}{\log(n)}\bigg)\label{eq:gammacondition_firstphase} 
 \end{align}
for all $j\in\{\pm1\}$, $r\in[m]$ and $t$ in the training interval $[0,T_*]$, then we have 
\begin{align}
    0\geq\rho_{j,r,i}^{(t)}\geq -8\sqrt{\log(8mn/\delta)}\cdot\sigma_0\sigma_p\sqrt{d} \label{eq:rhocondition_firstphase}
\end{align}
for all $j\in\{\pm1\}$, $r\in[m]$ and $0\leq t \leq T_*$. 
\end{proposition}
Before we present  the proof, there are several lemmas listed which will be applied to the proof of Proposition~\ref{prop:Phase_I_prop}.
\begin{lemma}
\label{lemma:technique_lemma}
Suppose that \eqref{eq:gammacondition_firstphase} and \eqref{eq:rhocondition_firstphase} hold, then under Condition~\ref{condition:condition} we have 
\begin{align*}
    &\big(\la\wb_{j,r}^{(0)},\bmu\ra+\gamma_{j,r}^{(t)}\big)^{2}\cdot\|\bmu\|^2\leq\Theta\bigg( \frac{1}{\lambda\cdot \log(n)}\bigg),\\
    &\big(\la\wb_{j,r}^{(0)},\bxi_i\ra+\sum_{i'\not=i}^n\rho_{j,r,i'}^{(t)}\frac{\la\bxi_{i'},\bxi_i\ra}{\|\bxi_i\|^2}+\rho_{j,r,i}^{(t)}\big)^{2}\cdot\|\bxi_i\|^2\leq \Theta\bigg(\frac{1}{\lambda\cdot \log(n)}\bigg).
\end{align*}
\end{lemma}
\begin{proof}[Proof of Lemma~\ref{lemma:technique_lemma}]
From Condition~\ref{condition:condition}, we 
see that 
$1\ll\|\bmu\|^{-1}\lambda^{-1/2}$, therefore \eqref{eq:gammacondition_firstphase} indicates that 
\begin{align*}
    0\leq\gamma_{j,r}^{(t)}\leq \Theta\bigg( \frac{1}{\log{(n)}\cdot\|\bmu\|\lambda^{1/2}}\bigg).
\end{align*}
From Lemma~\ref{lemma:data_CNN_concentration}, with high probability we have
\begin{align*}
    \big(\la\wb_{j,r}^{(0)},\bmu\ra+\gamma_{j,r}^{(t)}\big)^2\leq \gamma_{j,r}^{2(t)}+2\gamma_{j,r}^{(t)}\cdot \sqrt{2\log(8m/\delta)}\cdot\sigma_0\|\bmu\|+2\log(8m/\delta)\sigma_0^2\|\bmu\|^2.
\end{align*}
Note that Condition~\ref{condition:condition} and $\gamma_{j,r}^{(t)}\leq1$ further give
\begin{align*}
    &2\gamma_{j,r}^{(t)}\cdot \sqrt{2\log(8m/\delta)}\cdot\sigma_0\|\bmu\|+2\log(8m/\delta)\sigma_0^2\|\bmu\|^2\\
    &\qquad\quad=\frac{1}{\log{(n)}}\cdot\tilde\Theta\big(\sigma_p^{-2}n^{-\frac{\alpha}{2}}d^{-1}\|\bmu\|+\sigma_p^{-4}d^{-2}\|\bmu\|^2\big)\\
    &\qquad\quad\ll\Theta\bigg( \frac{1}{2\lambda\cdot \log{(n)}\sigma_p\sqrt{d}}\bigg)\leq\Theta\bigg(\frac{1}{\lambda\|\bmu\|^2\cdot\log{(n)}}\bigg),
\end{align*}
where the first equality is by $\sigma_0\leq O\Big(\frac{1}{\sigma_p^{2}d\cdot (nm)
^{2\alpha}}\Big)$, the first inequality utilizes the fact $\lambda=\sigma_p^{-1}d^{-1/2}$ and $\sigma_p\sqrt{d}\to+\infty$, the last inequality holds from the fact $\|\bmu\|^2/(\sigma_p\sqrt{d})\to0$. We have
\begin{align*}
    \big(\la\wb_{j,r}^{(0)},\bmu\ra+\gamma_{j,r}^{(t)}\big)^2\|\bmu\|^2\leq \gamma_{j,r}^{2(t)}\|\bmu\|^2+\Theta\bigg(\frac{1}{\lambda\cdot \log{(n)}}\bigg)\leq \Theta\bigg(\frac{1}{\lambda\cdot \log{(n)}}\bigg).
\end{align*}
Here, the last equality follows from the fact $\|\bmu\|^2\ll\lambda^{-1}$.
Hence the proof for the first inequality has been completed. To prove the second inequality, recall that Lemma~\ref{lemma:data_noise_concentration} gives us
\begin{align*}
    \sigma_p^2d/2\leq\| \bxi_i\|^2\leq3\sigma_p^2d/2,\quad|\la\bxi_i,\bxi_{i'}\ra|\leq2\sigma_p^2\cdot\sqrt{d\log(4n^2/\delta)},
\end{align*}
Lemma~\ref{lemma:data_CNN_concentration} gives us 
\begin{align*}
    |\la\wb_{j,r}^{(0)},\bxi_i\ra|\leq2\sqrt{\log(8mn/\delta)}\cdot\sigma_0\sigma_p\sqrt{d},
\end{align*}
we have 
\begin{align}
    &\big(\la\wb_{j,r}^{(0)},\bxi_i\ra+\sum_{i'\not=i}^n\rho_{j,r,i'}^{(t)}\frac{\la\bxi_{i'},\bxi_i\ra}{\|\bxi_i\|^2}+\rho_{j,r,i}^{(t)}\big)^{2}\cdot\|\bxi_i\|^2\nonumber\\
    &\qquad=\Theta\bigg(\sigma_p^2d\cdot\max\Big\{|\la\wb_{j,r}^{(0)},\bxi_i\ra|^2,\Big|\sqrt{\log(8mn/\delta)}\cdot\sigma_0\sigma_p\sqrt{d}\cdot n\sqrt{\log(4n^2/\delta)/d}\Big|^2,\Big|\sqrt{\log(8mn/\delta)}\cdot\sigma_0\sigma_p\sqrt{d}\Big|^2\Big\} \bigg)\nonumber\\
    &\qquad= \Theta\bigg(\sigma_p^2d\cdot \Big|\sqrt{\log(8mn/\delta)}\cdot\sigma_0\sigma_p\sqrt{d}\Big|^2\bigg)\leq \Theta\bigg(\frac{1}{\lambda\cdot \log{(n)}}\bigg).\label{eq:clue_A}
\end{align}
The last equality comes from the truth  $d\gg n^2\sqrt{\log(4n^2/\delta)}$ and  $\sigma_0\leq O\Big(\frac{1}{\sigma_p^{2}d\cdot (nm)
^{2\alpha}}\Big)$ ($\sigma_0$ sufficiently small).
\end{proof}

We further have the lemma which states that under condition \eqref{eq:gammacondition_firstphase} and \eqref{eq:rhocondition_firstphase}, the gradient loss can be considered as a constant.

\begin{lemma}
\label{lemma:ell_scope}
Assume $\ell'^{(t)}_i=-\frac{1}{2}+\Upsilon_i^{(t)}$,
under conditions \eqref{eq:gammacondition_firstphase} and \eqref{eq:rhocondition_firstphase}, 
we have
\begin{align*}
    |\Upsilon_i^{(t)}|=O\bigg(\frac{1}{\log^2{(n)}}\bigg).
\end{align*}
\end{lemma}

\begin{proof}[Proof of Lemma~\ref{lemma:ell_scope}]
Recall that 
\begin{align*}
    |\Upsilon_i^{(t)}|=\bigg|\ell'^{(t)}_i+\frac{1}{2}\bigg|=\bigg|\frac{1-\exp\{y_i[F_{+1}(W_{+1}^{(t)},\xb_i)-F_{-1}(W_{-1}^{(t)},\xb_i)]\}}{2+2\exp\{y_i[F_{+1}(W_{+1}^{(t)},\xb_i)-F_{-1}(W_{-1}^{(t)},\xb_i)]\}}\bigg|.
\end{align*}
To prove 
$
    |\Upsilon^{(t)}_i|= O\big(\frac{1}{\log^2{(n)}}\big),
$ easy to see it is equal to prove 
$\exp\{y_i[F_{+1}(W_{+1}^{(t)},\xb_i)-F_{-1}(W_{-1}^{(t)},\xb_i)]\}=1\pm O\Big(\frac{1}{\log^2{(n)}}\Big)$. We now investigate  $F_{j}(W_{j}^{(t)},\xb_i)$. Note that 
\begin{align*}
    F_{j}(W_{j}^{(t)},\xb_i)&=\frac{1}{m}\sum_{r=1}^m\sigma(\la\wb_{j,r}^{(t)}, y_i\bmu \ra )+\sigma(\la\wb_{j,r}^{(t)}, \bxi_i \ra )\\
    &\leq\max_{j,r}\big(|\la\wb_{j,r}^{(0)},\bmu\ra|+\gamma_{j,r}^{(t)}\big)^{2}+\max_{j,r,i}\big(\la\wb_{j,r}^{(0)},\bxi_i\ra+\sum_{i'\not=i}^n\rho_{j,r,i'}^{(t)}\frac{\la\bxi_{i'},\bxi_i\ra}{\|\bxi_i\|^2}+\rho_{j,r,i}^{(t)}\big)^{2}\\
    &\leq |\gamma_{j,r}^{2(t)} |+o\bigg(\frac{1}{\log^2(n)}\bigg) \leq O\bigg(\frac{1}{\log^2{(n)}}\bigg), 
\end{align*}
where the second inequality comes from $\sigma_0$ small enough, we have 
\begin{align*}
   \big|y_i[F_{+1}(W_{+1}^{(t)},\xb_i)-F_{-1}(W_{-1}^{(t)},\xb_i)] \big|=O\bigg(\frac{1}{\log^2{(n)}}\bigg),
\end{align*}
which directly shows $\exp\{y_i[F_{+1}(W_{+1}^{(t)},\xb_i)-F_{-1}(W_{-1}^{(t)},\xb_i)]\}=1\pm O\big(\frac{1}{\log^2{(n)}}\big)$.
\end{proof}
With these two lemmas above, we may apply mathematical induction to prove Proposition~\ref{prop:Phase_I_prop}.
\begin{proof}[Proof of Proposition~\ref{prop:Phase_I_prop}]
It is easy to check that $t=0$, the inequalities \eqref{eq:gammacondition_firstphase} and \eqref{eq:rhocondition_firstphase} hold, and $T_0>1$ under Condition~\ref{condition:condition} due to the smallness of $\eta$. 
Suppose that the inequalities \eqref{eq:gammacondition_firstphase} and \eqref{eq:rhocondition_firstphase} hold for $t$ in  $0\leq t\leq T_0<T_*$, from Lemma~\ref{lemma:technique_lemma} we have for $t\in[T_0]$,
\begin{align*}
    \frac{4\lambda\zeta_i^{(t)}}{m}\cdot\ell''^{(t)}_i&=O\bigg(\lambda\Big( \la \wb_{j',r'}^{(0)},  y_i\bmu \ra+j'y_i\cdot\gamma_{j',r'}^{(t)} \Big)^2\|\bmu\|^2+\lambda\Big( \la \wb_{j',r'}^{(0)}, \bxi_i \ra +\sum_{i'=1}^n\rho_{j',r',i'}^{(t)}\frac{\la\bxi_i,\bxi_{i'}\ra}{\|\bxi_{i'}\|^2}\Big)^2\|\bxi_i\|^2\bigg)\\
    &=O(1/\log^2{(n)})\ll1,
\end{align*}
Therefore, when $\ell'^{(t)}_i=\Theta(1)$, $\lambda\|\bxi_i\|^2/m=\Theta(\sigma_p\sqrt{d}/m)\gg1$,  we have
\begin{align*}
   &\frac{|{\ell'^{(t)}_i}|}{2} \ReLU\Big( \la \wb_{j,r}^{(0)},  y_i\bmu \ra+jy_i\cdot\gamma_{j,r}^{(t)} \Big)\|\bmu\|^2\geq\frac{  2\lambda}{m}{\ell'^{2(t)}_i}\ReLU\Big( \la \wb_{j,r}^{(0)},  y_i\bmu \ra+jy_i\cdot\gamma_{j,r}^{(t)} \Big)\|\bmu\|^4,\\
    &\frac{\lambda}{m}{\ell'^{2(t)}_i}\ReLU\Big( \la \wb_{j,r}^{(0)}, \bxi_i \ra +\sum_{i'=1}^n\rho_{j,r,i'}^{(t)}\frac{\la\bxi_i,\bxi_{i'}\ra}{\|\bxi_{i'}\|^2}\Big)\|\bxi_i\|^4\geq|{\ell'^{(t)}_i}| \ReLU\Big( \la \wb_{j,r}^{(0)}, \bxi_i \ra +\sum_{i'=1}^n\rho_{j,r,i'}^{(t)}\frac{\la\bxi_i,\bxi_{i'}\ra}{\|\bxi_{i'}\|^2}\Big)\|\bxi_i\|^2,
\end{align*}
thus from \eqref{eq:gamma_t_express} and \eqref{eq:rho_t_express} we have
\begin{align*}
\gamma_{j,r}^{(t+1)}\geq\gamma_{j,r}^{(t)}\geq0,\rho_{j,r,i}^{(t+1)}\leq \rho_{j,r,i}^{(t)}\leq0\qquad\text{for    }t\in[T_*].
\end{align*}
This directly indicates that 
\begin{align*}
\gamma_{j,r}^{(T_0+1)}\geq\gamma_{j,r}^{(T_0)}\geq0,\rho_{j,r,i}^{(T_0+1)}\leq \rho_{j,r,i}^{(T_0)}\leq0.
\end{align*}

By Gaussian tail bounds, there exists $A = 2\sqrt{\log(8mn/\delta)}\cdot\sigma_0\sigma_p\sqrt{d}$ such that with high probability, 
$|\la \wb_{j,r}^{(0)}, \xi_i \ra| \leq A $.
We next show that
\begin{align}
    \max_{t\leq T_0+1} \max_{j,r,i} |\rho_{j,r,i}^{(t)}| \leq 4A.
\end{align}
If this does not hold, then $T_0+1$ is the first time such that there exists some $j_0,r_0,i_0$ with 
\begin{align}
\label{eq:pho_firstphase_p=2_condition_vialate}
     |\rho_{j_0,r_0,i_0}^{(\tau)}| > 4A.
\end{align}
Moreover, let $\tau$ be the first iteration such that  
\begin{align*}
    |\rho_{j_0,r_0,i_0}^{(\tau)}| > 2A.
\end{align*}
With a small enough learning rate $\eta$, it is clear that $\tau<T_0+1$. Rigorously, given condition \eqref{eq:rhocondition_firstphase}, from \eqref{eq:clue_A} we have $|\la\wb_{j,r}^{(0)},\bxi_i\ra+\sum_{i'\not=i}^n\rho_{j,r,i'}^{(t)}\frac{\la\bxi_{i'},\bxi_i\ra}{\|\bxi_i\|^2}+\rho_{j,r,i}^{(t)}|=O(A)$, therefore if we assume that $\tau=T_0+1$, we have
\begin{align}
    |\rho_{j_0,r_0,i_0}^{(\tau)} |&\leq |\rho_{j_0,r_0,i_0}^{(\tau-1)}| + \Theta(\eta) \cdot \bigg[
    \frac{A  \| \xi_{i_0}\|^4}{nm} \bigg]\leq 4A.
\end{align}
This contradicts to the equation \eqref{eq:pho_firstphase_p=2_condition_vialate}. Then we  conclude that $\tau < T_0+1$. Now by the definition of $ \tau$, we have that 
\begin{itemize}
    \item For any $t \leq \tau$, and any $j,r,i$, $|\rho_{j,r,i}^{(t)}| \leq 4A$.
    \item For any $t\in [\tau,T_0+1]$, $ \rho_{j_0,r_0,i_0}^{(t)} < -  2A$.
\end{itemize}
Then it is clear that for any $t\in [\tau,T_0+1]$, we have
\begin{align*}
    \la \wb_{j_0,r_0}^{(0)}, \bxi_{i_0} \ra +\sum_{i'=1}^n\rho_{j_0,r_0,i'}^{(t)}\frac{\la\bxi_{i_0},\bxi_{i'}\ra}{\|\bxi_{i'}\|^2} &\leq  A + A\cdot 2n\sqrt{\log(4n^2/\delta)} / \sqrt{d} +  \rho_{j_0,r_0,i_0}^{(t)} \leq 0.
\end{align*}
Therefore this neuron is not activated for $t\in [\tau,T_0+1]$, and we have $\rho_{j_0,r_0,i_0}^{(\tau)} = \rho_{j_0,r_0,i_0}^{(\tau + 1)} = \rho_{j_0,r_0,i_0}^{(\tau+ 2)} = \cdots = \rho_{j_0,r_0,i_0}^{(T_0+1)}$. But this contradicts with the assumption that
\begin{align*}
     |\rho_{j_0,r_0,i_0}^{(\tau)}| > 4A.
\end{align*}
Therefore, we see that
\begin{align}
\label{eq:rho_firstphase_p=2_complete_mathinduction}
    \max_{t\leq T_0+1} \max_{j,r,i} |\rho_{j,r,i}^{(t)}| \leq  4A. 
\end{align}
\eqref{eq:rho_firstphase_p=2_complete_mathinduction} completes the mathematical induction for \eqref{eq:rhocondition_firstphase}.
\end{proof}

\begin{proposition}[Restatement of Proposition~\ref{prop:Phase_I_prop2}]
\label{prop:Phase_I_prop2_appendix}
Under the same conditions as Theorem~\ref{thm:withregularization}, define $\tT_{1}$ be an iteration time which satisfies
\begin{align*}
    T_1=\frac{m}{\eta\|\bmu\|^2}\log\Big(\frac{4}{\sigma_0\|\bmu\|\cdot\log{(n)}\sqrt{2\log(8m/\delta)}}\Big).
\end{align*}
Then the following facts hold:
\begin{enumerate}
    \item For any $i\in[n]$, $|\Upsilon_i^{(t)}|=|\ell'^{(t)}_i+\frac{1}{2}|= O(\frac{1}{\log^2{(n)}})$.
    \item $0\leq\gamma_{j,r}^{(t)}\leq 5/\log(n)$ for all $j\in\{\pm1\}$, $r\in[m]$ and $t\in[\tT_1]$.
    \item For  each $j$, 
     $\max_{r}\gamma_{j,r}^{(\tT_1)}\geq  1/(\sqrt{2\log(8m/\delta)}\log(n) )$. Moreover, for $\rho_{j,r,i}$ we have
    \begin{align*}
    0\geq\rho_{j,r,i}^{(\tT_1)}\geq -8\sqrt{\log(8mn/\delta)}\cdot\sigma_0\sigma_p\sqrt{d}.
\end{align*}
\end{enumerate}
\end{proposition}

\begin{proof}[proof of Proposition~\ref{prop:Phase_I_prop2_appendix}]
 We use mathematical induction to prove the second conclusion that $\gamma_{j,r}^{(t)}\leq 5/\log(n)$. Assume that the conclusion holds for $0\leq t\leq T_0-1<T_1$, we aim to prove the conclusion when $t=T_0$. 
From Proposition~\ref{prop:Phase_I_prop}, we have
$
    0\geq\rho_{j,r,i}^{(t)}\geq -8\sqrt{\log(8mn/\delta)}\cdot\sigma_0\sigma_p\sqrt{d}$ for $t\in[T_0-1]$,
therefore Lemma~\ref{lemma:ell_scope} indicates that $|\ell'^{(t)}_i|= \frac{1}{2}\pm O(\frac{1}{\log^2{(n)}})$ for all $t\in[T_0-1]$ and $i\in[n]$. 
For the inequality in $T_0$, note that for all $t\in[T_0]$, we have 
\begin{align*}
   &\frac{|{\ell'^{(t)}_i}|}{2} \ReLU\Big( \la \wb_{j,r}^{(0)},  y_i\bmu \ra+jy_i\cdot\gamma_{j,r}^{(t)} \Big)\|\bmu\|^2\gg \log^2(n)\cdot \frac{  2\lambda}{m}{\ell'^{2(t)}_i}\ReLU\Big( \la \wb_{j,r}^{(0)},  y_i\bmu \ra+jy_i\cdot\gamma_{j,r}^{(t)} \Big)\|\bmu\|^4,\\
    &\frac{4\lambda\zeta_i^{(t)}}{m}\cdot\ell''^{(t)}_i\ll O(\log^{-2}(n)),
\end{align*}
therefore we can get that 
\begin{align*}
    \gamma_{1,r}^{(t+1)}\geq\gamma_{1,r}^{(t)}+\bigg(\frac{1}{2}-O\bigg(\frac{1}{\log^2{(n)}}\bigg)\bigg)\frac{2\eta}{nm}\sum_{y_i=1}^{n} \ReLU\Big( \la \wb_{1,r}^{(0)},  \bmu \ra+\gamma_{1,r}^{(t)} \Big)\|\bmu\|^2.
\end{align*}
Set $t=T_0-1$ we have that $\gamma_{+1,r}^{(t+1)}\geq \gamma_{+1,r}^{(t)}\geq0$. Similarly we have $\gamma_{-1,r}^{(t+1)}\geq \gamma_{-1,r}^{(t)}\geq0$. Moreover, define $A^{(t)}=\max_{r}\gamma_{1,r}^{(t)}+\la \wb_{1,r}^{(0)},  \bmu \ra $ and $B^{(t)}=\max_{r}\gamma_{-1,r}^{(t)}+\la \wb_{-1,r}^{(0)},  \bmu \ra $, we then get that
\begin{align*}
    A^{(t+1)}&\geq A^{(t)}+\bigg(1-O\bigg(\frac{1}{\log^2{(n)}}\bigg)\bigg)\frac{\eta}{nm}\sum_{y_i=1}^{n} \ReLU\Big( A^{(t)} \Big)\|\bmu\|^2\\
    &\geq A^{(t)}+\bigg(1-O\bigg(\frac{1}{\log^2{(n)}}+\frac{1}{\sqrt{n}}\bigg)\bigg)\frac{\eta\|\bmu\|^2}{m}A^{(t)} \qquad\text{for all }t\in[T_0],
\end{align*}
where the second inequality comes from  the lower bound on the number of positive data in Lemma~\ref{lemma:data_count_yi=1}. Similarly,  we also have 
\begin{align*}
    A^{(t+1)}\leq A^{(t)}+\bigg(1+O\bigg(\frac{1}{\log^2{(n)}}+\frac{1}{\sqrt{n}}\bigg)\bigg)\frac{\eta\|\bmu\|^2}{m}A^{(t)} \qquad\text{for all }t\in[T_0].
\end{align*}
We conclude that for all $t\in[T_0]$,
\begin{align*}
    \left\{
    \begin{aligned}
    &\bigg(1+\bigg(1-O\bigg(\frac{1}{\log^2{(n)}}\bigg)\bigg)\frac{\eta\|\bmu\|^2}{m}\bigg)^{(t+1)}A^{(0)}\leq A^{(t+1)}\leq \bigg(1+\bigg(1+O\bigg(\frac{1}{\log^2{(n)}}\bigg)\bigg)\frac{\eta\|\bmu\|^2}{m}\bigg)^{(t+1)}A^{(0)},\\
    &\bigg(1+\bigg(1-O\bigg(\frac{1}{\log^2{(n)}}\bigg)\bigg)\frac{\eta\|\bmu\|^2}{m}\bigg)^{(t+1)}B^{(0)}\leq B^{(t+1)}\leq \bigg(1+\bigg(1+O\bigg(\frac{1}{\log^2{(n)}}\bigg)\bigg)\frac{\eta\|\bmu\|^2}{m}\bigg)^{(t+1)}B^{(0)}.
    \end{aligned}
    \right.
\end{align*}
Set $t=T_0-1$,  we can see that 
\begin{align*}
    \left\{
    \begin{aligned}
&A^{(T_0)}\leq \exp\bigg\{\bigg(1+O\bigg(\frac{1}{\log^2(n)}\bigg)\bigg)\cdot\frac{\eta\|\bmu\|^2}{m}T_0\bigg\}A^{(0)}\leq  \frac{\sqrt{2\log(8m/\delta)}\sigma_0\|\bmu\|\cdot4}{\sqrt{2\log(8m/\delta)}\sigma_0\|\bmu\|\log(n)}\leq 4/\log(n),\\
&B^{(T_0)}\leq \exp\bigg\{\bigg(1+O\bigg(\frac{1}{\log^2(n)}\bigg)\bigg)\cdot\frac{\eta\|\bmu\|^2}{m}T_0\bigg\}B^{(0)}\leq  \frac{\sqrt{2\log(8m/\delta)}\sigma_0\|\bmu\|\cdot4}{\sqrt{2\log(8m/\delta)}\sigma_0\|\bmu\|\log(n)}\leq 4/\log(n).
    \end{aligned}\right.
\end{align*}
Here, the second inequality is from Lemma~\ref{lemma:data_CNN_concentration}. We thus conclude that $\gamma_{j,r}^{(T_0)}\leq 5/\log(n)$ for all $j\in\{\pm1\}$ and $r\in[m]$.

For the other side, similarly we have
\begin{align*}
    \left\{
    \begin{aligned}
&A^{(T_0)}\geq \bigg[\exp\bigg\{\bigg(1-O\bigg(\frac{1}{\log^2(n)}\bigg)\bigg)\cdot\frac{\eta\|\bmu\|^2}{m}\bigg\}-\bigg( \bigg(1-O\bigg(\frac{1}{\log^2(n)}\bigg)\bigg)\cdot\frac{\eta\|\bmu\|^2}{m}\bigg)^2\bigg]^{(T_0)}A^{(0)},\\
&B^{(T_0)}\geq \bigg[\exp\bigg\{\bigg(1-O\bigg(\frac{1}{\log^2(n)}\bigg)\bigg)\cdot\frac{\eta\|\bmu\|^2}{m}\bigg\}-\bigg( \bigg(1-O\bigg(\frac{1}{\log^2(n)}\bigg)\bigg)\cdot\frac{\eta\|\bmu\|^2}{m}\bigg)^2\bigg]^{(T_0)}B^{(0)}. 
    \end{aligned}\right.
\end{align*}
Here, we utilize the fact that $1+z\geq \exp(z)-z^2$ for $0<z<0.1$. Easy to see that
\begin{align*}
    &\bigg[\exp\bigg\{\bigg(1-O\bigg(\frac{1}{\log^2(n)}\bigg)\bigg)\cdot\frac{\eta\|\bmu\|^2}{m}\bigg\}-\bigg( \bigg(1-O\bigg(\frac{1}{\log^2(n)}\bigg)\bigg)\cdot\frac{\eta\|\bmu\|^2}{m}\bigg)^2\bigg]^{T_0}\\
    &\quad \geq\exp\bigg\{\bigg(1-O\bigg(\frac{1}{\log^2(n)}\bigg)\bigg)\cdot\frac{\eta\|\bmu\|^2}{m}T_0\bigg\}\cdot\bigg(1-\bigg( \bigg(1-O\bigg(\frac{1}{\log^2(n)}\bigg)\bigg)\cdot\frac{\eta\|\bmu\|^2}{m}\bigg)^2\bigg)^{T_0}\\
    &\quad \geq 2/3\cdot \frac{4}{\sqrt{2\log(8m/\delta)}\sigma_0\|\bmu\|\log(n)}.
\end{align*}
Here, we utilize the fact $(1-z^2)^{1/z}\to1$ when $z\to0$ in the last inequality. Therefore we have 
\begin{align*}
    A^{(T_0)}\geq A^{(0)}\cdot2/3\cdot \frac{4}{\sqrt{2\log(8m/\delta)}\sigma_0\|\bmu\|\log(n)}\geq \frac{4}{3\sqrt{2\log(8m/\delta)}\log(n)},
\end{align*}
where the second inequality is by Lemma~\ref{lemma:data_CNN_concentration}.
Then we have $\max_{r}\gamma_{+1,r}^{(T_0)}\geq1/(\sqrt{2\log(8m/\delta)}\log(n))$ since $A^{(T_0)}=\max_{r}\la\wb_{+1,r}^{(0)},\bmu\ra+\gamma_{+1,r}$ and $\sigma_0$ is small. Similarly we have $\max_{r}\gamma_{-1,r}^{(T_0)}\geq1/(\sqrt{2\log(8m/\delta)}\log(n))$.
The first and    third conclusion is directly obtained from  Proposition~\ref{prop:Phase_I_prop} and Lemma~\ref{lemma:ell_scope}.
\end{proof}

\subsection{Stage 2: Convergence of training loss }
We close the regularization at the beginning of stage 2. Note that when we close the regularization at this time, 
Proposition~\ref{prop:admissible_time_bound_appendix} holds.
We remind readers the update rule
\begin{align*}
\gamma_{j,r}^{(t+1)}&=\gamma_{j,r}^{(t)}-\frac{2\eta}{nm}\sum_{i=1}^{n}{\ell'^{(t)}_i}\ReLU ( \la\wb_{j,r}^{(t)}, y_i\cdot\bmu \ra)\|\bmu\|^2\\
\rho_{j,r,i}^{(t+1)}&=\rho_{j,r,i}^{(t)}-\frac{2\eta}{nm}{\ell'^{(t)}_i}\ReLU ( \la\wb_{j,r}^{(t)}, \bxi_i \ra)\|\bxi_i\|^2\cdot jy_i
\end{align*} In this section, we aim to prove that the signal learning will grow to a constant order while the noise  memorizing remains to the small scale, thus the gradient loss is getting changed, no longer equal to $-1/2+o(1)$. From Proposition~\ref{prop:Phase_I_prop2_appendix}, we first        list several properties of $\gamma_{j,r}$ and $\rho_{j,r,i}$ at time $\tT_1$.
We have
\begin{enumerate}
    \item $0\leq \gamma_{j,r}^{(\tT_1)}\leq 4/\log(n)$ for any $j\in\{\pm1\}$ and $r\in[m]$. For each $j=\pm1$, $\max_r\gamma_{j,r}^{(\tT_1)}\geq1/(\sqrt{2\log(8m/\delta)}\log(n))$.
    \item $|\rho_{j,r,i}^{(t)}|\leq8\sqrt{\log(8mn/\delta)}\cdot\sigma_0\sigma_p\sqrt{d}$ for all $i\in[n]$, $j\in\{\pm1\}$, $r\in[m]$ and $0\leq t\leq \tT_1$.
\end{enumerate}
We give the next proposition, which shows that the training loss will converge to any $\varepsilon>0$. 
\begin{proposition}[Restatement of Proposition~\ref{prop:Phase_II_prop}]
\label{prop:Phase_II_prop_appendix}
 Under Condition~\ref{condition:condition}, for any $\varepsilon>0$, define $\varepsilon_0=1-e^{-\varepsilon}$, 
and 
\begin{align*}
    \tT_2=\frac{2nm}{\eta \varepsilon_0\|\bmu\|^2}\log\big(\sqrt{2\log(8m/\delta)}\cdot\log(n)d\big)
\end{align*}
then there exists $t\in [\tT_1,\tT_1+\tT_2]$ such that
\begin{align*}
    L_S{(\Wb^{(t)})}\leq \varepsilon.
\end{align*}
\end{proposition}

\begin{proof}[Proof of Proposition~\ref{prop:Phase_II_prop_appendix}]
It is clear that $\log(1+e^{-z})>\varepsilon$ equals to $\frac{e^{-z}}{1+e^{-z}}>\varepsilon_0$.

Suppose that there does not exist $t\in [\tT_1,\tT_1+\tT_2]$ such that $L_S{(\Wb^{(t)})}\leq \varepsilon$, then we have for each $t\in [\tT_1,\tT_1+\tT_2]$, $L_S{(\Wb^{(t)})}\geq \varepsilon$.   It means that for each $t$, there exists $i=i(t)$ such that $\ell^{(t)}_i\geq \varepsilon$. The index $i$ here depends on the iteration time $t$. Then there exists at least $\tT_2/2$ indexes $i(t)$ such that $i(t)=1$ (or $i(t)=-1$) for $t\in[\tT_1,\tT_1+\tT_2]$. Without loss of generality, we assume that there exists at least $\tT_2/2$ indexes $i(t)$ such that $i(t)=1$. Hence we have there exists at least $\tT_2/2$ iterations, such that $-\ell'^{(t)}_i\geq \varepsilon_0$ for some $y_i=1$. By the update rule
\begin{align*}
    \gamma_{j,r}^{(t+1)}&=\gamma_{j,r}^{(t)}-\frac{2\eta}{nm}\sum_{i=1}^{n}{\ell'^{(t)}_i}\ReLU ( \la\wb_{j,r}^{(t)}, y_i\cdot\bmu \ra)\|\bmu\|^2, 
\end{align*}
Define $A(t)=\max_{r}\gamma_{+1,r}^{(t)}+\la \wb_{+1,r}^{(0)},\bmu\ra$, we have $A(t)$ increases when $t$ increases, there exists at least $\tT_2/2$ times, such that 
\begin{align*}
    A(t+1)\geq A(t)+\frac{2\eta\varepsilon_0}{nm}\ReLU(A(t))\|\bmu\|^2,
\end{align*}
thus we have
\begin{align*}
    A(\tT_1+\tT_2)\geq A(\tT_1)\cdot \bigg(1+\frac{2\eta\varepsilon_0}{nm}\bigg)^{\frac{\tT_2}{2}}\geq \frac{\exp(\frac{\eta\varepsilon_0\tT_2}{2mn})}{\sqrt{2\log(8mn/\delta)}\log(n)}\geq d,
\end{align*}
hence $\gamma_{+1,r}^{(t)}\geq d/2$, which violates to $\gamma_{+1,r}^{(t)}\leq 4\log(T^*)$ in the Proposition~\ref{prop:admissible_time_bound}.
\end{proof}

\subsection{Proof of Theorem~\ref{thm:withregularization}}
\label{sec:thmProof1}
With the analysis above, we already prove the convergence of training loss  in stage 2. 
Set $\delta=1/\poly(n)$, we have 
with probability at least $1-1/\poly(n)$, $\max_{r} \gamma_{j,r}^{(\tT_1)}\geq c_1/\polylog(n)$. Note that $\gamma_{j,r}^{(t)}$ increases when $t$ increases, and $|\overrho_{j,r,i}|, |\underrho_{j,r,i}|\leq 4\log(T^*)$. From Proposition~\ref{prop:Phase_II_prop_appendix}, there exists $\tT_1\leq t\leq \tT_1+\tilde\Omega\big( \frac{nm\sigma_p^2d}{\eta\varepsilon_0\|\bmu\|}\big)$, such that
\begin{align*}
L_S(\Wb^{(t)})\leq \varepsilon/(72\sigma_p^2d).
\end{align*}
Lemma~\ref{lemma:admissible_trainloss_appendix} gives the convergence of training gradient. Under this time $t$, for any testing data $\xb=(y\bmu,\bxi)$, if $j=y$, we have
\begin{align*}
    F_j(\Wb_j^{(t)},\xb)&=\frac{1}{m}\sum_{r=1}^m \bigg(\ReLU^2(\la\wb_{j,r}^{(0)},j\bmu\ra+\gamma_{j,r}^{(t)})+\ReLU^2\bigg(\la\wb_{j,r}^{(0)},\bxi_i\ra+\sum_{i'=1}^n\rho_{j,r,i'}^{(t)}\frac{\la\bxi,\bxi_{i'}\ra}{\|\bxi_{i'}\|^2}\bigg)\bigg)\\
    &\gg \frac{c_1}{m\polylog(n)},
\end{align*}
if $j'\not=y$, we have
\begin{align*}
    F_{j'}(\Wb_{j'}^{(t)},\xb)&=\frac{1}{m}\sum_{r=1}^m \bigg(\ReLU^2(\la\wb_{{j'},r}^{(0)},{j'}\bmu\ra-\gamma_{{j'},r}^{T^*})+\ReLU^2\bigg(\la\wb_{{j'},r}^{(0)},\bxi_i\ra+\sum_{i'=1}^n\rho_{{j'},r,i'}^{(t)}\frac{\la\bxi,\bxi_{i'}\ra}{\|\bxi_{i'}\|^2}\bigg)\bigg)\\
    &\leq \frac{1}{m}\sum_{r=1}^m \bigg(\ReLU^2(\la\wb_{{j'},r}^{(0)},{j'}\bmu\ra)+\ReLU^2\bigg(\la\wb_{{j'},r}^{(0)},\bxi_i\ra+\sum_{i'=1}^n\rho_{{j'},r,i'}^{(t)}\frac{\la\bxi,\bxi_{i'}\ra}{\|\bxi_{i'}\|^2}\bigg)\bigg)\\
    &\ll O\bigg(\frac{1}{m}\cdot\log^2(T^*)\sqrt{\log(4n^2\cdot \poly(n))/d} \bigg)\leq O\bigg(\frac{1}{m\cdot n^{\alpha/2}}\bigg),
\end{align*}
where the first inequaliy is by $\gamma_{j,r}^{(t)}\geq 0$, and the last inequality is by Condition~\ref{condition:condition}. Therefore we have with probability $1-1/\poly(n)$,
\begin{align*}
    F_j(\Wb_j^{(t)},\xb)\gg F_{j'}(\Wb_{j'}^{(t)},\xb),
\end{align*}
thus we conclude that 
\begin{align*}
P\big(yf(\Wb^{(t)},\xb)<0\big)\leq \frac{1}{\poly(n)},
\end{align*}
which completes the proof.

\section{Analysis without regularization}
\label{sec:withoutregular}
In this section, we present the analysis of the dynamics without regularization, which gives the direct comparison of the dynamics with regularization. Different from the structure in Section~\ref{sec:withregular}, we prove the convergence of the training loss during the whole training procedure first. 

Note that we have the decomposition
\begin{align*}
\wb_{j,r}^{(t)}=\wb_{j,r}^{(0)}+j\cdot\gamma_{j,r}^{(t)}\frac{\bmu}{\|\bmu\|^2}+\sum_{i=1}^n\rho_{j,r,i}^{(t)}\frac{\bxi_i}{\|\bxi_i\|^2},
\end{align*}
if we let $\lambda=0$, from Lemma~\ref{lemma:decomposition_coef_dynamic_appendix} we further have 
\begin{align*}
    \gamma_{j,r}^{(t+1)}&=\gamma_{j,r}^{(t)}-\frac{2\eta}{nm}\sum_{i=1}^{n}{\ell'^{(t)}_i} \ReLU ( \la\wb_{j,r}^{(t)}, y_i\cdot\bmu \ra)\|\bmu\|^2,\\
    \rho_{j,r,i}^{(t+1)}&=\rho_{j,r,i}^{(t)}-\frac{2\eta}{nm}{\ell'^{(t)}_i} \ReLU ( \la\wb_{j,r}^{(t)}, \bxi_i \ra)\|\bxi_i\|^2\cdot jy_i.
\end{align*}
 Here, $\gamma_{j,r}^{(0)}=\rho_{j,r,i}^{(0)}=0$. Proposition~\ref{prop:admissible_time_bound_appendix} also ensures that under Condition~\ref{condition:condition}, 
\begin{align}
    &0\leq\gamma_{j,r}^{(t)},\overrho_{j,r,i}^{(t)}\leq 4\log(T^*),\label{eq:admissible_withoutregu1}\\
    &0\geq\underrho_{j,r,i}^{(t)}\geq -\beta-64n\sqrt{\frac{\log(4n^2/\delta)}{d}}\log(T^*)\geq-4\log(T^*)\label{eq:admissible_withoutregu2}
\end{align}
for $0\leq t\leq T^*$ during the whole training process, where $\beta=2\max_{i,j,r}\{|\la w_{j,r}^{(0)},\bmu\ra|,|\la w_{j,r}^{(0)},\bxi_i\ra|\}$. We remind the readers that we define
\begin{align*}
\overrho_{j,r,i}^{(t)}=\rho_{j,r,i}^{(t)}\one(j=y_i),\quad \underrho_{j,r,i}^{(t)}=\rho_{j,r,i}^{(t)}\one(j\not=y_i).
\end{align*}

We briefly sketch out our proof.
From equation \eqref{eq:Update_eq}, it is evident that $\gamma_{j,r}^{(t)}$, $\overrho_{j,r,i}^{(t)}$ increase, and $\underrho_{j,r,i}^{(t)}$ decreases as $t$ increases. In this section, we adopt a three-stage analysis to disentangle the interaction between $\gamma_{j,r}^{(t)}$, $\overrho_{j,r,i}^{(t)}$, $\underrho_{j,r,i}^{(t)}$, and $\ell'^{(t)}_i$ without using gradient regularization to train the CNN.
In the first stage analysis, similar to the previous analysis, we set $\ell'^{(t)}_i=\ell'(y_if(\Wb^{(t)},\xb_i))\approx-1/2$ for all $i\in[n]$. We then show that $\gamma_{j,r}^{(t)}$, $\max_{r}\overrho_{j,r,i}^{(t)}$, and $\max_{r}|\underrho_{j,r,i}^{(t)}|$ have a large scale difference.
In the second stage, we further increase the gap between $\max_{r}\overrho_{j,r,i}^{(t)}$ and $\gamma_{j,r}^{(t)}$, $\max_{r}|\underrho_{j,r,i}^{(t)}|$, which ensures the analysis in the third stage during the training process.
Finally, in the third stage, we prove Theorem~\ref{thm:withoutregularization} by leveraging the scale differences established in the previous two stages. 

\textbf{Stage 1} It is still true that until some $\gamma_{j,r}^{(t)}$, $\rho_{j,r,i}^{(t)}$ reach $\Theta(1/\polylog(n))$, $\ell'^{(t)}_i$ will be around $1/2$. The following proposition summarizes our main conclusion in the first stage training without gradient regularization. 
\begin{proposition}
\label{prop:without_regu_bound-1}
Let $\tT_1=\frac{nm}{\eta \sigma_p^2d}\log\Big( \frac{0.8}{2\log(n)\sqrt{\log(8mn/\delta)}\sigma_0\sigma_p\sqrt{d}}\Big)$, for any $j\in\{\pm1\}$, $r\in\{m\}$, $i\in[n]$ and  $t\in[\tT_1]$, we have
\begin{align*}
    &\gamma_{j,r}^{(\tT_1)}=\tilde{O}(\sigma_0\|\bmu\|),\quad \frac{1}{21\log(n)\sqrt{\log(8mn/\delta)}}\leq\max_{j=y_i,r}\rho_{j,r,i}^{(\tT_1)}\leq 2\log^{-1}(n),\\
    &\max_{j\not=y_i,r}|\rho_{j,r,i}^{(\tT_1)}|\leq 65n\sqrt{\frac{\log(4n^2/\delta)}{d}}\log(T^*).
\end{align*}
\end{proposition}
Proposition~\ref{prop:without_regu_bound-1} shows that CNN will memorize noise under the training without gradient regularization. At the end of this stage, $\max_{r}\overrho_{j,r,i}^{(t)}$ reaches $\Theta(1/\polylog(n))$ which is sufficiently larger than $\gamma_{j,r}^{(t)}$ and $|\underrho_{j,r,i}^{(t)}|$. After this, we show that the difference becomes larger in the second stage.

\textbf{Stage 2} 
In this stage, we show that the gap between  $\max_{r}\overrho_{j,r,i}^{(t)}$ and $\gamma_{j,r}^{(t)}$, $\max_{r}|\underrho_{j,r,i}^{(t)}|$ becomes larger. The larger gap ensures the analysis in the third stage. 
\begin{proposition}
    \label{prop:without_stage2_induct}
    Let $\tT_2=\frac{Cnm}{\eta\sigma_p^2d }\log\Big(\frac{1}{\log(n)\sqrt{\log(8mn/\delta)}}\Big)$ for some sufficient large constant $C>0$, then for any $i\in[n]$ and $r\in[m]$ it holds that 
    \begin{align*}
        &\gamma_{j,r}^{(\tT_1+\tT_2)}=\tilde{O}(\sigma_0\|\bmu\|),\quad \max_{j=y_i,r}\rho_{j,r,i}^{(\tT_1+\tT_2)}\geq 2,\\
        & \max_{j\not=y_i,r}\big|\rho_{j,r,i}^{(\tT_1+\tT_2)}\big| \leq 65n\sqrt{\frac{\log(4n^2/\delta)}{d}}\log(T^*).
    \end{align*}
\end{proposition}
Proposition~\ref{prop:without_stage2_induct} plays a crucial role in our analysis by significantly increasing the scale of $\max_r\overrho_{j,r,i}^{(t)}$ from a small to a constant order. Despite this enlargement, the quantities $\gamma_{j,r}^{(t)}$ and $\underrho_{j,r,i}^{(t)}$ still remain at a relatively small order. Building on the analysis in stages 1 and 2, we are now able to proceed with the third stage of our analysis, allowing us to complete the proof of Theorem~\ref{thm:withoutregularization}.

\textbf{Stage 3}
In this stage, we consider the change in $\ell'^{(t)}_i$. The main technique here is constructing a convex point $\Wb^*$, where
\begin{align*}
\wb_{j,r}^*=\wb_{j,r}^{(0)}+4m\log(4/\varepsilon)\cdot \bigg(\sum_{i=1}^n \one(j=y_i)\frac{\bxi_i}{\| \bxi_i\|} \bigg).
\end{align*}
We introduce the following lemma, which we explain later how the following lemma helps the proof of the convergence of training gradient.
\begin{lemma}
    \label{lemma:stage3_WT1_diff}
    Under the same condition as Theorem~\ref{thm:withoutregularization}, it holds that $\|\Wb^{(\tT_1+\tT_2)}-\Wb^{*}\|_{F}\leq \tilde{O}(m^{\frac{3}{2}}n^{\frac{1}{2}}\sigma_p^{-1}d^{-\frac{1}{2}})$.
\end{lemma}
Lemma~\ref{lemma:stage3_WT1_diff} gives an upper bound of $\|\Wb^{(\tT_1+\tT_2)}-\Wb^{*}\|_{F}$. We introduce the following lemma, which give an upper bound of $L_S(\Wb^{(t)})$.
\begin{lemma}
\label{lemma:stage3_diff_Wt}
Under the same condition as Theorem~\ref{thm:withoutregularization}, it holds that 
\begin{align*}
    \|\Wb^{(t)}-\Wb^*\|_F^2-\|\Wb^{(t+1)}-\Wb^*\|_F^2\geq 3\eta L_S(\Wb^{(t)})-\eta\varepsilon
\end{align*}
for all $\tT_1+\tT_2\leq t\leq T^*$.
\end{lemma}
Combined with Lemma~\ref{lemma:stage3_WT1_diff} and \ref{lemma:stage3_diff_Wt}, we can give a summation of the inequality in the Lemma~\ref{lemma:stage3_diff_Wt} and then give an upper bound of $L_S(\Wb^{(t)})$. The bound of the training gradient is then given by Lemma~\ref{lemma:admissible_trainloss_appendix}. We give the following proposition, which shows that in the final stage, $\gamma_{j,r}^{(t)}$ keeps small and $\max_{r}\overrho_{j,r,i}^{(t)}$ grows.
\begin{proposition}
\label{prop:without_Gradient_bound}
 Under the same condition as Theorem~\ref{thm:withoutregularization}, for any $\varepsilon>0$, let $T=\tT_1+\tT_2+\Big\lfloor\frac{\|\Wb^{(\tT_1+\tT_2)}-\Wb^*\|_F^2 }{2\eta\varepsilon}\Big\rfloor=\tT_1+\tT_2+\tilde{O}(\eta^{-1}\varepsilon^{-1}m^3n\sigma_p^{-2}d^{-1})$,  there exists $t$ in $\tT_1+\tT_2\leq t\leq T$ such that
\begin{align*}
L_S(\Wb^{(t)})\leq \varepsilon.
\end{align*}
Meanwhile, it holds that $\max_{j,r}\gamma_{j,r}^{(t)}=\tilde{O}(\sigma_0\|\bmu\|)$, $\max_{j,r,i}|\underrho_{j,r,i}^{(t)}|=\tilde{O}(nd^{-1/2})$ for all $\tT_1+\tT_2\leq t\leq T$.
\end{proposition}
With Proposition~\ref{prop:admissible_time_bound_appendix}, Proposition~\ref{prop:without_stage2_induct} and Proposition~\ref{prop:without_Gradient_bound}, we can prove Theorem~\ref{thm:withoutregularization}. Details can be found in Appendix~\ref{sec:thmProof2}.

\subsection{Stage 1: Signal learning and noise memorizing}
We give the analysis at stage 1. In this training period, we can find that the noise learning is much faster than the signal learning. The next proposition indicates that the inner product of $\wb$ and $\bxi_i$ shares similar rates when $j=y_i$, but has great difference between $j=y_i$ and $j\not=y_i$. 
\begin{proposition}
    \label{prop:Withoutregu_induct}
    Define $B_i^{(t)}=\max_{j=y_i,r}|\la \wb_{j,r}^{(t)},\bxi_i\ra|$ and $C_i^{(t)}=\max_{j\not=y_i,r}|\la \wb_{j,r}^{(t)},\bxi_i\ra|$, and let 
    \begin{align*}
        \tT_1=\frac{nm}{\eta \sigma_p^2d}\log\bigg( \frac{0.8}{2\log(n)\sqrt{\log(8mn/\delta)}\sigma_0\sigma_p\sqrt{d}}\bigg),
    \end{align*}
then for any $i,i_1\in[n]$ and  $t\in[\tT_1]$, we have 
\begin{align}
&\max_{j,r}|\la\wb_{j,r}^{(t)},\bmu\ra|=\tilde{O}(\sigma_0\|\bmu\|),\label{eq:without_regu_directmu} \\
&B_i^{(t)}/B_{i_1}^{(t)}\leq 16\sqrt{\log(8mn/\delta)},\quad  C_{i_1}^{(t)}\leq B_i^{(t)}\cdot 17n\sqrt{\log(8mn/\delta)}\cdot \sqrt{\log(4n^2/\delta)/d},\label{eq:without_induction1}\\
&  B_i^{(t)}\leq \log^{-1}(n)\label{eq:without_induction2}.
\end{align}
\end{proposition}
Before we prove Proposition~\ref{prop:Withoutregu_induct}, we present a lemma which depicts that  $|\ell'^{(t)}_i|$ will be around $1/2$ under \eqref{eq:without_regu_directmu}-\eqref{eq:without_induction2}.
\begin{lemma}
    \label{lemma:ell_i_around1/2}
Under conditions \eqref{eq:without_regu_directmu},\eqref{eq:without_induction1} and \eqref{eq:without_induction2}, for any $i\in[n]$ we have
\begin{align*}
    \bigg|\ell'^{(t)}_i+\frac{1}{2}\bigg|=O\bigg(\frac{1}{\log^2(n)}\bigg).
\end{align*}
\end{lemma}
\begin{proof}[Proof of Lemma~\ref{lemma:ell_i_around1/2}]
Recall that 
\begin{align*}
    \bigg|\ell'^{(t)}_i+\frac{1}{2}\bigg|=\bigg|\frac{1-\exp\{y_i[F_{+1}(\Wb_{+1}^{(t)},\xb_i)-F_{-1}(\Wb_{-1}^{(t)},\xb_i)]\}}{2+2\exp\{y_i[F_{+1}(\Wb_{+1}^{(t)},\xb_i)-F_{-1}(\Wb_{-1}^{(t)},\xb_i)]\}}\bigg|.
\end{align*}
To prove the conclusion, easy to see it is equal to prove 
$\exp\{y_i[F_{+1}(W_{+1}^{(t)},\xb_i)-F_{-1}(W_{-1}^{(t)},\xb_i)]\}=1\pm O\Big(\frac{1}{\log^2{(n)}}\Big)$. We now investigate  $F_{j}(W_{j}^{(t)},\xb_i)$. Note that 
\begin{align*}
    F_{j}(W_{j}^{(t)},\xb_i)&=\frac{1}{m}\sum_{r=1}^m\sigma(\la\wb_{j,r}^{(t)}, y_i\bmu \ra )+\sigma(\la\wb_{j,r}^{(t)}, \bxi_i \ra )\\
    &\leq\max_{j,r} |\la\wb_{j,r}^{(t)},\bmu\ra|+ \max\{B_i^{(t)2},C_i^{(t)2}\}\\
    &\leq \tilde{O}(\sigma_0\|\bmu\|)+\log^{-2}(n)=O\bigg(\frac{1}{\log^2(n)}\bigg),
    \end{align*}
where the second inequality comes from $\sigma_0$ small enough, \eqref{eq:without_induction1} and \eqref{eq:without_induction2}, we have 
\begin{align*}
   \big|y_i[F_{+1}(W_{+1}^{(t)},\xb_i)-F_{-1}(W_{-1}^{(t)},\xb_i)] \big|=O\bigg(\frac{1}{\log^2{(n)}}\bigg),
\end{align*}
which directly shows $\exp\{y_i[F_{+1}(W_{+1}^{(t)},\xb_i)-F_{-1}(W_{-1}^{(t)},\xb_i)]\}=1\pm O\big(\frac{1}{\log^2{(n)}}\big)$.
\end{proof}
We now present the proof of Proposition~\ref{prop:Withoutregu_induct}.
\begin{proof}[Proof of Proposition~\ref{prop:Withoutregu_induct}]
Before the present of the proof, we remind the readers that the number $0.8$ set in $\tT_1$ is to make the following equation hold:
\begin{align}
    \big[1+O(\log^{-2}(n))\big]\cdot\log\bigg( \frac{0.8}{2\log(n)\sqrt{\log(8mn/\delta)}\sigma_0\sigma_p\sqrt{d}}\bigg) \leq \log\bigg( \frac{1}{2\log(n)\sqrt{\log(8mn/\delta)}\sigma_0\sigma_p\sqrt{d}}\bigg). \label{eq:slight_adjust}    
\end{align}
    We first recall the update rule in Lemma~\ref{lemma:update_rule_appendix} with $\lambda=0$, then we have that 
    \begin{align*}
        \wb_{j,r}^{(t+1)}=\wb_{j,r}^{(t)}-\frac{2\eta}{nm}\sum_{i=1}^{n}{\ell'^{(t)}_i} jy_i \big(\ReLU ( \la\wb_{j,r}^{(t)}, y_i\bmu \ra)y_i\bmu+\ReLU (\la \wb_{j,r}^{(t)}, \bxi_i \ra) \bxi_i\big).
    \end{align*}
    Note that $\la\bmu,\bxi_i\ra=0$, we further have
    \begin{align}
    &\la\wb_{j,r}^{(t+1)},y_i\bmu\ra=\la\wb_{j,r}^{(t)},y_i\bmu\ra-\frac{2\eta}{nm}\sum_{i'=1}^{n}{\ell'^{(t)}_{i'}} jy_{i'} \cdot \ReLU ( \la\wb_{j,r}^{(t)}, y_i\bmu \ra)\cdot\|\bmu\|^2,\label{eq:inner_mu_update}\\
    &\la\wb_{j,r}^{(t+1)},\bxi_i\ra=\la\wb_{j,r}^{(t)},\bxi_i\ra-\frac{2\eta}{nm}\sum_{i'=1}^{n}{\ell'^{(t)}_{i'}} jy_{i'}\cdot\ReLU (\la \wb_{j,r}^{(t)}, \bxi_{i'} \ra) \cdot \la \bxi_i,\bxi_{i'}\ra.\label{eq:inner_noise_update}
    \end{align}
We first prove \eqref{eq:without_regu_directmu} holds. From \eqref{eq:inner_mu_update}, we can see that 
\begin{align*}
    |\la\wb_{j,r}^{(t+1)},y_i\bmu\ra|&\leq |\la\wb_{j,r}^{(t)},y_i\bmu\ra|+\bigg|\frac{2\eta}{nm}\sum_{i'=1}^{n} \ReLU ( \la\wb_{j,r}^{(t)}, y_i\bmu \ra)\cdot\|\bmu\|^2\bigg|\\
    &\leq |\la\wb_{j,r}^{(t)},y_i\bmu\ra|\cdot \bigg(1+\frac{2\eta\|\bmu\|^2}{m}\bigg),
    \end{align*}
where the first inequality holds from $\big|\ell'^{(t)}_i\big|\leq1$, and the second inequality is by $\ReLU ( \la\wb_{j,r}^{(t)}, y_i\bmu \ra)\leq | \la\wb_{j,r}^{(t)}, y_i\bmu \ra|$. Hence we have
\begin{align*}
\max_{j,r}|\la\wb_{j,r}^{(t)},\bmu\ra|&\leq \max_{j,r}|\la\wb_{j,r}^{(0)},\bmu\ra|\cdot \bigg( 1+\frac{2\eta\|\bmu\|^2}{m}\bigg)^t\\
&\leq \max_{j,r}|\la\wb_{j,r}^{(0)},\bmu\ra|\cdot \bigg( 1+\frac{2\eta\|\bmu\|^2}{m}\bigg)^{\tT_1}\\
&\leq \max_{j,r}|\la\wb_{j,r}^{(0)},\bmu\ra|\cdot \exp\bigg\{\frac{2\eta\|\bmu\|^2}{m}\cdot\tT_1\bigg\}\leq 2\max_{j,r}|\la\wb_{j,r}^{(0)},\bmu\ra|=\tilde{O}(\sigma_0\|\bmu\|).
\end{align*}
Here, the second inequality is by $t\leq \tT_1$, the third inequality is by $1+z\leq \exp(z)$ for $z>0$ and the last inequality is by $\frac{2\eta\|\bmu\|^2}{m}\cdot\tT_1=o(1)$.

We apply mathematical induction to prove \eqref{eq:without_induction1} and \eqref{eq:without_induction2}.
From Lemma~\ref{lemma:data_CNN_concentration}, we have $\sigma_0\sigma_p\sqrt{d}/4\leq\max_{j=y_i,r}|\la \wb_{j,r}^{(t)},\bxi_i\ra|,\max_{j\not=y_i,r}|\la \wb_{j,r}^{(t)},\bxi_i\ra|\leq2\sigma_0\sigma_p\sqrt{d}\sqrt{\log(8mn/\delta)}  $. Thus we can easily verify that $ B_i^{(0)}/B_{i_1}^{(0)}\leq 16\sqrt{\log(8mn/\delta)}$, $ C_{i_1}^{(0)}\leq B_i^{(0)}\cdot 16\sqrt{\log(8mn/\delta)}$. Assume that for any $i,i_1\in[n]$ and  $0\leq t\leq T_0-1<\tT_1$, 
\eqref{eq:without_induction1} and \eqref{eq:without_induction2} hold, we first prove that 
$
    B_i^{(T_0)}/B_{i_1}^{(T_0)}\leq 16\sqrt{\log(8mn/\delta)}
$ holds. For $t\in[T_0-1]$ we first have
\begin{align*}
\big|\la\wb_{j,r}^{(t+1)},\bxi_i\ra\big|&=\bigg|\la\wb_{j,r}^{(t)},\bxi_i\ra-\frac{2\eta}{nm}\sum_{i'=1}^{n}{\ell'^{(t)}_{i'}} jy_{i'}\cdot\ReLU (\la \wb_{j,r}^{(t)}, \bxi_{i'} \ra) \cdot \la \bxi_i,\bxi_{i'}\ra \bigg| \\
    &\leq\big|\la\wb_{j,r}^{(t)},\bxi_i\ra\big|+\bigg|\frac{2\eta}{nm}\sum_{i'=1}^{n}{\ell'^{(t)}_{i'}} jy_{i'}\cdot\ReLU (\la \wb_{j,r}^{(t)}, \bxi_{i'} \ra) \cdot \la \bxi_i,\bxi_{i'}\ra \bigg|,
\end{align*}
where the first equality comes from \eqref{eq:inner_noise_update}. If $j=y_i$, we have 
\begin{align*}
    &\frac{\max_{j=y_i,r}\big|\la\wb_{j,r}^{(t+1)},\bxi_i\ra\big|}{\max_{j=y_{i_1},r}\big|\la\wb_{j,r}^{(t+1)},\bxi_{i_1}\ra\big|}=B_{i}^{(t+1)}/B_{i_1}^{(t+1)}\\&
    \qquad \qquad\leq \frac{\max_{j=y_i,r}\big|\la\wb_{j,r}^{(t)},\bxi_i\ra\big|+\max_{j=y_i, r}\Big|\frac{\eta(1+O(\log^{-2}(n)))}{nm}\sum_{i'=1}^{n} jy_{i'}\cdot\ReLU (\la \wb_{j,r}^{(t)}, \bxi_{i'} \ra) \cdot \la \bxi_i,\bxi_{i'}\ra \Big|}{\max_{j=y_{i_1},r}\Big|\la\wb_{j,r}^{(t)},\bxi_{i_1}\ra+\frac{\eta(1-O(\log^{-2}(n)))}{nm}\sum_{i'=1}^{n} jy_{i'}\cdot\ReLU (\la \wb_{j,r}^{(t)}, \bxi_{i'} \ra) \cdot \la \bxi_{i_1},\bxi_{i'}\ra \Big|}\\
    &\qquad \qquad\leq \frac{B_i^{(t)}+\frac{\eta(1+O(\log^{-2}(n)))}{nm}\max_{j=y_i, r}\sigma_p^2d\bigg|B_i^{(t)}+\bigg(\sum_{\substack{y_{i'}=j\\ i'\not=i}}  B_{i'}^{(t)}-\sum_{\substack{y_{i'}\not=j}}  C_{i'}^{(t)}\bigg)\cdot\sqrt{\log(4n^2/\delta)/d} \bigg|}{B_{i_1}^{(t)}\cdot(1+\eta\sigma_p^2d/nm(1-O(\log^{-2}(n))))}\\
    &\qquad \qquad\leq \frac{B_i^{(t)}\cdot \Big(1+\frac{\eta \sigma_p^2d}{nm}\cdot(1+O(\log^{-2}(n))) \Big)}{B_{i_1}^{(t)}\cdot \Big(1+\frac{\eta \sigma_p^2d}{nm}\cdot(1-O(\log^{-2}(n))) \Big)}.
\end{align*}
Here, the first equality comes from the definition of $B_i^{(t)}$, the first inequality is by triangle inequality and $\ell'^{(t)}_i=-0.5\pm O(\log^{-2}(n))$ in Lemma~\ref{lemma:ell_i_around1/2}, the second inequality is by the definition that $B_i^{(t)}$ and $C_i^{(t)}$ are the maximum value of $\max_{j=y_i,r}|\la \wb_{j,r}^{(t)},\bxi_i\ra|$ and $\max_{j\not=y_i,r}|\la \wb_{j,r}^{(t)},\bxi_i\ra|$, respectively. The last inequality is by the  induction assumption and $d\gg n^{(2+2\alpha)}$ in Condition~\ref{condition:condition}. Therefore we take $t=T_0-1$ and then conclude that 
\begin{align*}
    \frac{B_{i}^{(T_0)}}{B_{i_1}^{(T_0)}}&\leq \frac{B_{i}^{(0)}}{B_{i_1}^{(0)}}\cdot \bigg(1+\frac{\eta \sigma_p^2d}{nm}\cdot O(\log^{-2}(n))\bigg)^{\tT_1}\leq 8\sqrt{\log(8mn/\delta)}\cdot \exp\bigg\{\tT_1\cdot\frac{\eta \sigma_p^2d}{nm}\cdot O(\log^{-2}(n)) \bigg\}\\
    &\leq 8\sqrt{\log(8mn/\delta)}\cdot\exp \big\{\log\big( 0.8/{2\log(n)\sqrt{\log(8mn/\delta)}\sigma_0\sigma_p\sqrt{d}}\big)\cdot O(\log^{-2}(n)) \big\}\\
    &=8\sqrt{\log(8mn/\delta)}\cdot\exp \big\{O(\log^{-0.5}(n)) \big\}\leq 16\sqrt{\log(8mn/\delta)}.
\end{align*}
Here, the first inequality is by $T_0\leq \tT_1$, the second inequality is by $1+z\leq \exp(z)$ for $z>0$ and $B_{i}^{(0)}/B_{i_1}^{(0)}\leq 8\sqrt{\log(8mn/\delta)}$, and the last inequality is by $\exp \big\{-O(\log^{-0.5}(n)) \big\}\leq 2$. We complete the proof for $B_i^{(t)}/B_{i_1}^{(t)}\leq 16\sqrt{\log(8mn/\delta)}$. 

For $C_{i_1}^{(t)}/B_{i}^{(t)}$, similarly we have
\begin{align*}
    &C_{i_1}^{(t+1)}/B_{i}^{(t+1)}\\
    &\qquad\leq \frac{\max_{j\not=y_{i_1},r}\Big|\la\wb_{j,r}^{(t)},\bxi_{i_1}\ra+\frac{\eta(1+O(\log^{-2}(n)))}{nm}\sum_{i'=1}^{n} jy_{i'}\cdot\ReLU (\la \wb_{j,r}^{(t)}, \bxi_{i'} \ra) \cdot \la \bxi_{i_1},\bxi_{i'}\ra \Big|}{\max_{j=y_{i},r}\Big|\la\wb_{j,r}^{(t)},\bxi_{i_1}\ra+\frac{\eta(1-O(\log^{-2}(n)))}{nm}\sum_{i'=1}^{n} jy_{i'}\cdot\ReLU (\la \wb_{j,r}^{(t)}, \bxi_{i'} \ra) \cdot \la \bxi_{i},\bxi_{i'}\ra \Big|}\\
    &\qquad\leq  \frac{\bigg|C_{i_1}^{(t)}+\frac{\eta\sigma_p^2d (1+O(\log^{-2}(n)))}{nm}\max_{j=y_{i_1}, r}\bigg\{ -C_{i_1}^{(t)}+\bigg(\sum_{\substack{y_{i'}=j}}  B_{i'}^{(t)}-\sum_{\substack{y_{i'}\not=j\\ i'\not=i_1}}  C_{i'}^{(t)}\bigg)\cdot2\sqrt{\log(4n^2/\delta)/d} \bigg\}\bigg|}{B_{i}^{(t)}\cdot(1+\eta\sigma_p^2d/nm(1-O(\log^{-2}(n))))}\\
    &\qquad\leq \frac{C_{i_1}^{(t)}\cdot \Big(1-\frac{\eta \sigma_p^2d}{nm}\cdot(1-O(\log^{-2}(n))) \Big)}{B_{i}^{(t)}\cdot \Big(1+\frac{\eta \sigma_p^2d}{nm}\cdot(1-O(\log^{-2}(n))) \Big)}+\frac{\eta\sigma_p^2d(1+\log^{-2}(n))}{m}32\sqrt{\log(8mn/\delta)}\cdot \sqrt{\log(4n^2/\delta)/d},
\end{align*}
where the second inequality is by the definition of $B_i^{(t)}$ and $C_i^{(t)}$ and the third inequality is by the mathematical induction. The equation above can be considered as 
\begin{align*}
    a_{t+1}\leq ca_t+b,
\end{align*}
where $0<c<1$, $0<b<1$ and $0<a_0<1$. We have $a_T+b/(c-1)\leq (a_0+b/(c-1))\cdot c^T$.
Taking $t=T_0-1$ and we conclude that 
\begin{align*}
    \frac{C_{i_1}^{(T_0)}}{B_{i}^{(T_0)}}&\leq \frac{C_{i_1}^{(0)}}{B_{i}^{(0)}}\cdot \bigg(1-\frac{2\eta \sigma_p^2d}{nm}\cdot (1+o(1))\bigg)^{\tT_1}+\frac{\eta\sigma_p^2d(1+o(1))}{m\cdot \frac{2\eta\sigma_p^2d}{nm}}32\sqrt{\log(8mn/\delta)}\cdot \sqrt{\log(4n^2/\delta)/d}\\
    &\leq 8\sqrt{\log(8mn/\delta)}\cdot \exp\bigg(-\frac{2\eta\sigma_p^2d}{nm}\cdot \tT_1(1+o(1))\bigg)+16.5n\sqrt{\log(8mn/\delta)}\cdot \sqrt{\log(4n^2/\delta)/d}\\
    &\leq \tilde{O}(\sigma_0\sigma_p\sqrt{d})+16.5n\sqrt{\log(8mn/\delta)}\cdot \sqrt{\log(4n^2/\delta)/d}\leq 17n\sqrt{\log(8mn/\delta)}\cdot \sqrt{\log(4n^2/\delta)/d}
\end{align*}
Here, The first inequality is by the conclusion above, the second inequality is by $1-z\leq \exp\{-z\}$ for $z>0$ and the third inequality is by the definition of $\tT_1$. The proof of \eqref{eq:without_induction1} is thus completed. We next prove that $B_i^{(t)}\leq \log^{-1}(n)$. From \eqref{eq:inner_noise_update}, we conclude that
\begin{align*}
    B_i^{(t+1)}&\leq B_i^{(t)}+\frac{\eta\sigma_p^2d(1+O(\log^{-2}(n)))}{nm} B_i^{(t)}\\
    &\quad+ \frac{\eta\sigma_p^2d(1+O(\log^{-2}(n)))}{nm} 16\sqrt{\log(8mn/\delta)} n\cdot B_i^{(t)}\cdot 2\sqrt{\log(4n^2/\delta)/d}\\
    &= B_i^{(t)}\cdot\bigg(1+\frac{\eta\sigma_p^2d}{nm}(1+O(\log^{-2}(n)))\bigg).
    \end{align*}
Here, the first inequality is by \eqref{eq:without_induction1} which is proved above, and the second inequality is by Condition~\ref{condition:condition} which gives $\sqrt{\log(8mn/\delta)} \cdot n\cdot\sqrt{\log(4n^2/\delta)/d}=o(1)$.  Set $t=T_0-1$, then we have 
\begin{align*}
    B_i^{(T_0)}&\leq B_i^{(0)}\cdot\bigg(1+\frac{\eta\sigma_p^2d}{nm}(1+O(\log^{-2}(n)))\bigg)^{T_0}\\
    &\leq B_i^{(0)}\cdot \exp\bigg\{\frac{\eta\sigma_p^2d}{nm}(1+O(\log^{-2}(n)))\cdot T_0\bigg\}\\
    &\leq 2\sqrt{\log(8mn/\delta)}\sigma_0\sigma_p\sqrt{d}\cdot \exp\{-\log(2\log(n)\sqrt{\log(8mn/\delta)}\sigma_0\sigma_p\sqrt{d})\}\\
    &=\log^{-1}(n). 
\end{align*}
Here, the second inequality comes from $1+z\leq \exp(z)$ for $z>0$, and the third inequality comes from Lemma~\ref{lemma:data_CNN_concentration}, $T_0\leq T_1$ and \eqref{eq:slight_adjust}.
\end{proof}

With the proposition above, we can prove the following proposition which gives the lower bound of $\max_{j,r}|\la \wb_{j,r}^{(\tT_1)},\bxi_i\ra|$ in the first stage.
\begin{proposition}
    \label{prop:without_regu_lower_bound_rho_appendix}
Let $\tT_1$ be defined in Proposition~\ref{prop:Withoutregu_induct}, for all $i\in[n]$ it holds that
\begin{align*}
\max_{j=y_i,r}\la\wb_{j,r}^{(\tT_1)},\bxi_i\ra\geq\frac{1}{20\log(n)\sqrt{\log(8mn/\delta)}}.
\end{align*}
\end{proposition}
\begin{proof}[Proof of Proposition~\ref{prop:without_regu_lower_bound_rho_appendix}]
Before the proof of Proposition~\ref{prop:without_regu_lower_bound_rho_appendix}, we remind that the value $0.8$ set in $\tT_1$ also makes the following equation hold:
\begin{align}
    \big[1-O(\log^{-2}(n))\big]\cdot\log\bigg( \frac{0.8}{2\log(n)\sqrt{\log(8mn/\delta)}\sigma_0\sigma_p\sqrt{d}}\bigg) \geq \log\bigg( \frac{0.6}{2\log(n)\sqrt{\log(8mn/\delta)}\sigma_0\sigma_p\sqrt{d}}\bigg). \label{eq:slight_adjust1}    
\end{align}
Recall \eqref{eq:inner_noise_update} we have
    \begin{align*}
        \la\wb_{j,r}^{(t+1)},\bxi_i\ra=\la\wb_{j,r}^{(t)},\bxi_i\ra-\frac{2\eta}{nm}\sum_{i'=1}^{n}{\ell'^{(t)}_{i'}} jy_{i'}\cdot\ReLU (\la \wb_{j,r}^{(t)}, \bxi_{i'} \ra) \cdot \la \bxi_i,\bxi_{i'}\ra.
    \end{align*}
We first use mathematical induction to prove that $\max_{j=y_i,r}\la\wb_{j,r}^{(t)},\bxi_i\ra>0$. It is easy to verify that when $t=0$, $\max_{j=y_i,r}\la\wb_{j,r}^{(0)},\bxi_i\ra>0$. Suppose that there exists $T_0$ such that $0\leq t\leq T_0-1<\tT_1$, $\max_{j=y_i,r}\la\wb_{j,r}^{(0)},\bxi_i\ra>0$, 
denote by $(j_i^t,r_i^t)$ which satisfies $\la\wb_{j_i^t,r_i^t,r}^{(t)},\bxi_i\ra=\max_{j=y_i,r}\la\wb_{j,r}^{(t)},\bxi_i\ra$, we have
\begin{align*}
\la\wb_{j_i^t,r_i^t}^{(t+1)},\bxi_i\ra&=\la\wb_{j_i^t,r_i^t}^{(t)},\bxi_i\ra-\frac{2\eta}{nm}\sum_{i'=1}^{n}{\ell'^{(t)}_{i'}} jy_{i'}\cdot\ReLU (\la \wb_{j_i^t,r_i^t}^{(t)}, \bxi_{i'} \ra) \cdot \la \bxi_i,\bxi_{i'}\ra\\
&\geq \la\wb_{j_i^t,r_i^t}^{(t)},\bxi_i\ra+\frac{\eta\sigma_p^2d}{nm}(1-O(\log^{-2}(n)))\la\wb_{j_i^t,r_i^t}^{(t)},\bxi_i\ra\\
& \qquad -\frac{\eta\sigma_p^2d}{nm}(1+O(\log^{-2}(n))) \sum_{i'\not=i} 16\sqrt{\log(8mn/\delta)}\cdot \la\wb_{j_i^t,r_i^t}^{(t)},\bxi_i\ra\cdot 2\sqrt{\frac{\log(4n^2/\delta)}{d}}\\
&\geq \la\wb_{j_i^t,r_i^t}^{(t)},\bxi_i\ra\cdot\bigg(1+\frac{\eta\sigma_p^2d}{nm}\cdot(1-O(\log^{-2}(n)))  \bigg).
\end{align*}
Here, the second inequality is by Proposition~\ref{prop:Withoutregu_induct} and Lemma~\ref{lemma:data_noise_concentration}, and the second inequality is by Condition~\ref{condition:condition} where $n^{1+\alpha}\cdot \sqrt{\log(4n^2/\delta)/d}=o(1)$.
Hence, set $t=T_0-1$ we conclude $\max_{j=y_i,r}\la\wb_{j,r}^{(T_0)},\bxi_i\ra>0$. Moreover, from the induction inequality, we obtain that 
\begin{align*}
    \max_{j=y_i,r}\la\wb_{j,r}^{(\tT_1)},\bxi_i\ra&\geq \max_{j=y_i,r}\la\wb_{j,r}^{(0)},\bxi_i\ra\cdot \bigg(1+\frac{\eta\sigma_p^2d}{nm}\cdot(1-O(\log^{-2}(n)))  \bigg)^{\tT_1}\\
    &\geq \sigma_0\sigma_p\sqrt{d}/4\cdot \bigg[\exp\bigg\{\frac{\eta\sigma_p^2d}{nm}\cdot(1-O(\log^{-2}(n))) \bigg\} -O\bigg(\frac{\eta\sigma_p^2d}{nm}\bigg)^2\bigg]^{\tT_1}\\
    &= \sigma_0\sigma_p\sqrt{d}/4\cdot \exp\bigg\{\frac{\eta\sigma_p^2d}{nm}\cdot(1-O(\log^{-2}(n)))\cdot \tT_1 \bigg\} \cdot\bigg(1-O\bigg(\frac{\eta\sigma_p^2d}{nm}\bigg)^2\bigg)^{\tT_1} \\
    &\geq \sigma_0\sigma_p\sqrt{d}/4\cdot \exp\{\log(0.6/2\log(n)\sqrt{\log(8mn/\delta)}\sigma_0\sigma_p\sqrt{d}) \}\cdot 2/3 \\
    &= \frac{1}{20\log(n)\sqrt{\log(8mn/\delta)}}.
\end{align*}
Here, the second inequality is by $1+z\geq \exp(z)-z^2 $ for $0<z<0.1$, the third inequality is by \eqref{eq:slight_adjust1} and $\big(1-O\big(\frac{\eta\sigma_p^2d}{nm}\big)^2\big)^{\tT_1}\gg2/3$. Hence, for all $i\in[n]$ we have
\begin{align*}
\max_{j=y_i,r}\la\wb_{j,r}^{(\tT_1)},\bxi_i\ra\geq\frac{1}{20\log(n)\sqrt{\log(8mn/\delta)}},
\end{align*}
which completes the proof.
\end{proof}

From Proposition~\ref{prop:Withoutregu_induct} and \ref{prop:without_regu_lower_bound_rho_appendix}, we return to the decomposition of $\wb_{j,r}^{(t)}$ and conclude the following corollary.
\begin{corollary}[Restatement of Proposition~\ref{prop:without_regu_bound-1}]
\label{coro: without_decompose_bound}
Let $\tT_1$ be defined in Proposition~\ref{prop:Withoutregu_induct}, for any $j\in\{\pm1\}$, $r\in\{m\}$, $i\in[n]$ and  $t\in[\tT_1]$, we have
\begin{align*}
    &\gamma_{j,r}^{(\tT_1)}=\tilde{O}(\sigma_0\|\bmu\|),\quad \frac{1}{21\log(n)\sqrt{\log(8mn/\delta)}}\leq\max_{j=y_i,r}\rho_{j,r,i}^{(\tT_1)}\leq 2\log^{-1}(n),\\
    &\max_{j\not=y_i,r}|\rho_{j,r,i}^{(\tT_1)}|\leq 65n\sqrt{\frac{\log(4n^2/\delta)}{d}}\log(T^*).
\end{align*}
\end{corollary}
\begin{proof}[Proof of Corollary~\ref{coro: without_decompose_bound}]
    Recall that
    \begin{align*}
        \wb_{j,r}^{(t)}=\wb_{j,r}^{(0)}+j\cdot\gamma_{j,r}^{(t)}\frac{\bmu}{\|\bmu\|^2}+\sum_{i=1}^n\rho_{j,r,i}^{(t)}\frac{\bxi_i}{\|\bxi_i\|^2},
    \end{align*}
    from \eqref{eq:without_regu_directmu} we have
    \begin{align*}
       \max_{j,r} |\la \wb_{j,r}^{(\tT_1)},\bmu\ra|= \max_{j,r}|\la \wb_{j,r}^{(0)},\bmu\ra+j\gamma_{j,r}^{(\tT_1)}|=\tilde{O}(\sigma_0\|\bmu\|),
    \end{align*}
    we can easily conclude that $\max_{j,r}\gamma_{j,r}^{(\tT_1)}=\tilde{O}(\sigma_0\|\bmu\|)$. From \eqref{eq:without_induction2} and Proposition~\ref{prop:without_regu_lower_bound_rho_appendix}, we have 
    \begin{align*}  
    \frac{1}{20\log(n)\sqrt{\log(8mn/\delta)}}\leq\max_{j=y_i,r}\bigg|\rho_{j,r,i}^{(\tT_1)}+\la\wb_{j,r}^{(0)},\bxi_i\ra+\sum_{i'\not=i}\rho_{j,r,i'}^{(\tT_1)}\frac{\la\bxi_{i'},\bxi_i\ra}{\|\bxi_{i'}\|^2}\bigg|\leq \log^{-1}(n), 
    \end{align*}
    note that $|\rho_{j,r,i}^{(t)}|\leq 4\log(T^*)$, 
\begin{align*}
\bigg|\la\wb_{j,r}^{(0)},\bxi_i\ra+\sum_{i'\not=i}\rho_{j,r,i'}^{(\tT_1)}\frac{\la\bxi_{i'},\bxi_i\ra}{\|\bxi_{i'}\|^2}\bigg|&\leq 2\sqrt{\log(8mn)/\delta}\sigma_0\sigma_p\sqrt{d}+8\log(T^*)n\sqrt{\log(4n^2/\delta)/d}\\
&\ll \frac{1}{\log(n)\sqrt{\log(8mn/\delta)}},
\end{align*}
therefore we conclude that 
\begin{align*}
    \max_{j=y_i,r}\big|\rho_{j,r,i'}^{(\tT_1)}\big|&\geq \max_{j=y_i,r}\bigg|\rho_{j,r,i}^{(\tT_1)}+\la\wb_{j,r}^{(0)},\bxi_i\ra+\sum_{i'\not=i}\rho_{j,r,i'}^{(\tT_1)}\frac{\la\bxi_{i'},\bxi_i\ra}{\|\bxi_{i'}\|^2}\bigg|-\max_{j=y_i,r}\bigg|\la\wb_{j,r}^{(0)},\bxi_i\ra+\sum_{i'\not=i}\rho_{j,r,i'}^{(\tT_1)}\frac{\la\bxi_{i'},\bxi_i\ra}{\|\bxi_{i'}\|^2}\bigg|\\
    &\geq \frac{1}{20\log(n)\sqrt{\log(8mn/\delta)}}-\frac{1}{420\log(n)\sqrt{\log(8mn/\delta)}}=\frac{1}{21\log(n)\sqrt{\log(8mn/\delta)}}.
\end{align*}
Similarly, we can also conclude that  
\begin{align*}
    \max_{j=y_i,r}\big|\rho_{j,r,i'}^{(\tT_1)}\big|\leq 2\log^{-1}(n).
\end{align*}
The last inequality directly comes from \eqref{eq:admissible_withoutregu2}, where we set $\sigma_0$ small enough in Condition~\ref{condition:condition}.
\end{proof}

\subsection{Stage 2: Continuous growth of noise memorizing}
In this section, we aim to prove that the noise memorizing will grow to a constant order while the signal learning remains to the small scale. Thus the gradient loss is getting changed, no longer equal to $-1/2+o(1)$. In this stage, we select a large enough constant $C>0$, and define the time interval $\tT_2$.
\begin{align}
    \tT_2=\frac{Cnm}{\eta\sigma_p^2d }\log\bigg(\frac{1}{\log(n)\sqrt{\log(8mn/\delta)}}\bigg).\label{eq:def_T2}
\end{align}
We then have the following proposition, which shows the growth of noise memorizing.
\begin{proposition}[Restatement of Proposition~\ref{prop:without_stage2_induct}]
    \label{prop:without_stage2_induct_appendix}
    Let $\tT_2$ be defined in \eqref{eq:def_T2}, then for any $i\in[n]$ and $r\in[m]$ it holds that 
    \begin{align*}
        &\gamma_{j,r}^{(\tT_1+\tT_2)}=\tilde{O}(\sigma_0\|\bmu\|),\quad \max_{j=y_i,r}\rho_{j,r,i}^{(\tT_1+\tT_2)}\geq 2,\\
        & \max_{j\not=y_i,r}\big|\rho_{j,r,i}^{(\tT_1+\tT_2)}\big| \leq 65n\sqrt{\frac{\log(4n^2/\delta)}{d}}\log(T^*).
    \end{align*}
\end{proposition}
\begin{proof}[Proof of Proposition~\ref{prop:without_stage2_induct_appendix}]
    We first prove that $\gamma_{j,r}^{(\tT_1+\tT_2)}=\tilde{O}(\sigma_0\|\bmu\|)$. It holds that
    \begin{align*}
        \gamma_{j,r}^{(\tT_1+\tT_2)}+|\la \wb_{j,r}^{(0)},\bmu\ra|\leq \big(\gamma_{j,r}^{\tT_1}+|\la \wb_{j,r}^{(0)},\bmu\ra|\big)\cdot \bigg(1+\frac{2\eta\|\bmu\|^2}{nm}\bigg)^{\tT_2}\\
        \leq 2\big(\gamma_{j,r}^{\tT_1}+|\la \wb_{j,r}^{(0)},\bmu\ra|\big)=\tilde{O}(\sigma_0\|\bmu\|).
    \end{align*}
    Here, the first inequality utilizes $|\ell'_i|\leq1$ and the second inequality is by $\frac{2\eta\|\bmu\|^2}{nm}\cdot\tT_2=o(1)$. Hence we have $\gamma_{j,r}^{(\tT_1+\tT_2)}=\tilde{O}(\sigma_0\|\bmu\|)$. \eqref{eq:admissible_withoutregu2} also gives that $\max_{j\not=y_i,r}\big|\rho_{j,r,i}^{(\tT_1+\tT_2)}\big|\leq 65n\sqrt{\frac{\log(4n^2/\delta)}{d}}\log(T^*)$. We next prove that $\max_{j=y_i,r}\rho_{j,r,i}^{(\tT_1+\tT_2)}\geq 2$. We remind the readers the update rule.
    \begin{align*}
    \rho_{j,r,i}^{(t+1)}=\rho_{j,r,i}^{(t)}-\frac{2\eta}{nm}{\ell'^{(t)}_i} \ReLU ( \la\wb_{j,r}^{(t)}, \bxi_i \ra)\|\bxi_i\|^2\cdot jy_i.
\end{align*}
    Note that when $j=y_i$, $\rho_{j,r,i}^{(t)}=\overrho_{j,r,i}^{(t)}$. For each $i$, define $\tT_2^{(i)}$ the last time satisfying $\overrho_{j,r,i}^{(\tT_1+t)}\leq 2$. Then for $0<t\leq\tT_2 $, we have $\max_{j,r}\{\overrho_{j,r,i}^{(\tT_1+t)},\underrho_{j,r,i}^{(\tT_1+t)} \}=O(1)$, $\max_{j,r}\{\gamma_{j,r}^{\tT_1+t} \}=o(1)$. Therefore we know that $F_{j}(\Wb^{(t)},\xb_i)=O(1)$ for each $j$. We conclude that there exists constant $c_1>0$ such that $-\ell'^{(\tT_1+t)}_i\geq c_1$ for all $t\in[\tT_2^{(i)}]$.  Now let 
    \begin{align*}
    B_i^{(t)}=\max_{j=y_i,r}\bigg\{\overrho_{j,r,i}^{(t)}+\la \wb_{j,r}^{(0)},\bxi_i\ra -65n\sqrt{\frac{\log(4n^2/\delta)}{d}}\log(T^*)\bigg\},
    \end{align*}
    from the analysis above we have
    \begin{align}
        B_i^{(t+\tT_1)}&\geq B_i^{(t-1+\tT_1)}+\frac{c_1\eta\sigma_p^2d}{nm}\cdot B_i^{t-1+\tT_1}\nonumber\\
        &\geq B_i^{(\tT_1)}\bigg(1+\frac{c_1\eta\sigma_p^2d}{nm}\bigg)^{\tT_2}\geq B_i^{(\tT_1)}\exp\bigg\{\frac{c_1\eta\sigma_p^2d}{2nm}\cdot \tT_2 \bigg\},\label{eq:stage2_lowerbound}
    \end{align}
    where the first inequality is by the definition of $B_i^{(t)}$ and the update rule of $\rho_{j,r,i}^{(t)}$. By Corollary~\ref{coro: without_decompose_bound}, we have 
    \begin{align*}
        B_i^{(\tT_1)}&=\max_{j=y_i,r}\bigg\{\overrho_{j,r,i}^{(t)}+\la \wb_{j,r}^{(0)},\bxi_i\ra -65n\sqrt{\frac{\log(4n^2/\delta)}{d}}\log(T^*)\bigg\}\\
        &\geq  \frac{1}{21\log(n)\sqrt{\log(8mn/\delta)}}+\max_{j=y_i,r}\bigg\{\la \wb_{j,r}^{(0)},\bxi_i\ra -65n\sqrt{\frac{\log(4n^2/\delta)}{d}}\log(T^*)\bigg\}\\
        &\geq \frac{1}{22\log(n)\sqrt{\log(8mn/\delta)}}.
    \end{align*}
    Here, the first inequality is by $\max_{j=y_i,r}\overrho_{j,r,i}^{(\tT_1)}\geq\frac{1}{21\log(n)\sqrt{\log(8mn/\delta)}}$ in Corollary~\ref{coro: without_decompose_bound}, and the second inequality is by the remaining part is far smaller than $1/\{\log(n)\sqrt{\log(8mn/\delta)}\}$. Within $\tT_2$ iterations after stage 1, \eqref{eq:stage2_lowerbound} indicates that $B_i^{(\tT_1+\tT_2)}\geq 3$. By the definition of $B_i^{(\tT_1+\tT_2)}$, and $\max_{j,r}|\la\wb_{j,r}^{(0)},\bxi_i\ra -65n\sqrt{\log(4n^2/\delta)/d}\log(T^*)|=o(1)$, we conclude that $\max_{j,r}\overrho_{j,r,i}^{(\tT_1+t)}\geq 2$ for $t=\tT_2$.  Hence we have that $\tT_2^{(i)}\leq \tT_2$, which indicates that $\max_{j=y_i,r}\rho_{j,r,i}^{(\tT_1+\tT_2)}\geq 2$.
\end{proof}

\subsection{Stage 3: The boundary of signal and noise learning in final stage}
In this section, we aim to prove the boundary of signal and noise learning which helps us prove the final result. By the signal noise decomposition, at the end of Stage 2, we have
\begin{align*}
    \wb_{j,r}^{(\tT_1+\tT_2)}=\wb_{j,r}^{(0)}+j\gamma_{j,r}^{(\tT_1+\tT_2)}\frac{\bmu}{\|\bmu\|^2}+\sum_{i=1}^{n}\overrho_{j,r,i}^{(\tT_1+\tT_2)}\frac{\bxi_i}{\|\bxi_i\|^2}+\sum_{i=1}^{n}\underrho_{j,r,i}^{(\tT_1+\tT_2)}\frac{\bxi_i}{\|\bxi_i\|^2}
\end{align*}
for $j\in\{\pm1\}$ and $r\in[m]$. By the results we get in the second stage, we know that at the beginning of this stage, the following properties hold:
\begin{enumerate}
    \item $\max_{j,r}\gamma_{j,r}^{(\tT_1+\tT_2)}=\tilde{O}(\sigma_0\|\bmu\|)$.
    \item $\max_{j\not=y_i,r,i}\underrho_{j,r,i}^{(\tT_1+\tT_2)}\leq 65n\sqrt{\log(4n^2/\delta)/d}\log(T^*)$ for all $i\in[n]$.
    \item $\max_{j=y_i,r,i}\overrho_{j,r,i}^{(\tT_1+\tT_2)}\geq 2$ for all $i\in[n]$.
\end{enumerate}
For any $\varepsilon>0$, we define 
\begin{align*}
\wb_{j,r}^*=\wb_{j,r}^{(0)}+4m\log(4/\varepsilon)\cdot \bigg(\sum_{i=1}^n \one(j=y_i)\frac{\bxi_i}{\| \bxi_i\|} \bigg).
\end{align*}
 Based on the definition of $\Wb^*$, we have the following lemma.
\begin{lemma}[Restatement of Lemma~\ref{lemma:stage3_WT1_diff}]
    \label{lemma:stage3_WT1_diff_appendix}
    Under the same condition as Theorem~\ref{thm:withoutregularization}, it holds that $\|\Wb^{(\tT_1+\tT_2)}-\Wb^{*}\|_{F}\leq \tilde{O}(m^{\frac{3}{2}}n^{\frac{1}{2}}\sigma_p^{-1}d^{-\frac{1}{2}})$.
\end{lemma}
\begin{proof}[Proof of Lemma~\ref{lemma:stage3_WT1_diff_appendix}]
It is easy to see that 
\begin{align*}
    \wb_{j,r}^{(\tT_1+\tT_2)}-\wb_{j,r}^*=-4m\log(4/\varepsilon)\cdot \bigg(\sum_{i=1}^n \one(j=y_i)\frac{\bxi_i}{\| \bxi_i\|} \bigg)+\wb_{j,r}^{(\tT_1+\tT_2)}-\wb_{j,r}^{(0)},
\end{align*}
thus we have
\begin{align*}
    \|\Wb^{(\tT_1+\tT_2)}-\Wb^{*}\|_{F}&\leq 2m\|\wb_{j,r}^{(\tT_1+\tT_2)}-\wb_{j,r}^{(0)}\|+2m\bigg\|4m\log(4/\varepsilon)\cdot \bigg(\sum_{i=1}^n \one(j=y_i)\frac{\bxi_i}{\| \bxi_i\|} \bigg) \bigg\|_F\\
    &\leq 2m\frac{|\gamma_{j,r}^{(t)}|}{\|\bmu\|}+\sum_{i=1}^n\frac{|\rho_{j,r,i}^{(t)}|}{\|\bxi_i\|}+ \tilde{O}(m^{\frac{3}{2}}n^{\frac{1}{2}}\sigma_p^{-1}d^{-\frac12})=\tilde{O}(m^{\frac{3}{2}}n^{\frac{1}{2}}\sigma_p^{-1}d^{-\frac12}).
\end{align*}
Here, the first inequality is by triangle inequality, and the second inequality comes from $|\la\bxi_i,\bxi_{i'} \ra|/\|\bxi_i \|^2=\tilde{O}(\sigma_p^{-1}d^{-1/2})$ and $|\gamma_{j,r}^{(t)}|, |\rho_{j,r,i}^{(t)}|\leq 4\log(T^*)$.
\end{proof}

\begin{lemma}
 \label{lemma:stage3_gradient_product_appendix}   
 Under the same condition as Theorem~\ref{thm:withoutregularization}, it holds that
 \begin{align*}
     y_i\la \nabla_{\Wb^{(t)}}f(\Wb^{(t)},\xb_i),\Wb^*\ra\geq 2\log(4/\varepsilon)
\end{align*}
for all $\tT_1+\tT_2\leq t\leq T^*$.
\end{lemma}
\begin{proof}[Proof of Lemma~\ref{lemma:stage3_gradient_product_appendix}]
    Recall that $f(\Wb^{(t)},\xb_i)=(1/m)\sum_{j,r}j\cdot\big[\sigma(\la\wb_{j,r}^{(t)},y_i\bmu\ra)+\sigma(\la\wb_{j,r}^{(t)},\bxi_i \ra)\big]$, so we have
    \begin{align*}
        &y_i\la\nabla f(\Wb^{(t)},\xb_i),\Wb^*\ra\\
        &\quad =\frac{1}{m}\sum_{j,r}\sigma'(\la\wb_{j,r}^{(t)},y_i\bmu\ra)\la\bmu,j\wb_{j,r}^* \ra+\frac{1}{m}\sum_{j,r}\sigma'(\la\wb_{j,r}^{(t)},\bxi_i \ra)\la y_i\bxi_i,j\wb_{j,r}^*\ra\\
        &\quad =\frac{1}{m}\sum_{j,r}\sum_{i'=1}^n\sigma'(\la\wb_{j,r}^{(t)},\bxi_i \ra)4\log(4/\varepsilon)\one(j=y_{i'})\cdot \frac{\la \bxi_{i'},\bxi_i\ra}{\|\bxi_{i'} \|}\\
        &\quad\quad+\frac{1}{m}\sum_{j,r}\sigma'(\la\wb_{j,r}^{(t)},y_i\bmu \ra)\la\bmu,j\wb_{j,r}^{(0)}\ra+\frac{1}{m}\sum_{j,r}\sigma'(\la\wb_{j,r}^{(t)},\bxi_i \ra)\la y_i\bxi_i,j\wb_{j,r}^{(0)}\ra\\
        &\quad \geq \frac{1}{m}\sum_{j=y_i,r}\sigma'(\la\wb_{j,r}^{(t)},\bxi_i \ra)4\log(4/\varepsilon)-\frac{1}{m}\sum_{j,r}\sum_{i'\not=i}\sigma'(\la\wb_{j,r}^{(t)},\bxi_i \ra)4\log(4/\varepsilon)\cdot \frac{|\la \bxi_{i'},\bxi_i\ra|}{\|\bxi_{i'} \|}\\
        &\quad \quad -\frac{1}{m}\sum_{j,r}\sigma'(\la\wb_{j,r}^{(t)},y_i\bmu \ra)\tilde{O}(\sigma_0\|\bmu\|)-\frac{1}{m}\sum_{j,r}\sigma'(\la\wb_{j,r}^{(t)},\bxi_i \ra)\tilde{O}(\sigma_0\sigma_p\sqrt{d})\\
        &\quad \geq \frac{1}{m}\sum_{j=y_i,r}\sigma'(\la\wb_{j,r}^{(t)},\bxi_i \ra)4\log(4/\varepsilon)-\frac{1}{m}\sum_{j,r}\sigma'(\la\wb_{j,r}^{(t)},\bxi_i \ra)4\log(4/\varepsilon)\tilde{O}(mnd^{-1/2})\\
        &\quad \quad -\frac{1}{m}\sum_{j,r}\sigma'(\la\wb_{j,r}^{(t)},y_i\bmu \ra)\tilde{O}(\sigma_0\|\bmu\|)-\frac{1}{m}\sum_{j,r}\sigma'(\la\wb_{j,r}^{(t)},\bxi_i \ra)\tilde{O}(\sigma_0\sigma_p\sqrt{d}),
    \end{align*}
where the first inequality is by Lemma~\ref{lemma:data_CNN_concentration} and the last inequality is by Lemma~\ref{lemma:data_noise_concentration}. We next give the lower bound of $y_i\la\nabla f(\Wb^{(t)},\xb_i),\Wb^*\ra$ from the equation above. From \eqref{eq:admissible_withoutregu1} and \eqref{eq:admissible_withoutregu2}, we can easily have that for any $j\in\{\pm1\}$, $r\in[m]$ and $i\in[n]$,
\begin{align*}
    |\la\wb_{j,r}^{(t)},\bmu \ra|=\tilde{O}(1), \quad |\la\wb_{j,r}^{(t)},\bxi_i\ra|=\tilde{O}(1).
\end{align*}
On the other hand, for $j=y_i$, we can bound the inner product between the parameter and the noise as follows
\begin{align*}
    \max_{j=y_i,\bxi_i}\la\wb_{j,r}^{(t)},\bxi_i \ra\geq \max_{j=y_i,r}[\la\wb_{j,r}^{(0)},\bxi_i\ra+\overrho_{j,r,i}^{(t)}]-32n\sqrt{\frac{\log(4n^2/\delta)}{d}}\log(T^*)\geq1,
\end{align*}
where the first inequality is by \eqref{eq:admissible_withoutregu1} and \eqref{eq:admissible_withoutregu2}, the second inequality is by Corollary~\ref{coro: without_decompose_bound} and $\overrho_{j,r,i}$ increases when $t$ increases. Therefore, we can conclude that
\begin{align*}
    y_i\la\nabla f(\Wb^{(t)},\xb_i),\Wb^*\ra&\geq4\log(4/\varepsilon)-\tilde{O}(mnd^{-1/2})-\tilde{O}(\sigma_0\|\bmu\|)-\tilde{O}(\sigma_0\sigma_p\sqrt{d})\\
    &\geq2\log(4/\varepsilon),
\end{align*}
where the last inequality is by Condition~\ref{condition:condition}.
\end{proof}

\begin{lemma}[Restatement of Lemma~\ref{lemma:stage3_diff_Wt}]
\label{lemma:stage3_diff_Wt_appendix}
Under the same condition as Theorem~\ref{thm:withoutregularization}, it holds that 
\begin{align*}
    \|\Wb^{(t)}-\Wb^*\|_F^2-\|\Wb^{(t+1)}-\Wb^*\|_F^2\geq 3\eta L_S(\Wb^{(t)})-\eta\varepsilon
\end{align*}
for all $\tT_1+\tT_2\leq t\leq T^*$.
\end{lemma}
\begin{proof}[Proof of Lemma~\ref{lemma:stage3_diff_Wt_appendix}]
    Note that from the update rule, we have
    \begin{align*}
        \Wb^{(t+1)}-\Wb^*=\Wb^{(t)}-\Wb^*-\eta\nabla_{\Wb} L_S(\Wb)|_{\Wb=\Wb^{(t)}},
    \end{align*}
we have the following inequalities.
\begin{align*}
&\|\Wb^{(t)}-\Wb^*\|_F^2-\|\Wb^{(t+1)}-\Wb^*\|_F^2\\
&\qquad=2\eta\la\nabla_{\Wb} L_S(\Wb)|_{\Wb=\Wb^{(t)}}, \Wb^{(t)}-\Wb^*\ra-\eta^2\|\nabla_{\Wb} L_S(\Wb)|_{\Wb=\Wb^{(t)}} \|_F^2\\
&\qquad = \frac{2\eta}{n}\sum_{i=1}^n\ell'^{(t)}_i[2y_if(\Wb^{(t)},\xb_i)-\la\nabla f(\Wb^{(t)},\xb_i),\Wb^* \ra]-\eta^2\|\nabla_{\Wb} L_S(\Wb)|_{\Wb=\Wb^{(t)}} \|_F^2\\
&\qquad \geq \frac{4\eta}{n}\sum_{i=1}^n\ell'^{(t)}_i[y_if(\Wb^{(t)},\xb_i)-\log(4/\varepsilon)]-\eta^2\|\nabla_{\Wb} L_S(\Wb)|_{\Wb=\Wb^{(t)}} \|_F^2\\
&\qquad \geq \frac{4\eta}{n}\sum_{i=1}^n [\ell\big(y_if(\Wb^{(t)},\xb_i)\big)-\varepsilon/4]-\eta^2\|\nabla_{\Wb} L_S(\Wb)|_{\Wb=\Wb^{(t)}} \|_F^2\\
&\qquad\geq 3\eta L_{S}(\Wb^{(t)})-\eta\varepsilon,
\end{align*}
where the first inequality is by Lemma~\ref{lemma:stage3_gradient_product_appendix}, the second inequality is due to the convexity of the cross entropy function and the last inequality is due to Lemma~\ref{lemma:admissible_trainloss_appendix} and $\eta=1/\poly{(n)}\ll (\sigma_p^2d)^{-1}$. The proof of Lemma~\ref{lemma:stage3_diff_Wt_appendix} is completed.
\end{proof}
From the lemmas above, we give the following lemma which gives the bound of the training loss. 
\begin{lemma}
    \label{lemma:stage3_trainloss}
Under the same conditions as Theorem~\ref{thm:withoutregularization}, Let $T=\tT_1+\tT_2+\Big\lfloor\frac{\|\Wb^{(\tT_1+\tT_2)}-\Wb^*\|_F^2 }{2\eta\varepsilon}\Big\rfloor=\tT_1+\tT_2+\tilde{O}(\eta^{-1}\varepsilon^{-1}m^3n\sigma_p^{-2}d^{-1})$. Then it holds that $\max_{j,r}\gamma_{j,r}^{(t)}=\tilde{O}(\sigma_0\|\bmu\|)$, $\max_{j,r,i}|\underrho_{j,r,i}^{(t)}|=\tilde{O}(nd^{-1/2})$ for all $\tT_1+\tT_2\leq t\leq T$. Besides, 
\begin{align*}
    \frac{1}{t-\tT_1-\tT_2+1}\sum_{s=\tT_1+\tT_2}^t L_S(\Wb^{(s)})\leq \frac{\|\Wb^{(\tT_1+\tT_2)}-\Wb^*\|_F^2}{3\eta (t-\tT_1-\tT_2+1)}+\frac{\varepsilon}{3}
\end{align*}
for all $\tT_1+\tT_2\leq t\leq T$, and there exists an  iterate with training loss smaller than  $\varepsilon$ within $T$ iterations.
\end{lemma}
\begin{proof}[Proof of Lemma~\ref{lemma:stage3_trainloss}]
   From Lemma~\ref{lemma:stage3_diff_Wt_appendix}, for any $\tT_1+\tT_2\leq t\leq T$, we obtain that
   \begin{align*}
    \|\Wb^{(s)}-\Wb^*\|_F^2-\|\Wb^{(s+1)}-\Wb^*\|_F^2\geq 3\eta L_S(\Wb^{(s)})-\eta\varepsilon
\end{align*}
holds for $\tT_1+\tT_2\leq s\leq t$. Taking a summation, we have that
\begin{align*}
    \sum_{s=\tT_1+\tT_2}^t L_S(\Wb^{(s)})&\leq \frac{\|\Wb^{(\tT_1+\tT_2)}-\Wb^*\|_F^2-\|\Wb^{(t+1)}-\Wb^*\|_F^2+\eta\varepsilon(t-\tT_1-\tT_2+1)}{3\eta}\\
    &\leq \frac{2\|\Wb^{(\tT_1+\tT_2)}-\Wb^*\|_F^2}{3\eta}=\tilde{O}(\eta^{-1}m^3nd^{-1}\sigma_p^{-2}),
\end{align*}
where the second inequality is by $t\leq T$ and the last equality is by Lemma~\ref{lemma:stage3_WT1_diff_appendix}. The inequality above also indicates that
\begin{align*}
    \frac{1}{t-\tT_1-\tT_2+1}\sum_{s=\tT_1+\tT_2}^t L_S(\Wb^{(s)})\leq \frac{\|\Wb^{(\tT_1+\tT_2)}-\Wb^*\|_F^2}{3\eta (t-\tT_1-\tT_2+1)}+\frac{\varepsilon}{3}.
\end{align*}
Note that from \eqref{eq:admissible_withoutregu2}, we already have $\max_{j,r,i}\underrho_{j,r,i}^{(t)}=\tilde{O}(nd^{-1/2})$, we next prove that  $\max_{j,r}\gamma_{j,r}^{(t)}=\tilde{O}(\sigma_0\|\bmu\|)$. From Proposition~\ref{prop:without_stage2_induct_appendix}, we can define $\beta'=\tilde{O}(\sigma_0\|\bmu\|)$ such that $\max_{j,r}\gamma_{j,r}^{(\tT_1+\tT_2)}\leq \beta'$. We then use mathematical induction to prove that for any $t\leq T$, $\max_{j,r}\gamma_{j,r}^{(t)}\leq 2\beta'$. Suppose that there exists $T_0\in[\tT_1+\tT_2,T]$ such that $\max_{j,r}\gamma_{j,r}^{(t)}\leq 2\beta'$ for all $t\in[\tT_1+\tT_2,T_0-1]$. Then by the update rule \eqref{eq:update_rule_lbd=0}, we have
\begin{align*}
    \gamma_{j,r}^{(T_0)}&\leq \gamma_{j,r}^{(\tT_1+\tT_2)}-\frac{\eta}{nm}\sum_{s=\tT_1+\tT_2}^{T_0-1}\sum_{i=1}^n\ell'^{(t)}_i\cdot \sigma'(\la \wb_{j,r}^{(0)},y_i\bmu\ra+jy_i\gamma_{j,r}^{(t)}\ra)\|\bmu\|^2\\
    &\leq \gamma_{j,r}^{(\tT_1+\tT_2)}+\frac{\eta\|\bmu\|^2}{nm}2\cdot |3\beta' |\cdot \sum_{s=\tT_1+\tT_2}^{T_0-1}\sum_{i=1}^n|\ell'^{(t)}_i|\\
    &\leq \beta'+6\eta\|\bmu\|^2\beta'm^{-1}\sum_{s=\tT_1+\tT_2}^{T_0-1}L_S(\Wb^{(s)})\\
    &\leq \beta'+  6\eta\|\bmu\|^2\beta'm^{-1}\cdot\tilde{O}(\eta^{-1}m^3nd^{-1}\sigma_p^{-2})\leq 2\beta'
\end{align*}
for all $j\in\{\pm1\}$ and $r\in[m]$. Here, the second inequality is by the induction hypothesis $\max_{j,r}\gamma_{j,r}^{(t)}\leq 2\beta'$, the third inequality is by $|\ell'|\leq\ell$, and the last inequality is by Condition~\ref{condition:condition}. Hence we complete the induction.
\end{proof}

\begin{proposition}[restatement of Proposition~\ref{prop:without_Gradient_bound}]
\label{prop:without_Gradient_bound_appendix}
    Under the same condition as Theorem~\ref{thm:withoutregularization}, for any $\varepsilon>0$, let $T=\tT_1+\tT_2+\Big\lfloor\frac{\|\Wb^{(\tT_1+\tT_2)}-\Wb^*\|_F^2 }{2\eta\varepsilon}\Big\rfloor=\tT_1+\tT_2+\tilde{O}(\eta^{-1}\varepsilon^{-1}m^3n\sigma_p^{-2}d^{-1})$,  there exists $t$ in $\tT_1+\tT_2\leq t\leq T$ such that
\begin{align*}
L_S(\Wb^{(t)})\leq \varepsilon.
\end{align*}
Meanwhile, it holds that $\max_{j,r}\gamma_{j,r}^{(t)}=\tilde{O}(\sigma_0\|\bmu\|)$, $\max_{j,r,i}|\underrho_{j,r,i}^{(t)}|=\tilde{O}(nd^{-1/2})$ for all $\tT_1+\tT_2\leq t\leq T$.
\end{proposition}
\begin{proof}[Proof of Proposition~\ref{prop:without_Gradient_bound_appendix}]
From Lemma~\ref{lemma:stage3_trainloss}, it holds that $\max_{j,r}\gamma_{j,r}^{(t)}=\tilde{O}(\sigma_0\|\bmu\|)$, $\max_{j,r,i}|\underrho_{j,r,i}^{(t)}|=\tilde{O}(nd^{-1/2})$ for all $\tT_1+\tT_2\leq t\leq T$. Moreover, we can see that there exists a $t\in[\tT_1+\tT_2,T]$, such that 
\begin{align*}
    L_S(\Wb^{(t)})\leq \varepsilon.
\end{align*}
The proof of Proposition~\ref{prop:without_Gradient_bound_appendix} is completed.
\end{proof}

\subsection{Proof of Theorem~\ref{thm:withoutregularization}}
\label{sec:thmProof2}
From Proposition~\ref{prop:without_Gradient_bound_appendix}, for any $\varepsilon>0$, let 
\begin{align*}
    T=\tT_1+\tT_2+\tilde{O}(\eta^{-1}\varepsilon^{-1}m^3n),
\end{align*}
  there exists $t$ in $\tT_1+\tT_2\leq t\leq T$ such that
\begin{align*}
L_S(\Wb^{(t)})\leq \varepsilon/(72\sigma_p^2d).
\end{align*}
Lemma~\ref{lemma:admissible_trainloss_appendix} completes the convergence of training gradient. For the specific time $t$, we also  have that 
\begin{align*}
\max_{r}\|\wb_{+1,r}^{(t)}\|&=\max_{r}\bigg\|\wb_{+1,r}^{(0)}+\gamma_{+1,r}^{(t)}\frac{\bmu}{\|\bmu\|^2}+\sum_{i=1}^n\overrho_{+1,r,i}^{(t)}\frac{\bxi_i}{\|\bxi_i\|^2}+\sum_{i=1}^n\underrho_{+1,r,i}^{(t)}\frac{\bxi_i}{\|\bxi_i\|^2}\bigg\|\\
   &\geq \max_{r} \bigg\|\sum_{i=1}^n\overrho_{+1,r,i}^{(t)}\frac{\bxi_i}{\|\bxi_i\|^2}    \bigg\|-\max_{r}\bigg\|\wb_{+1,r}^{(0)}+\gamma_{+1,r}^{(t)}\frac{\bmu}{\|\bmu\|^2}+\sum_{i=1}^n\underrho_{+1,r,i}^{(t)}\frac{\bxi_i}{\|\bxi_i\|^2}  \bigg\|\\
   &\geq \max_{r}\bigg|\sum_{i_1,i_2=1}^n\overrho_{+1,r,i_1}^{(t)}\overrho_{+1,r,i_2}^{(t)}\frac{\la\bxi_{i_1},\bxi_{i_2}\ra}{\|\bxi_{i_1}\|^2\|\bxi_{i_2}\|^2}    \bigg|^{\frac{1}{2}}-\tilde{O}(\sigma_0\|\bmu\|+\sigma_0\sigma_p\sqrt{d})-\sum_{i=1}^n\max_{r}\big|\underrho_{+1,r,i}^{(t)}\big|\frac{1}{\|\bxi_i\|}\\
   &= \max_{r}\bigg|\sum_{i_1=1}^n\overrho_{+1,r,i_1}^{2(t)}\frac{1}{\|\bxi_{i_1}\|^2}+  \sum_{i_1\not=i_2}^n\overrho_{+1,r,i_1}^{(t)}\overrho_{+1,r,i_2}^{(t)}\frac{\la\bxi_{i_1},\bxi_{i_2}\ra}{\|\bxi_{i_1}\|^2\|\bxi_{i_2}\|^2}    \bigg|^{\frac{1}{2}}-\tilde{O}(n\sigma_p^{-1}d^{-1}+\sigma_0\sigma_p\sqrt{d})\\
   &\geq \sqrt{4/(\sigma_p^2d)-0.3/(\sigma_p^2d)}-\tilde{O}(n\sigma_p^{-1}d^{-1}+\sigma_0\sigma_p\sqrt{d})\geq 1.9\sigma_p^{-1}d^{-1/2}.
\end{align*}
Here, the first inequality is by triangle inequality, the second inequality is by triangle inequality and $\max_{j,r}\gamma_{j,r}^{(t)}=\tilde{O}(\sigma_0\|\bmu\|)$ in Proposition~\ref{lemma:stage3_trainloss}, the second equality is by \eqref{eq:admissible_withoutregu2} which shows $\max_{j,r,i}|\underrho_{j,r,i}^{(t)}|=\tilde{O}(nd^{-1/2})$,
the third inequality comes from the reason that the part $i_1\not=i_2$ is negligible to the main term $1/(\sigma_p^2d)$,
and the last inequality is by  Condition~\ref{condition:condition} which assumes $\sigma_0$ small enough. Similarly we have that $\max_{r}\|\wb_{-1,r}^{(t)}\|\geq 1.9\sigma_p^{-1}d^{-1/2}$. 

Given a new test data $(\xb,y)$,  it holds that
\begin{align*}
    y\cdot f(\Wb^{(t)},\xb)&=\frac{y}{m}\sum_{r=1}^{r}\sigma( \la \wb_{+1,r}^{(t)}, y\bmu \ra )+\sigma( \la \wb_{+1,r}^{(t)}, \bxi \ra )-\frac{y}{m}\sum_{r=1}^{r}\sigma( \la \wb_{-1,r}^{(t)}, y\bmu \ra )+\sigma( \la \wb_{-1,r}^{(t)}, \bxi \ra )\\
    &=\tilde{O}(\sigma_0\|\bmu\|+\sigma_0\sigma_p\sqrt{d})+\frac{y}{m}\sum_{r=1}^{r}\big(\sigma( \la \wb_{+1,r}^{(t)}, \bxi \ra )-\sigma( \la \wb_{-1,r}^{(t)}, \bxi \ra )\big).
\end{align*}
We have 
\begin{align*}
   P_{(\xb,y)}(y\cdot f(\Wb^{(t)},\xb)<0)&= \frac{1}{2}P_{(\xb,y)}(y\cdot f(\Wb^{(t)},\xb)<0|y=1)+\frac{1}{2}P_{(\xb,y)}(y\cdot f(\Wb^{(t)},\xb)<0|y=-1)\\
   &\geq\frac{1}{2} P\bigg(\frac{1}{m}\sum_{r=1}^{r}\big(\sigma( \la \wb_{+1,r}^{(t)}, \bxi \ra )-\sigma( \la \wb_{-1,r}^{(t)}, \bxi \ra )\big)\leq -\tilde{O}(\sigma_0\sigma_p\sqrt{d})\bigg)\\
   &\qquad +\frac{1}{2} P\bigg(\frac{1}{m}\sum_{r=1}^{r}\big(\sigma( \la \wb_{-1,r}^{(t)}, \bxi \ra )-\sigma( \la \wb_{+1,r}^{(t)}, \bxi \ra )\big)\leq -\tilde{O}(\sigma_0\sigma_p\sqrt{d})\bigg)\\
   &=\frac{1}{2}\bigg(1-P\bigg( \bigg|\frac{1}{m}\sum_{r=1}^{r}\big(\sigma( \la \wb_{-1,r}^{(t)}, \bxi \ra )-\sigma( \la \wb_{+1,r}^{(t)}, \bxi \ra )\big)  \bigg|\leq  \tilde{O}(\sigma_0\sigma_p\sqrt{d}) \bigg)\bigg).\\
   &=\frac{1}{2}-o(1)\geq \frac{1}{2.01}.
\end{align*}
Here, the first inequality is by $\PP(X<0)\geq \PP(X<-a)$ for any $a\geq0$, and the $o(1)$ is by the normal assumption of $\bxi$, $\max_{r}\|\wb_{-1,r}^{(t)}\|,\max_{r}\|\wb_{+1,r}^{(t)}\|\geq 1.9\sigma_p^{-1}d^{-1/2}$ and $\sigma_0$ small enough assumed in Condition~\ref{condition:condition}. We thus complete the proof.

\end{document}